%% file: main.tex
\documentclass{article}
\usepackage{fullpage}

\usepackage[round]{natbib}

\usepackage[utf8]{inputenc} %
\usepackage[T1]{fontenc}    %
\usepackage{hyperref}       %
\usepackage{url}            %
\usepackage{booktabs}       %
\usepackage{amsfonts}       %
\usepackage{nicefrac}       %
\usepackage{microtype}      %

\usepackage{graphicx}
\usepackage{amssymb, amsmath, amsthm}
\usepackage{cleveref}
\usepackage{xcolor}
\usepackage{enumitem}

\DeclareMathOperator*{\E}{E}

\crefformat{equation}{(#2#1#3)}
\crefrangeformat{equation}{(#3#1#4) to~(#5#2#6)}
\crefmultiformat{equation}{(#2#1#3)}%
{ and~(#2#1#3)}{, (#2#1#3)}{ and~(#2#1#3)}

\newcommand{\defined}{:=}
\newcommand{\objective}{h}

\newcommand{\email}[1]{\url{#1}}
\newcommand{\myparagraph}[1]{\noindent {{\bf #1:}}}

\newtheorem{proposition}{Proposition}%
\newtheorem{lemma}[proposition]{Lemma}%
\newtheorem{assumption}[proposition]{Assumption}%
\newtheorem{corollary}[proposition]{Corollary}%

\title{Direct Loss Minimization for Sparse Gaussian Processes}

\author{ 
{\bf Yadi Wei} \\ Indiana University \\ Bloomington, IN \\ \email{weiyadi@iu.edu}
\and
{\bf Rishit Sheth} \\ Microsoft Research New England \\ Cambridge, MA \\ \email{rishet@microsoft.com} 
\and
{\bf Roni Khardon} \\ Indiana University \\ Bloomington, IN \\ \email{rkhardon@iu.edu}
}

\begin{document}

\maketitle

\input{allplots}

\input{abstract}

\input{introduction}

\input{sparse-GP-details}

\input{productsampling}
\input{convergence}

\input{related_work}

\input{experiments}
\input{conclusion}

\section*{Acknowledgements}
 This work was partly supported by NSF under grant IIS-1906694.
 Some of the experiments in this paper were run on the Big Red 3 computing system at Indiana University,
 supported in part by Lilly Endowment, Inc., through its support for the Indiana University Pervasive Technology Institute.

\bibliography{../local}

\newpage
\appendix
\input{appendix}

\end{document}

%% file: allplots.tex
%

%

%
%

\DeclareGraphicsExtensions{.eps,.png,.pdf}

\newcommand{\PutPlotSelection}[1]{
\begin{figure*}[t]
{ \scriptsize
\begin{center}
\begin{minipage}{0.330\linewidth}
\begin{center}
    \includegraphics[width=0.99\linewidth]{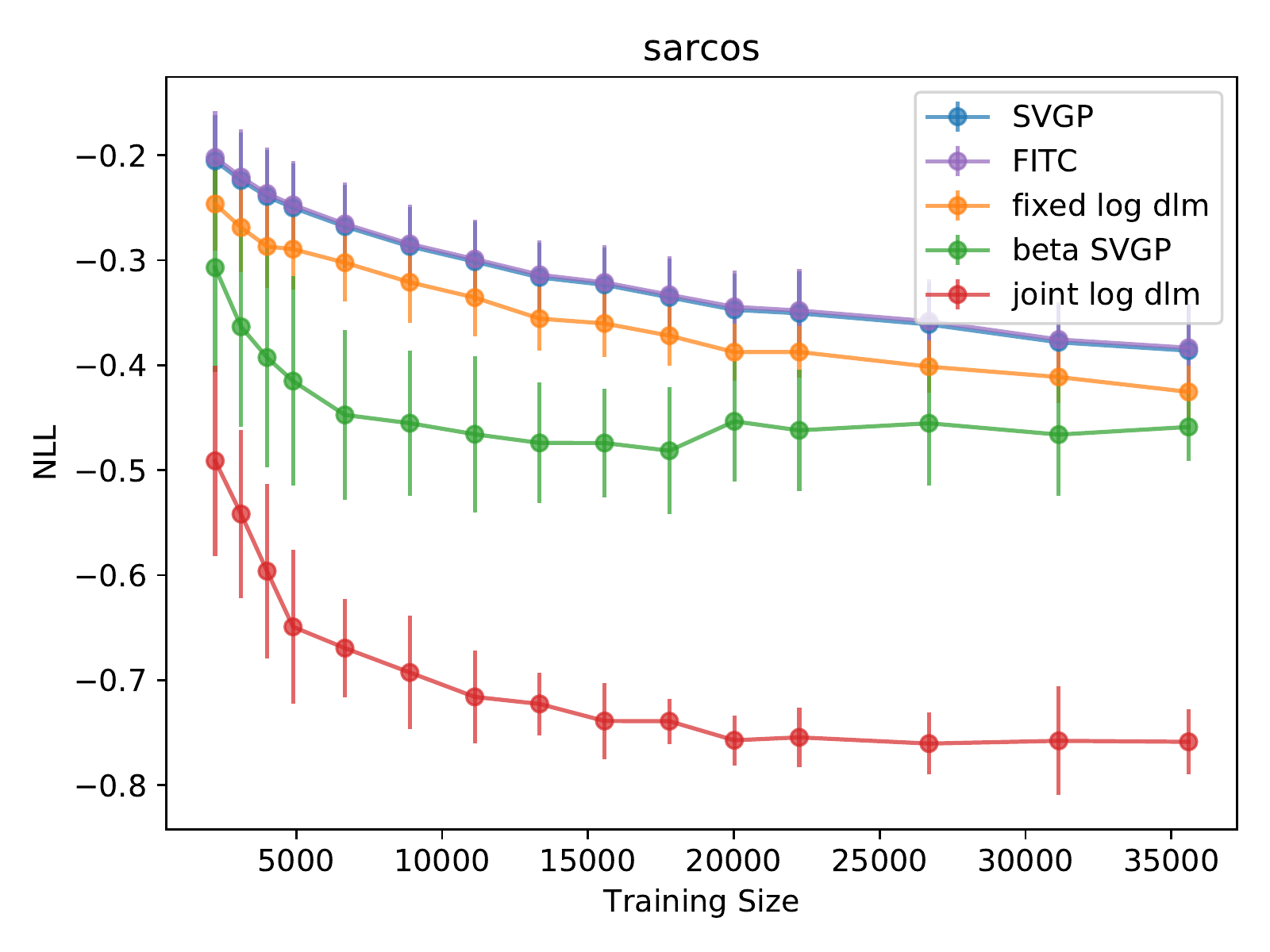}
        (a) regression; log loss
\end{center}
\end{minipage}%
\begin{minipage}{0.330\linewidth}
\begin{center}
    \includegraphics[width=0.99\linewidth]{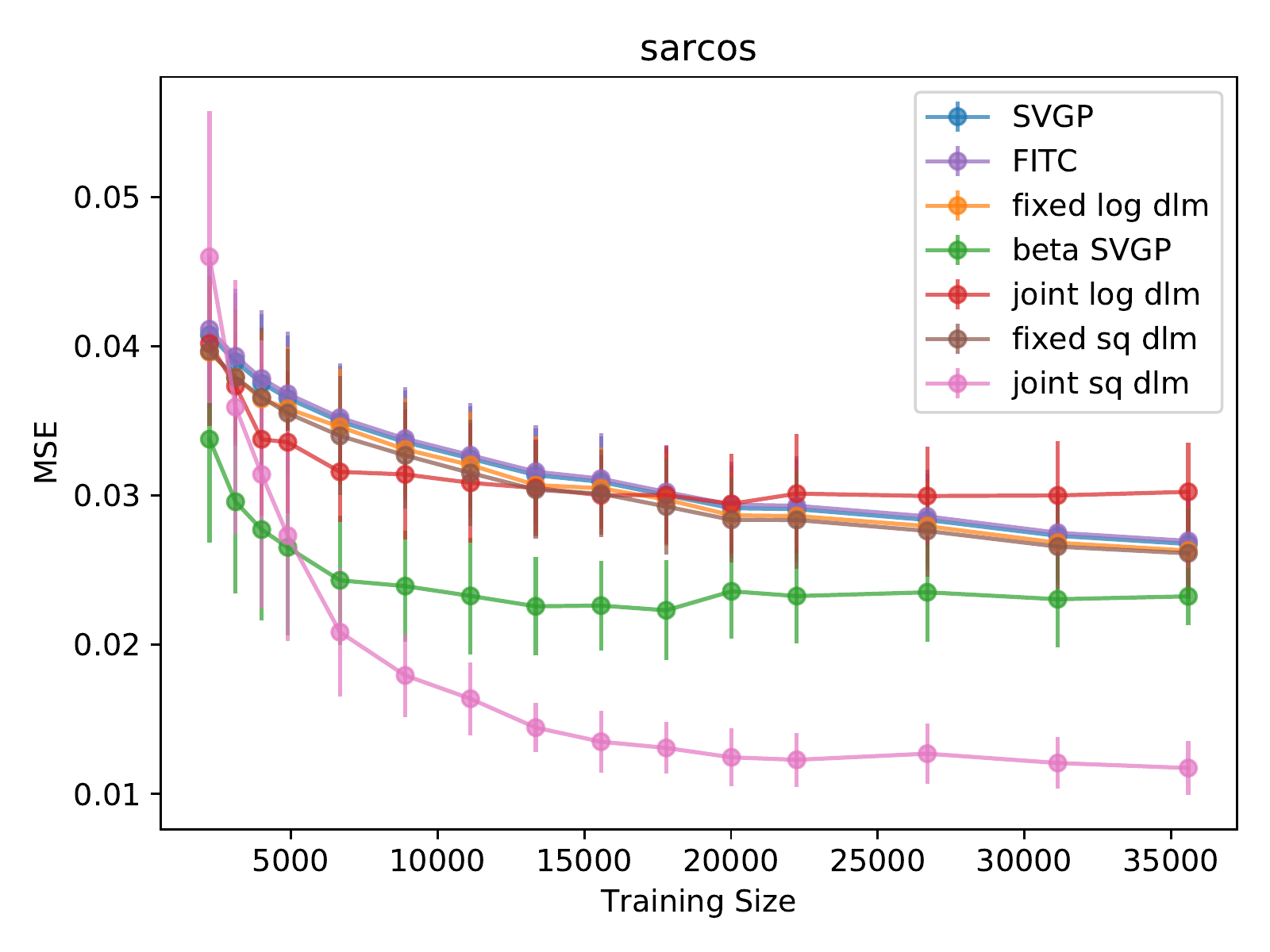}
       (b) regression; square loss
\end{center}
\end{minipage}%
\begin{minipage}{0.330\linewidth}
\begin{center}
    \includegraphics[width=0.99\linewidth]{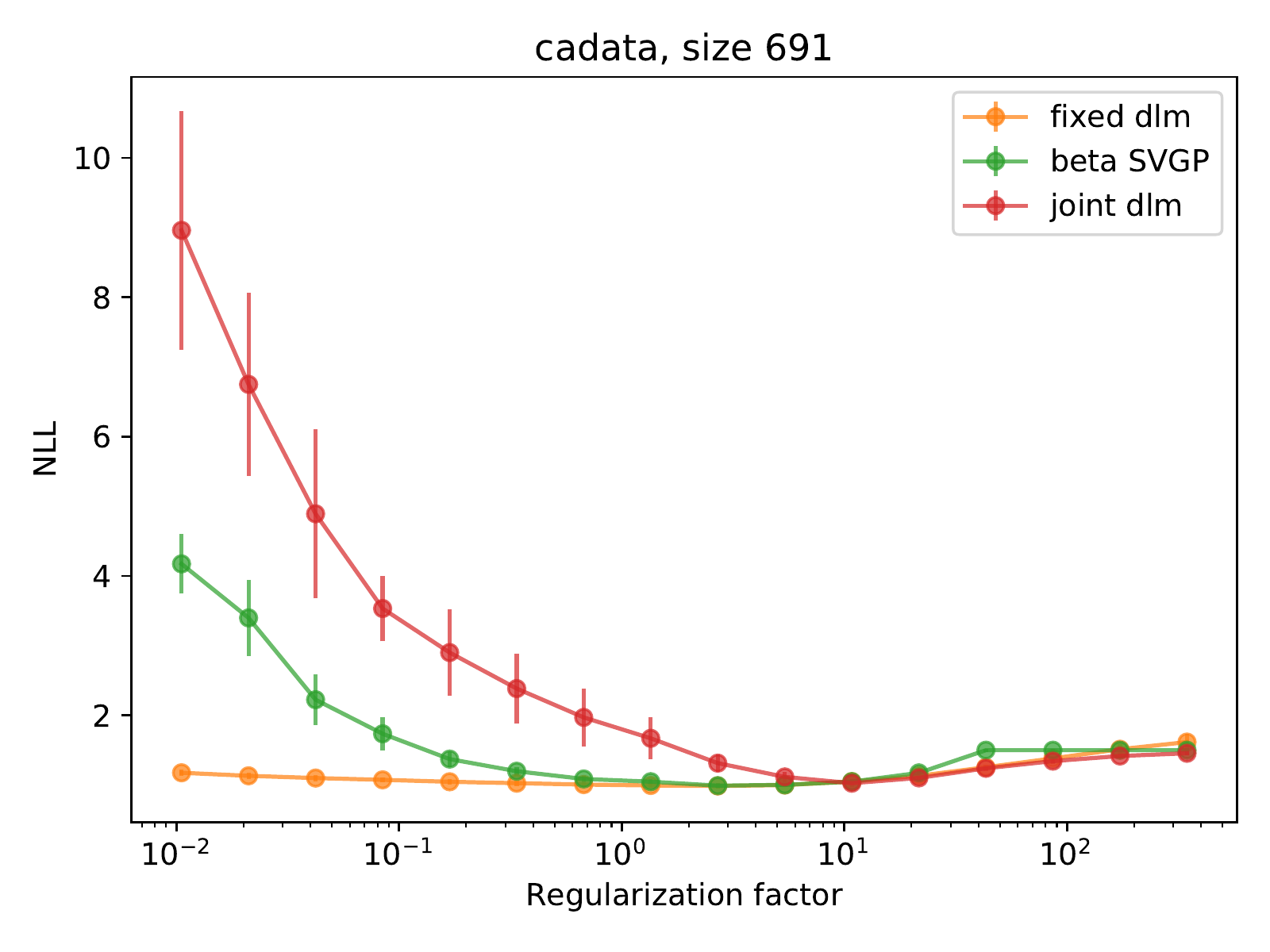}
    (c) selected $\beta$ values
\end{center}
\end{minipage}\quad
\begin{minipage}{0.330\linewidth}
\begin{center}
    \includegraphics[width=0.99\linewidth]{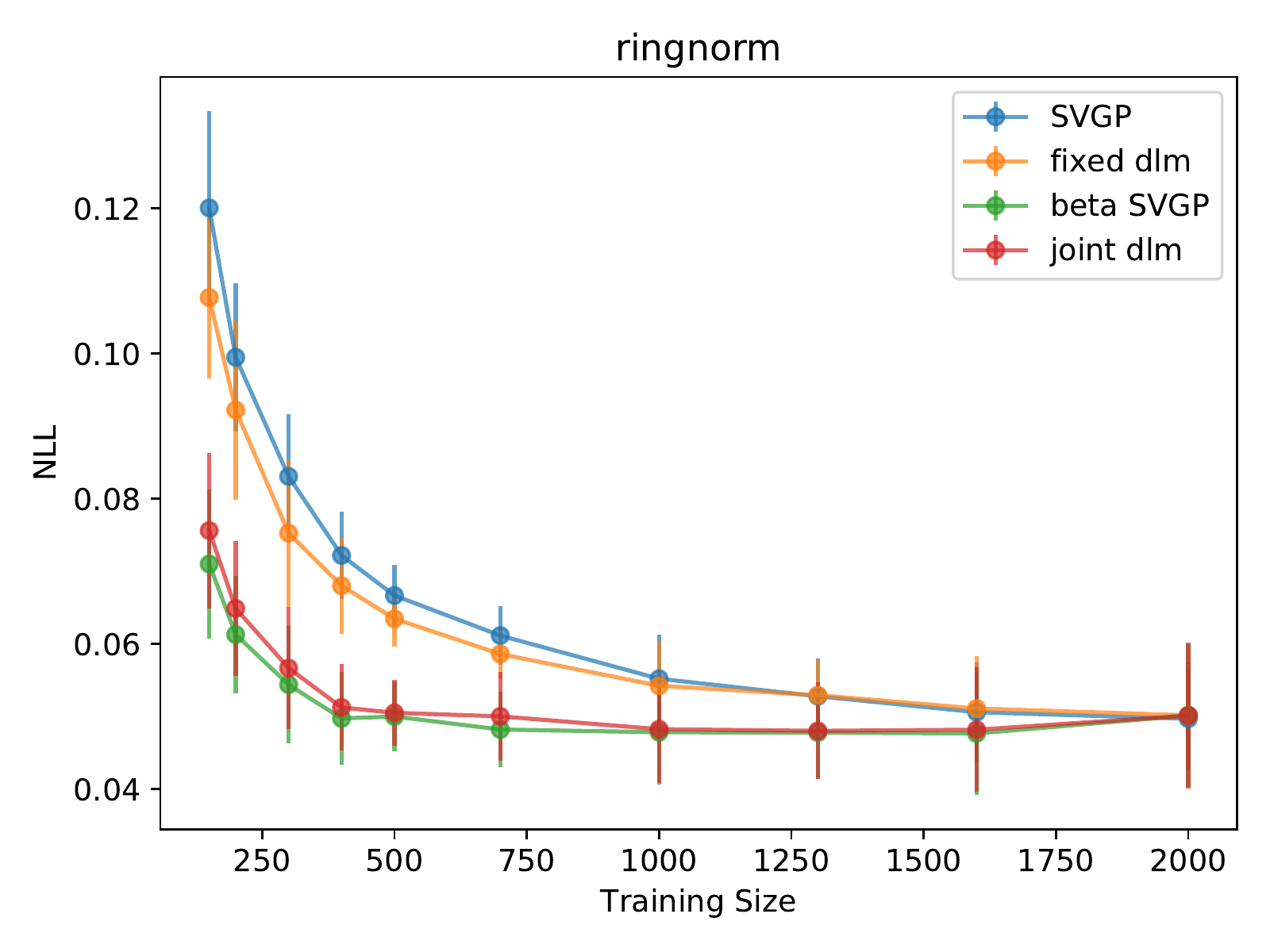}
        (d) classification; log loss
\end{center}
\end{minipage}%
\begin{minipage}{0.330\linewidth}
\begin{center}
    \includegraphics[width=0.99\linewidth]{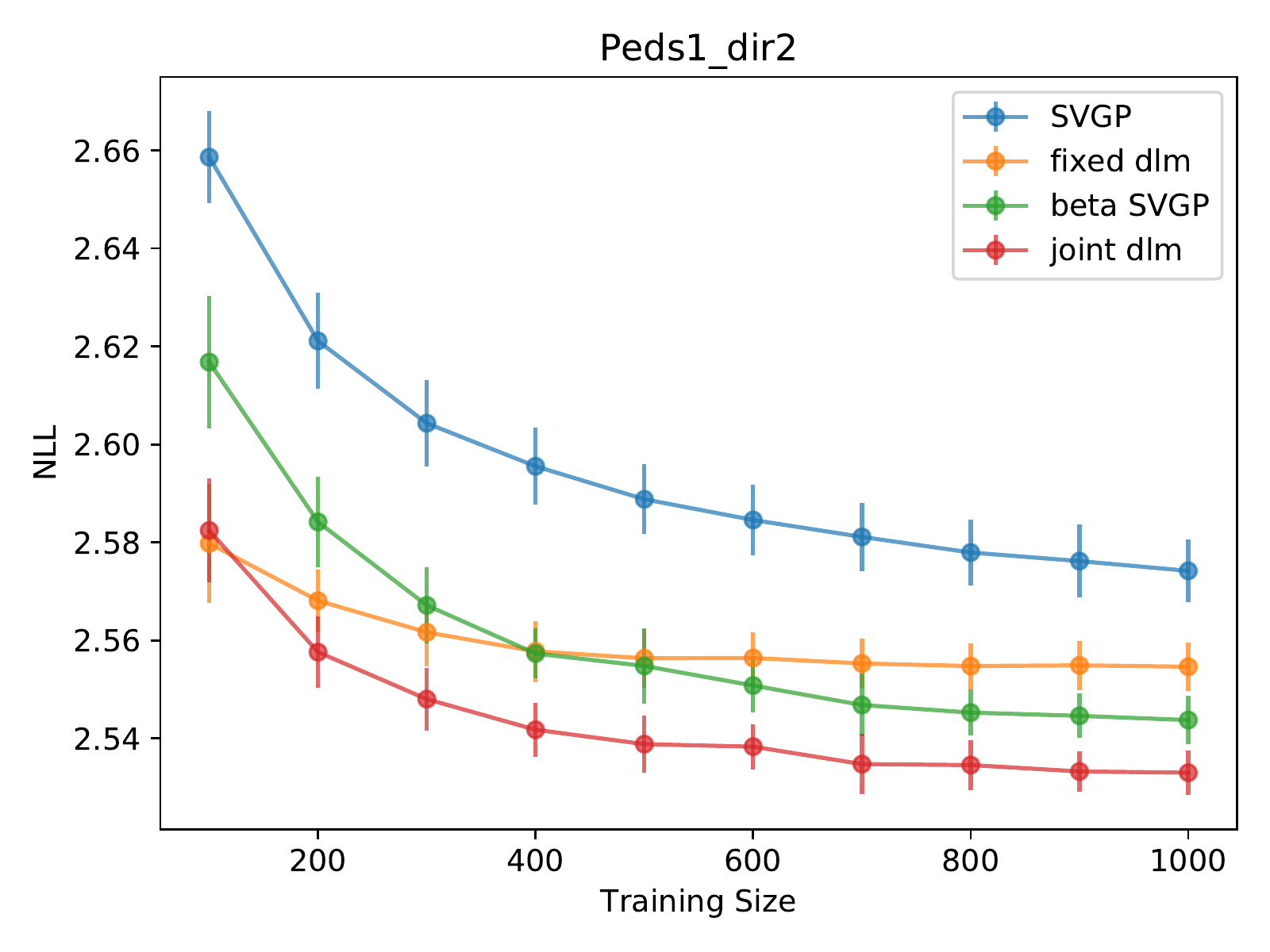}
    (e) count regression; log loss
\end{center}
\end{minipage}%
\begin{minipage}{0.330\linewidth}
\begin{center}
    \includegraphics[width=0.99\linewidth]{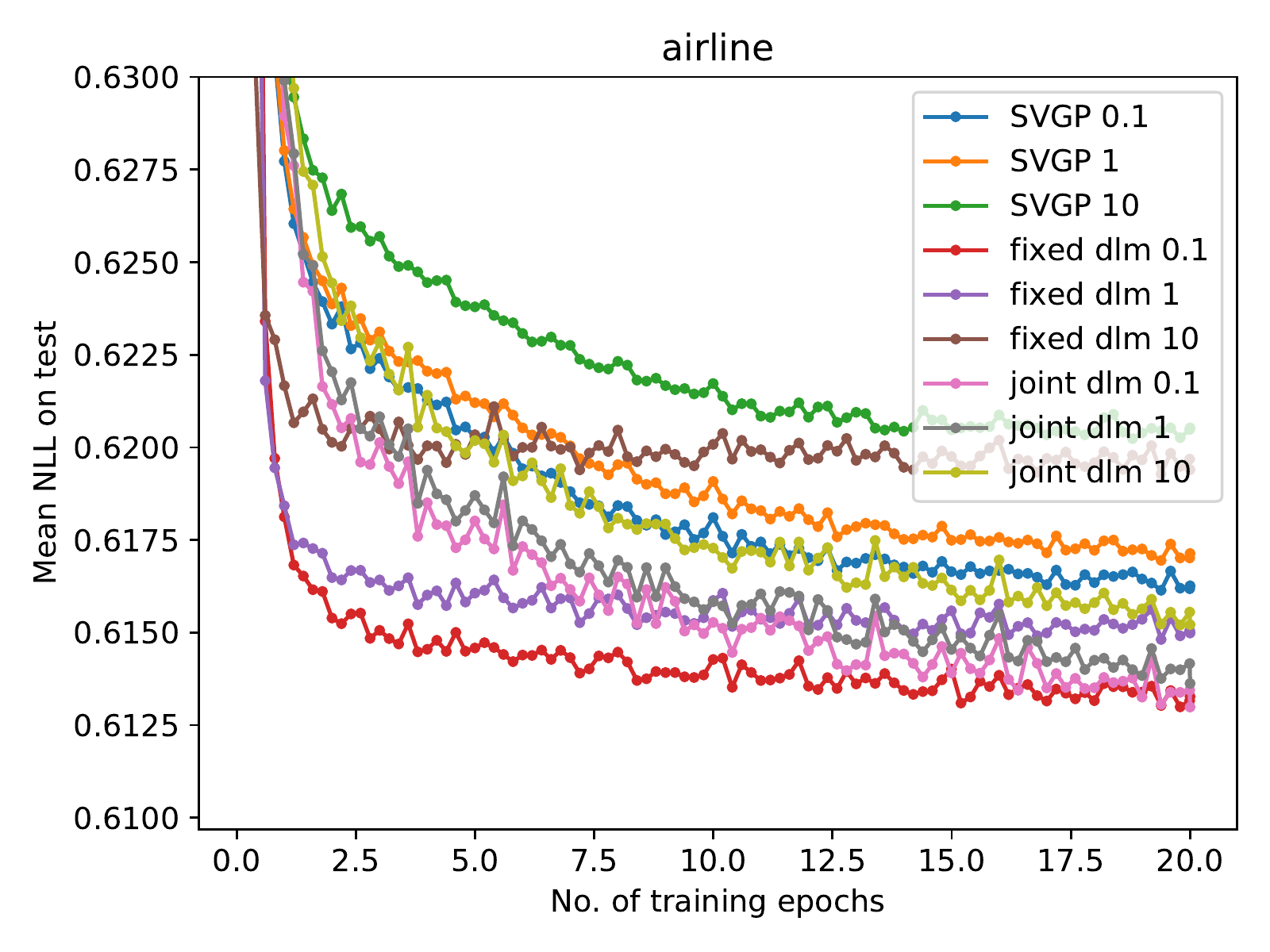}
    (f) ELBO vs.\ DLM on airline
\end{center}
\end{minipage}\quad
\begin{minipage}{0.330\linewidth}
\begin{center}
    \includegraphics[width=0.99\linewidth]{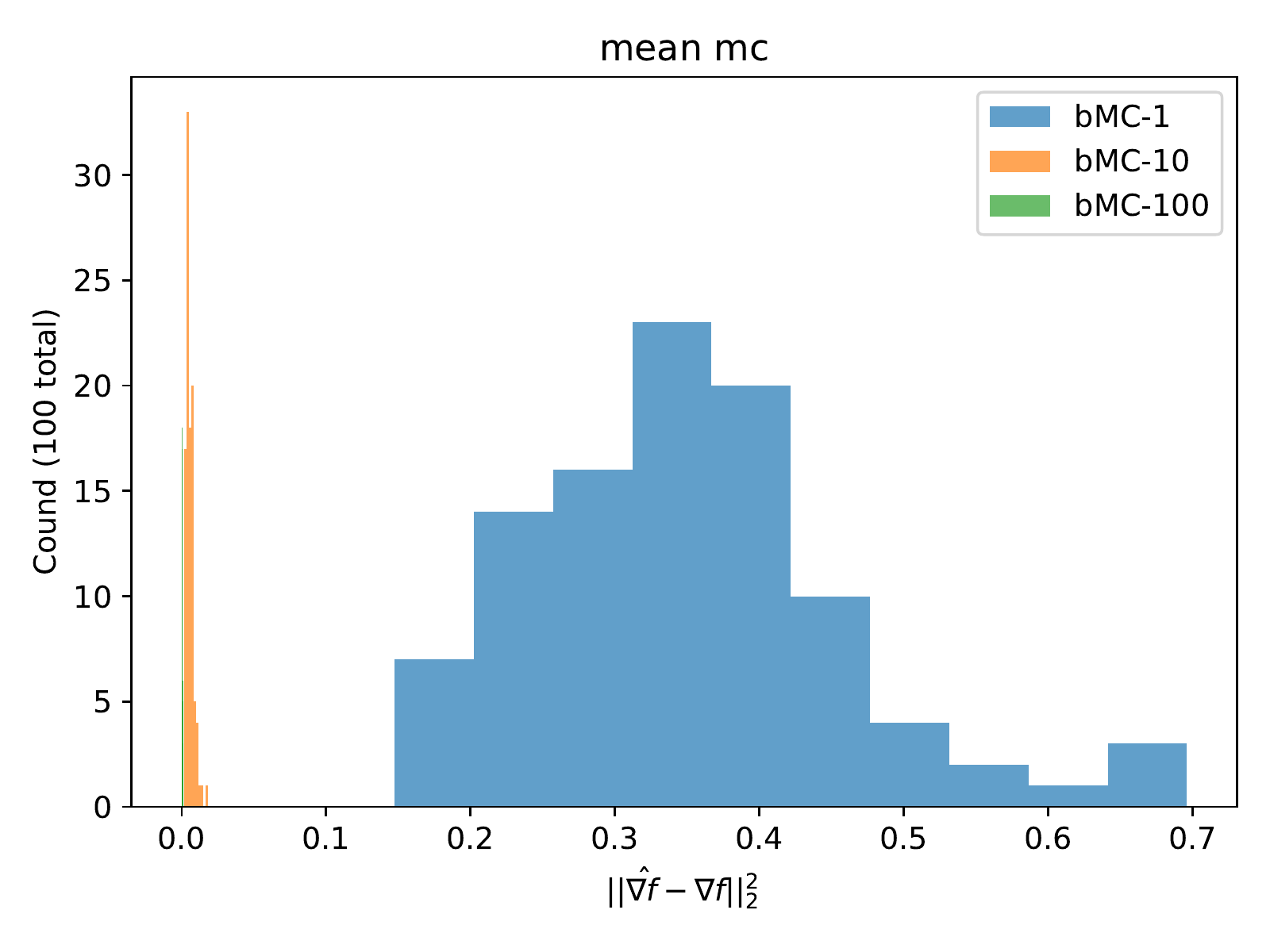}
   (g) bMC estimate of bias on abalone
\end{center}
\end{minipage}%
\begin{minipage}{0.330\linewidth}
\begin{center}
    \includegraphics[width=0.99\linewidth]{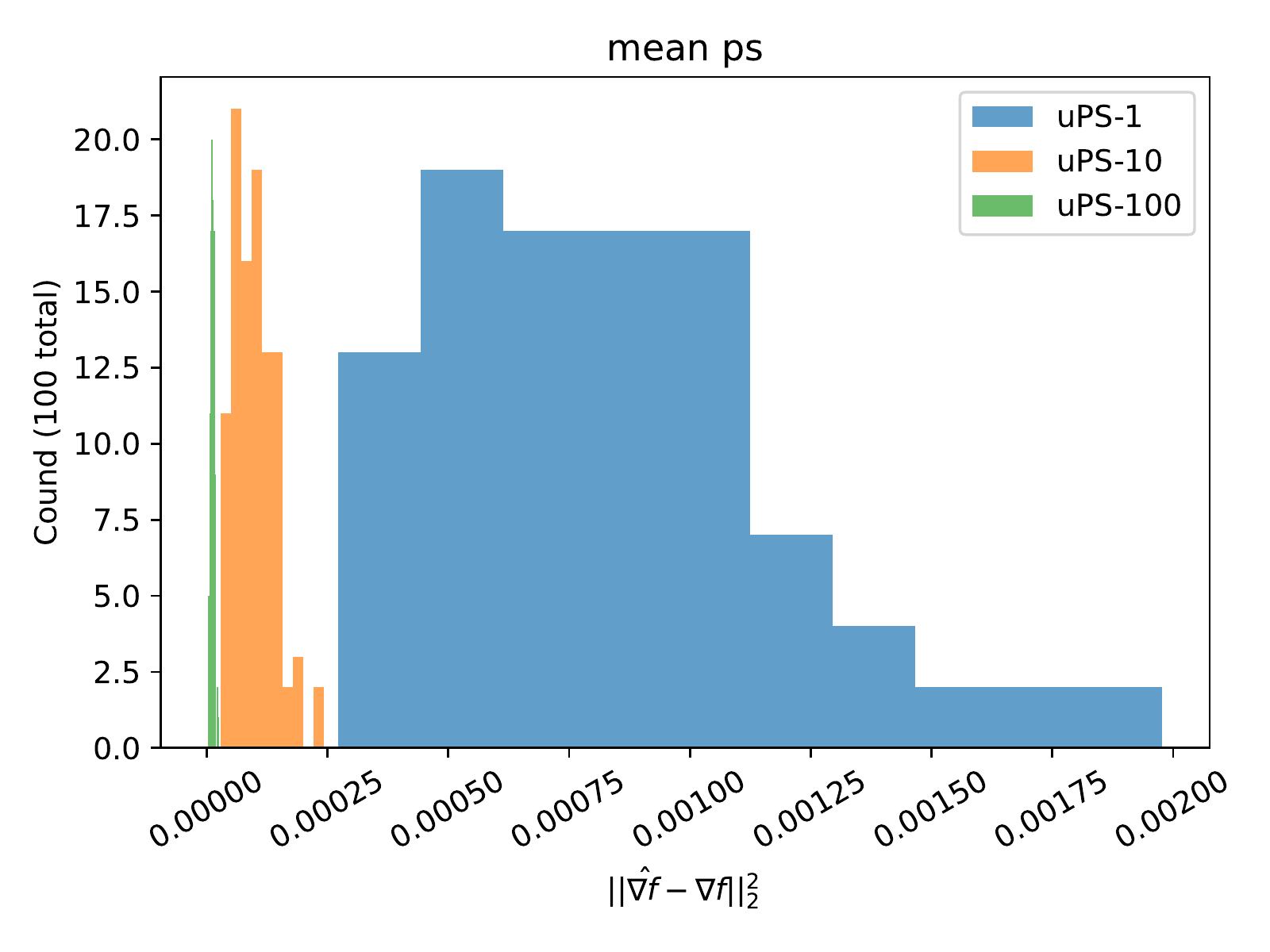}
   (h) uPS estimate of bias on abalone
\end{center}
\end{minipage}%
\begin{minipage}{0.330\linewidth}
\begin{center}
    \includegraphics[width=0.99\linewidth]{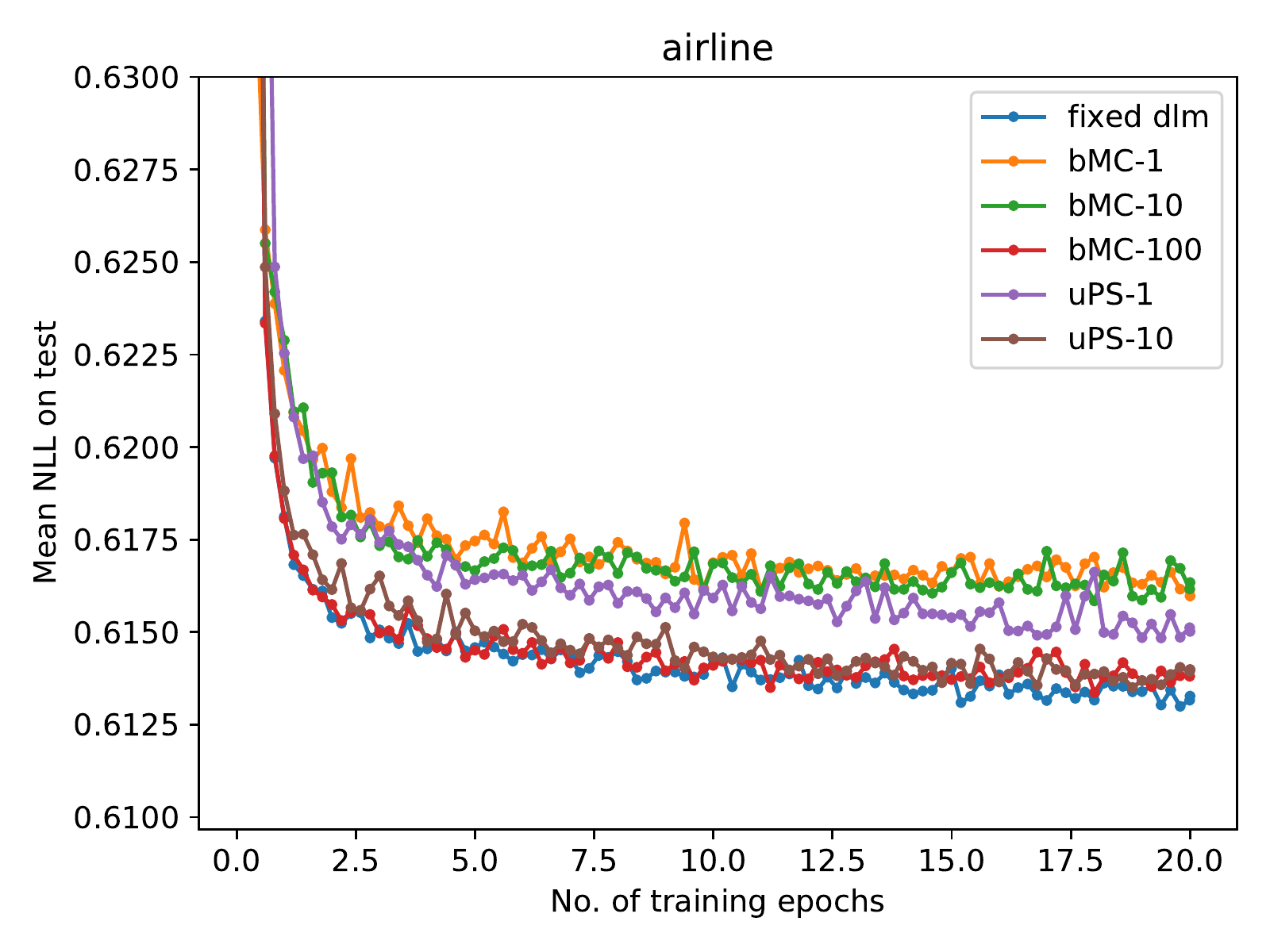}
   (i) DLM exact vs.\ sampling on airline
\end{center}
\end{minipage}\quad
\end{center}
}
\caption{Selected results. Description of individual plots is given in the text. }
\label{#1}
\end{figure*}
}

\newcommand{\PutRegressionLearningCurve}[1]{
\begin{figure*}[h]
\begin{center}
\begin{minipage}{0.433\linewidth}
\begin{center}
    \includegraphics[width=0.99\linewidth]{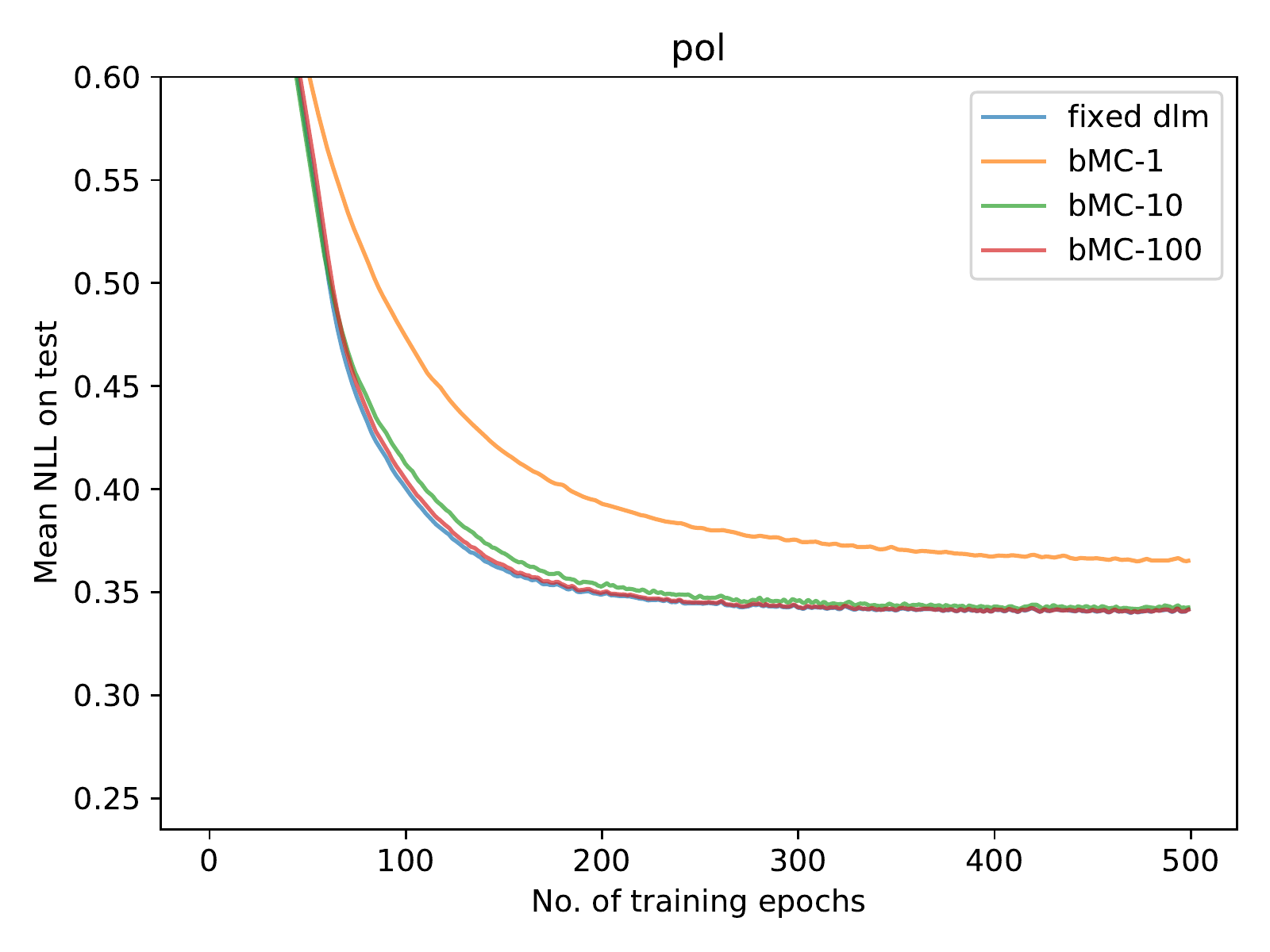}
\end{center}
\end{minipage}\quad
\begin{minipage}{0.433\linewidth}
\begin{center}
    \includegraphics[width=0.99\linewidth]{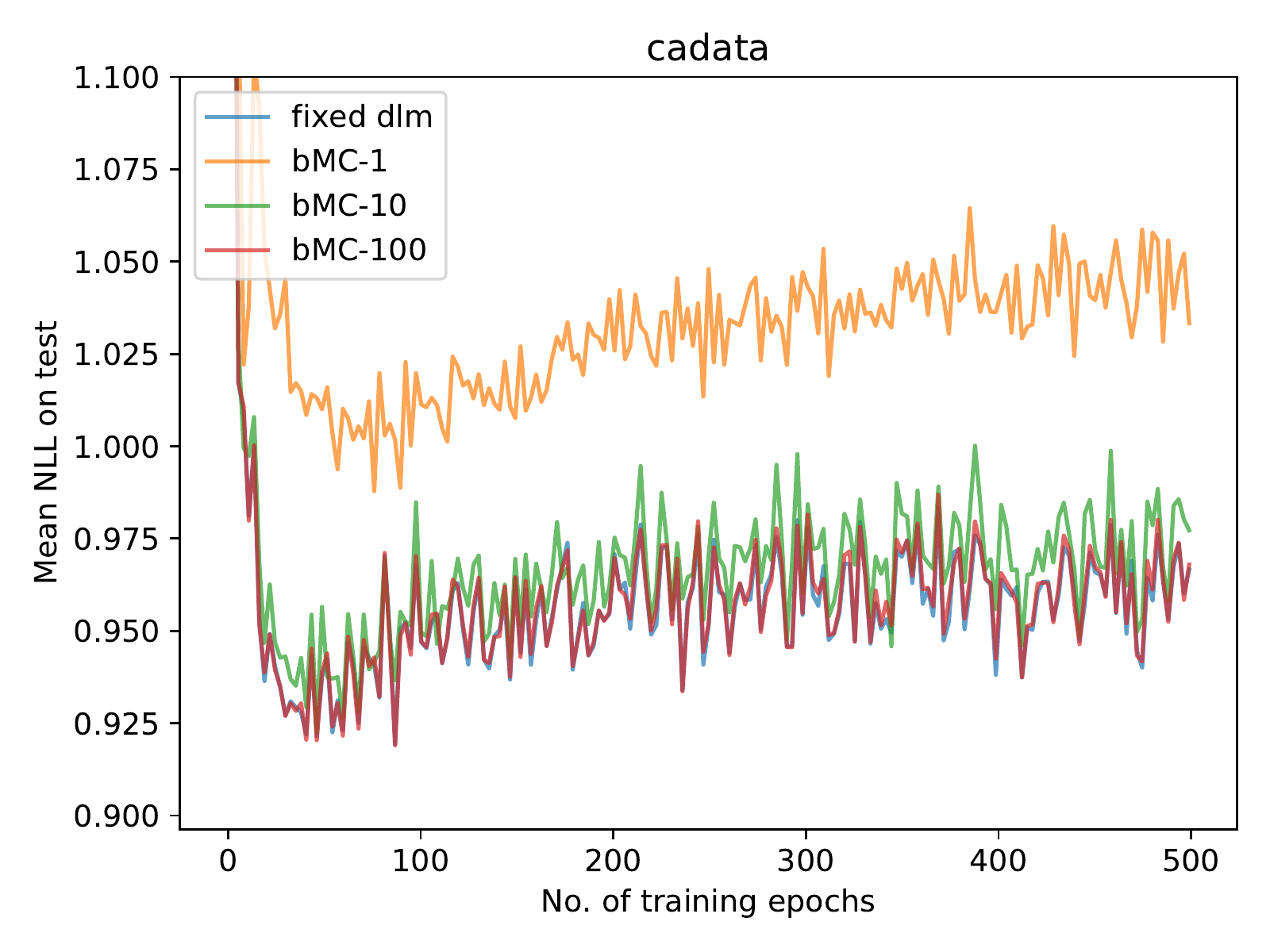}
\end{center}
\end{minipage}\quad
\begin{minipage}{0.433\linewidth}
\begin{center}
    \includegraphics[width=0.99\linewidth]{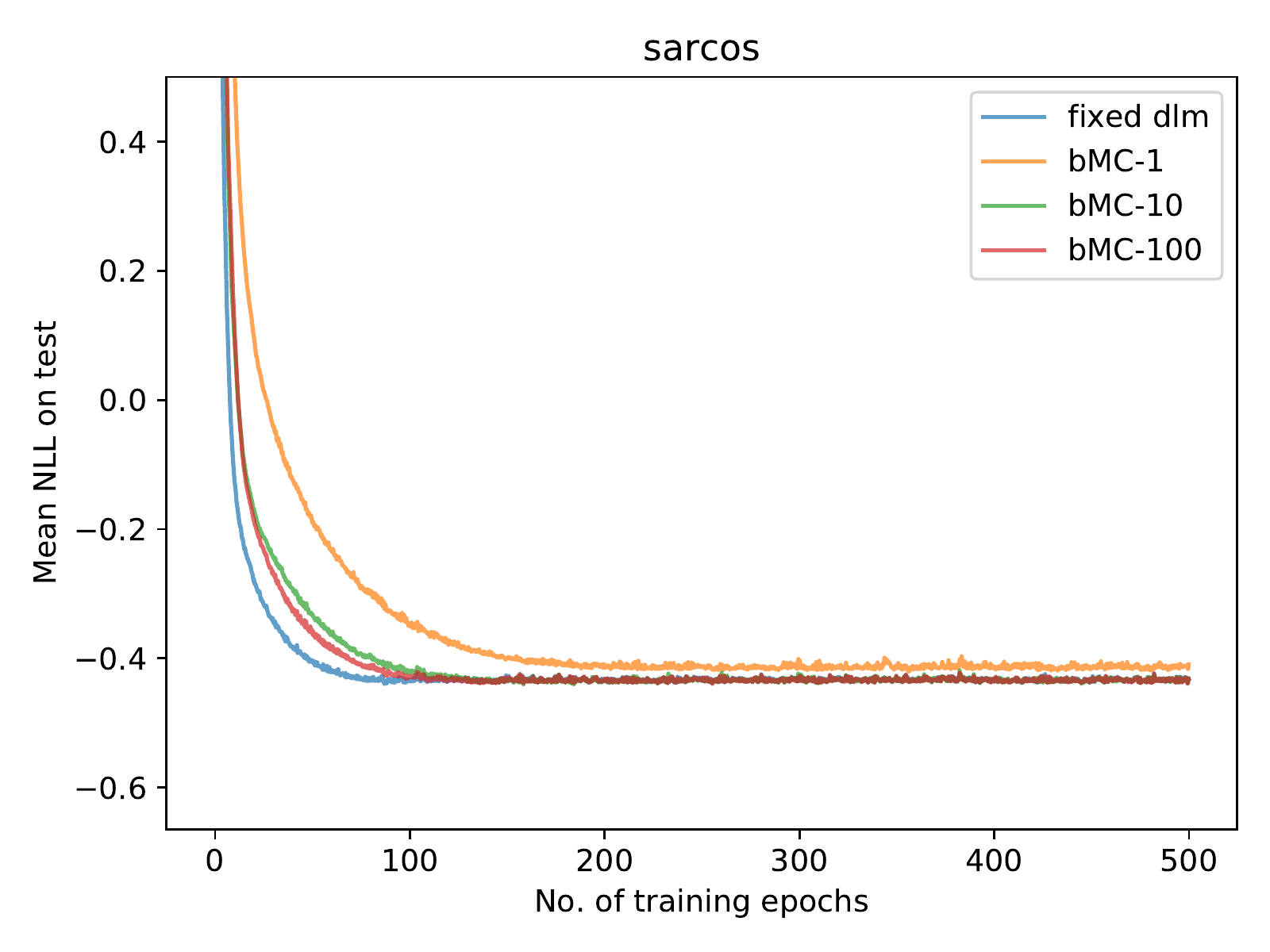}
\end{center}
\end{minipage}\quad
\begin{minipage}{0.433\linewidth}
\begin{center}
    \includegraphics[width=0.99\linewidth]{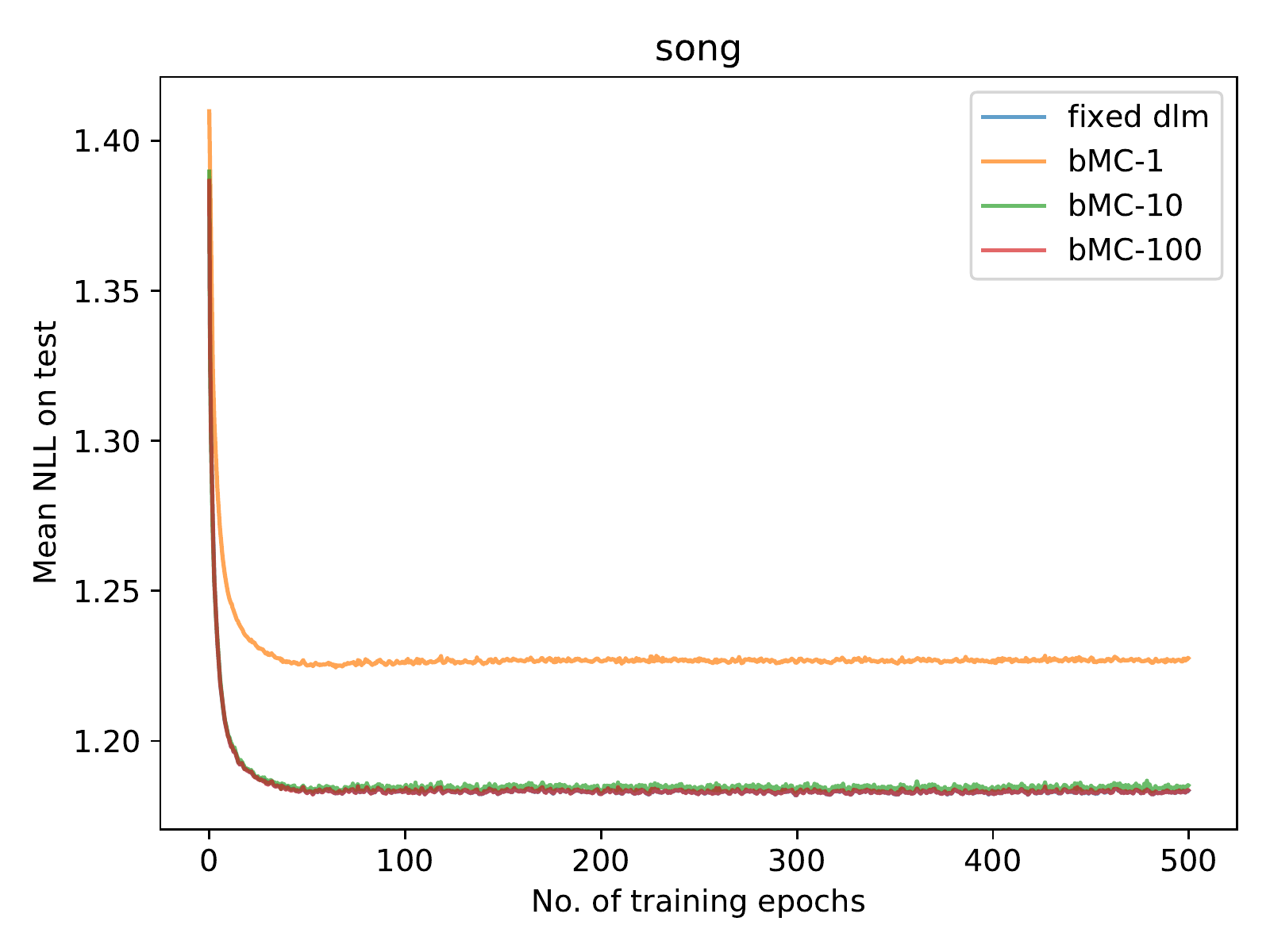}
\end{center}
\end{minipage}
\end{center}
\caption{Comparison of exact and bMC on four datasets for regression when $\beta=0.1$.}
\label{#1}
\end{figure*}

}

\newcommand{\PutPoissonLearningCurve}[1]{
\begin{figure*}[h]
\begin{center}
\begin{minipage}{0.310\linewidth}
\begin{center}
    \includegraphics[width=0.99\linewidth]{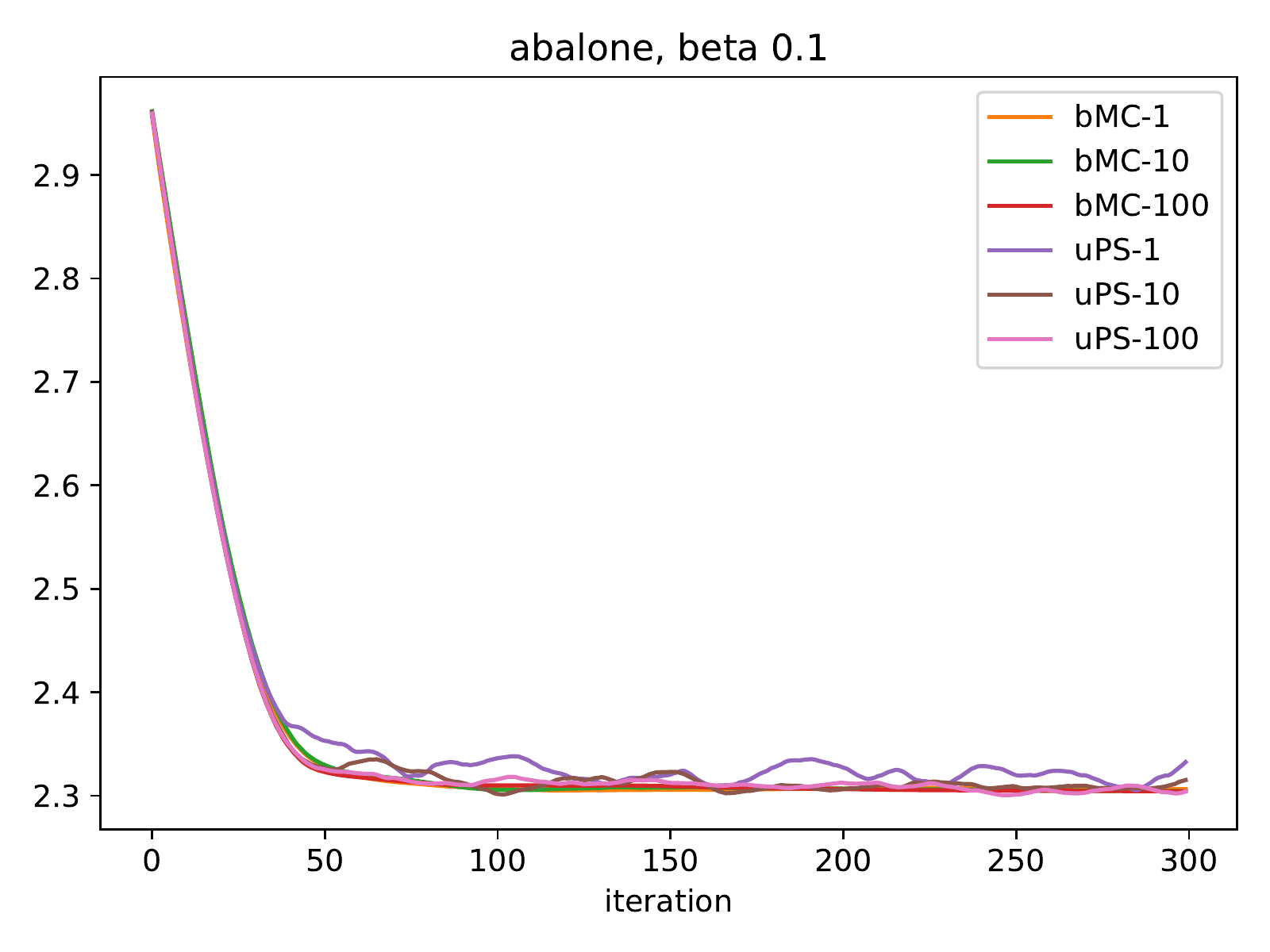}
\end{center}
\end{minipage}\quad
\begin{minipage}{0.310\linewidth}
\begin{center}
    \includegraphics[width=0.99\linewidth]{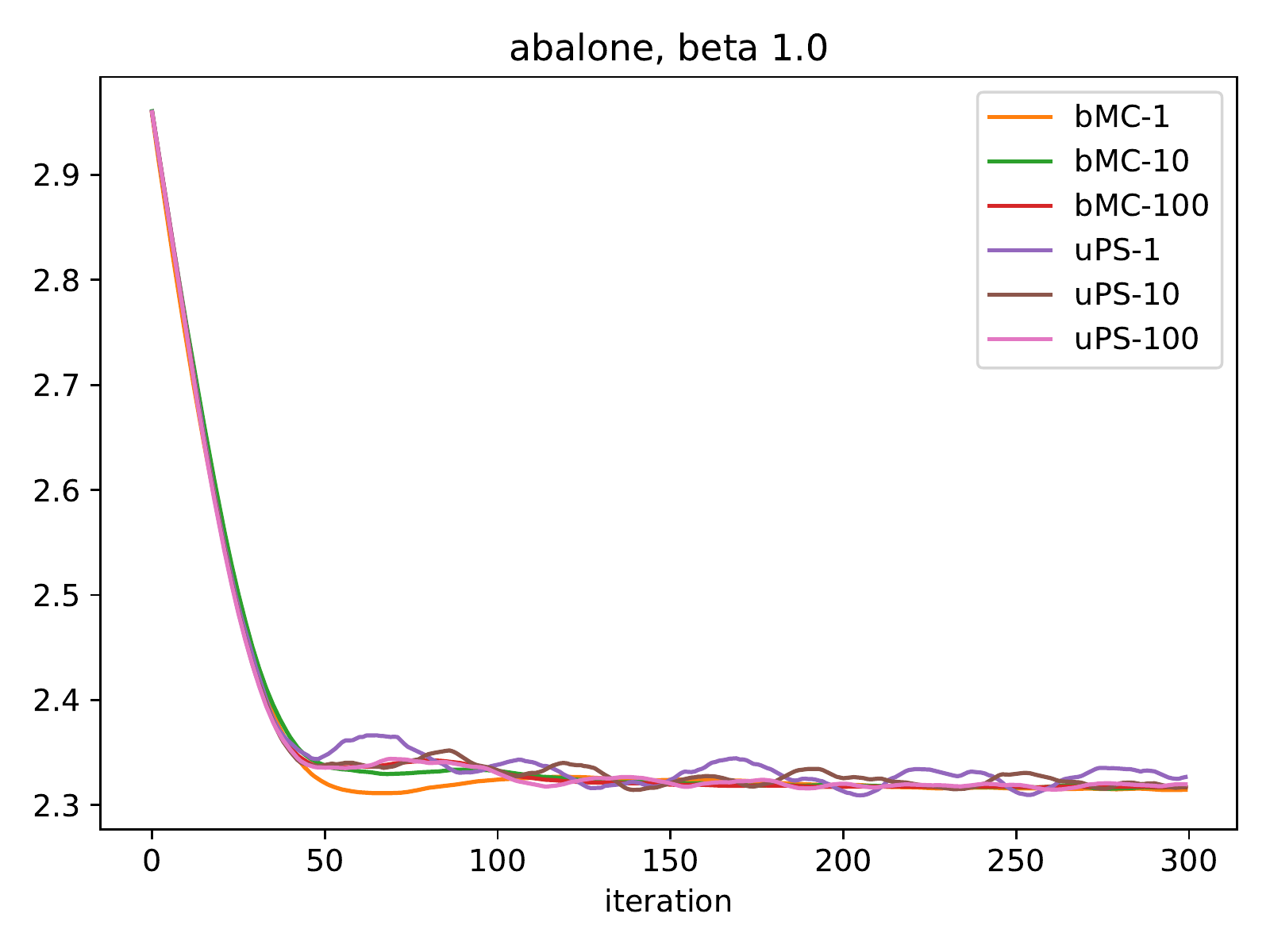}
\end{center}
\end{minipage}\quad
\begin{minipage}{0.310\linewidth}
\begin{center}
\includegraphics[width=0.99\linewidth]{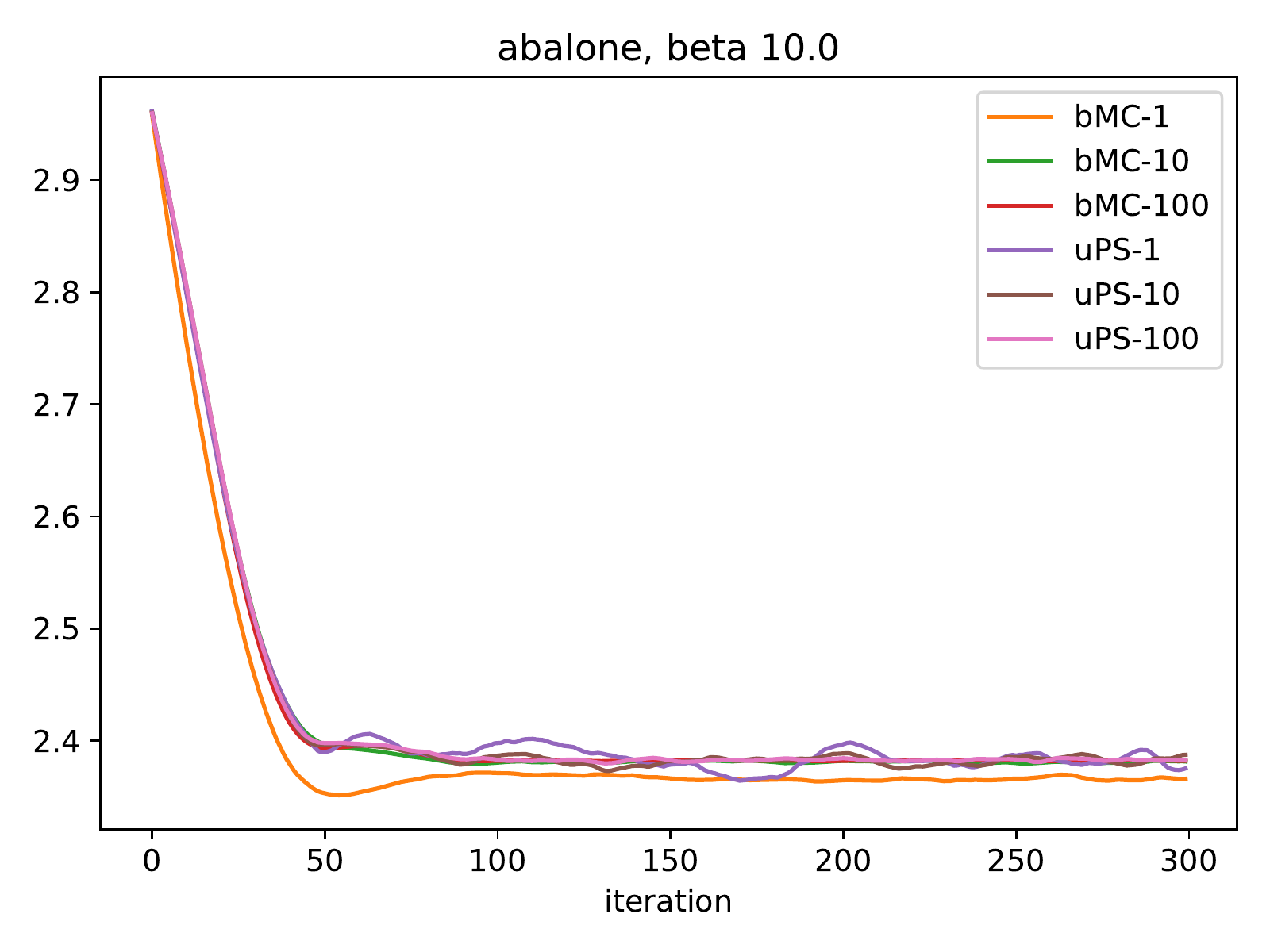}
\end{center}
\end{minipage}
\end{center}
\begin{center}
\begin{minipage}{0.310\linewidth}
\begin{center}
    \includegraphics[width=0.99\linewidth]{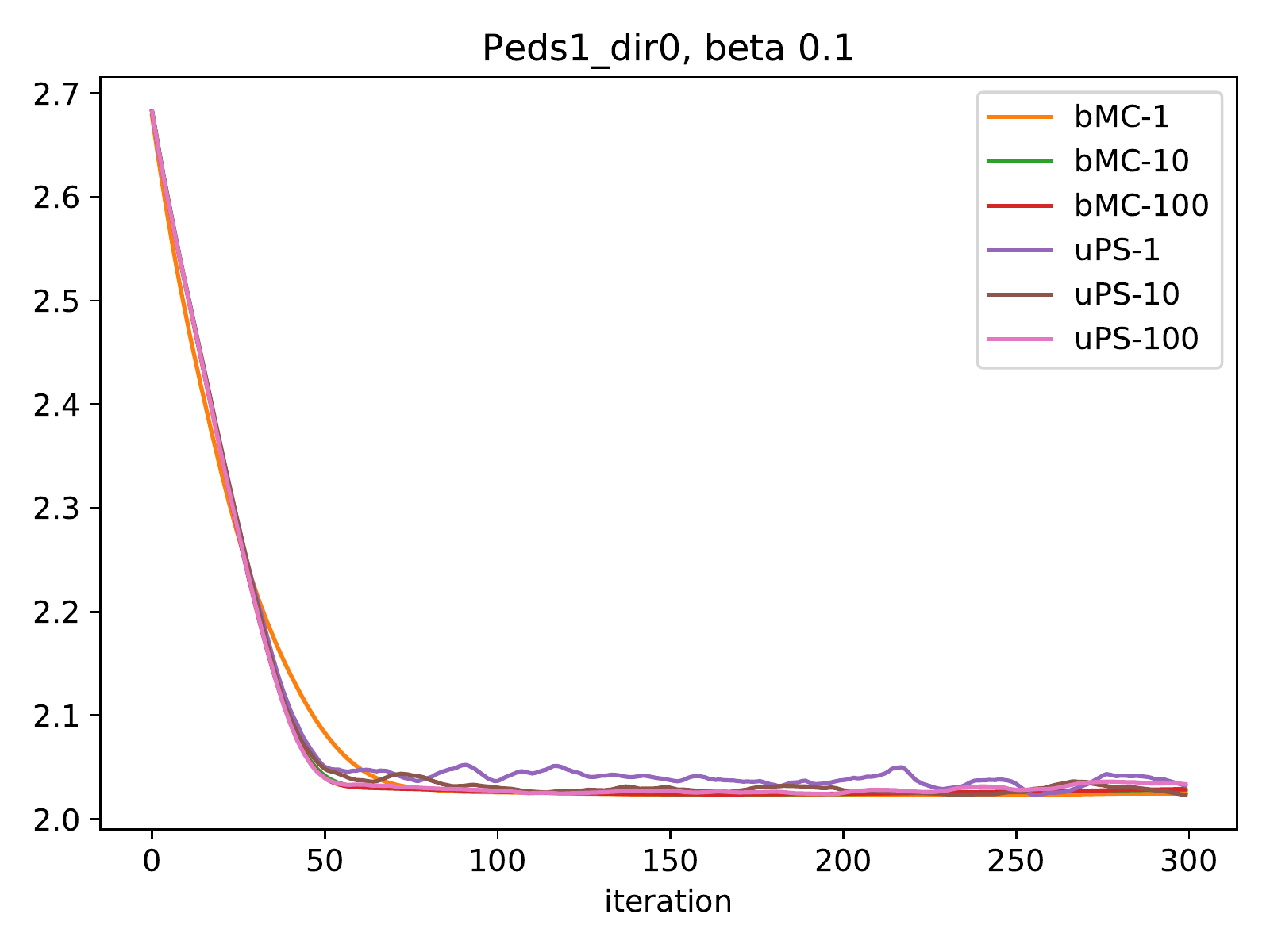}
\end{center}
\end{minipage}\quad
\begin{minipage}{0.310\linewidth}
\begin{center}
    \includegraphics[width=0.99\linewidth]{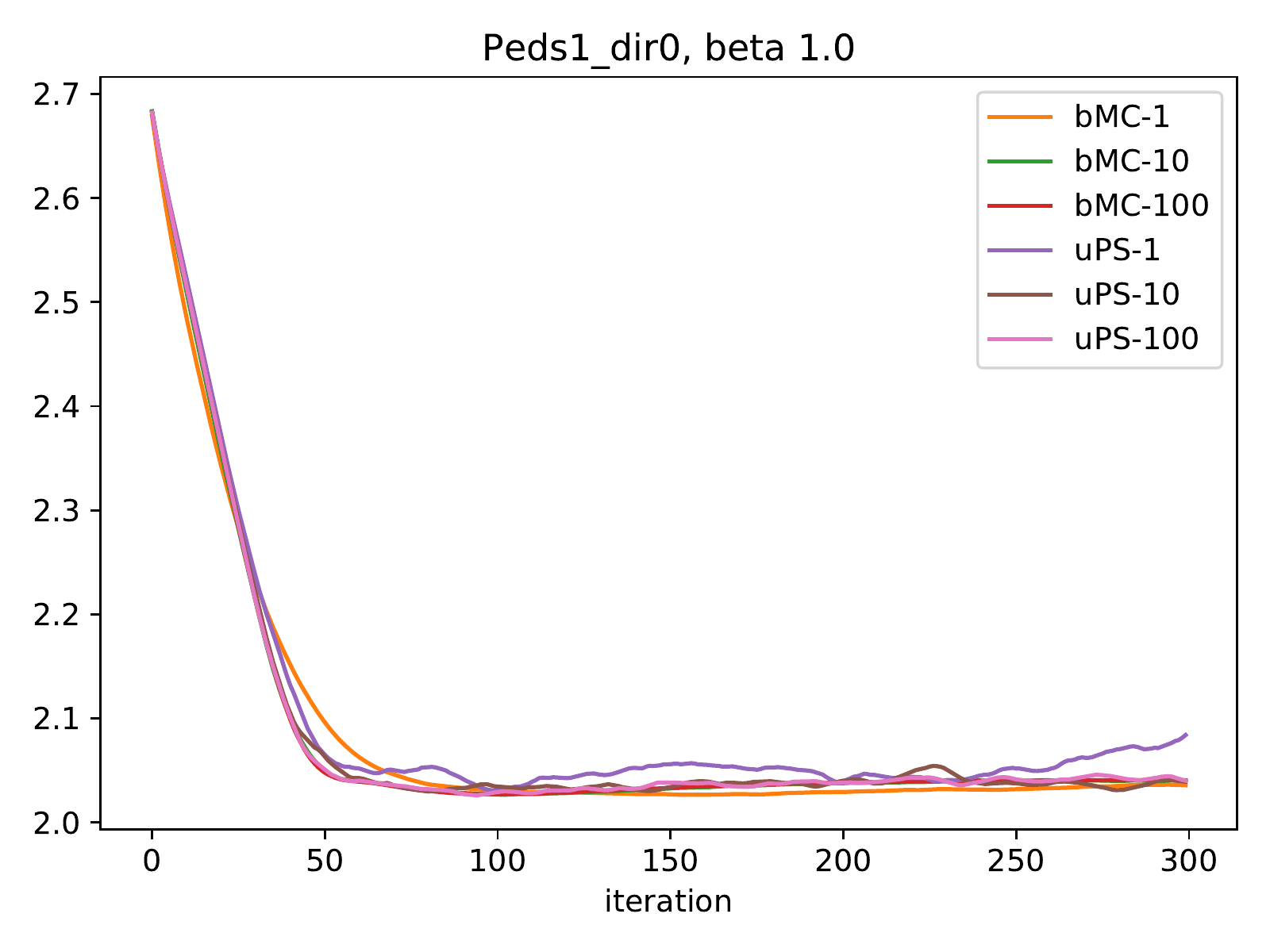}
\end{center}
\end{minipage}\quad
\begin{minipage}{0.310\linewidth}
\begin{center}
\includegraphics[width=0.99\linewidth]{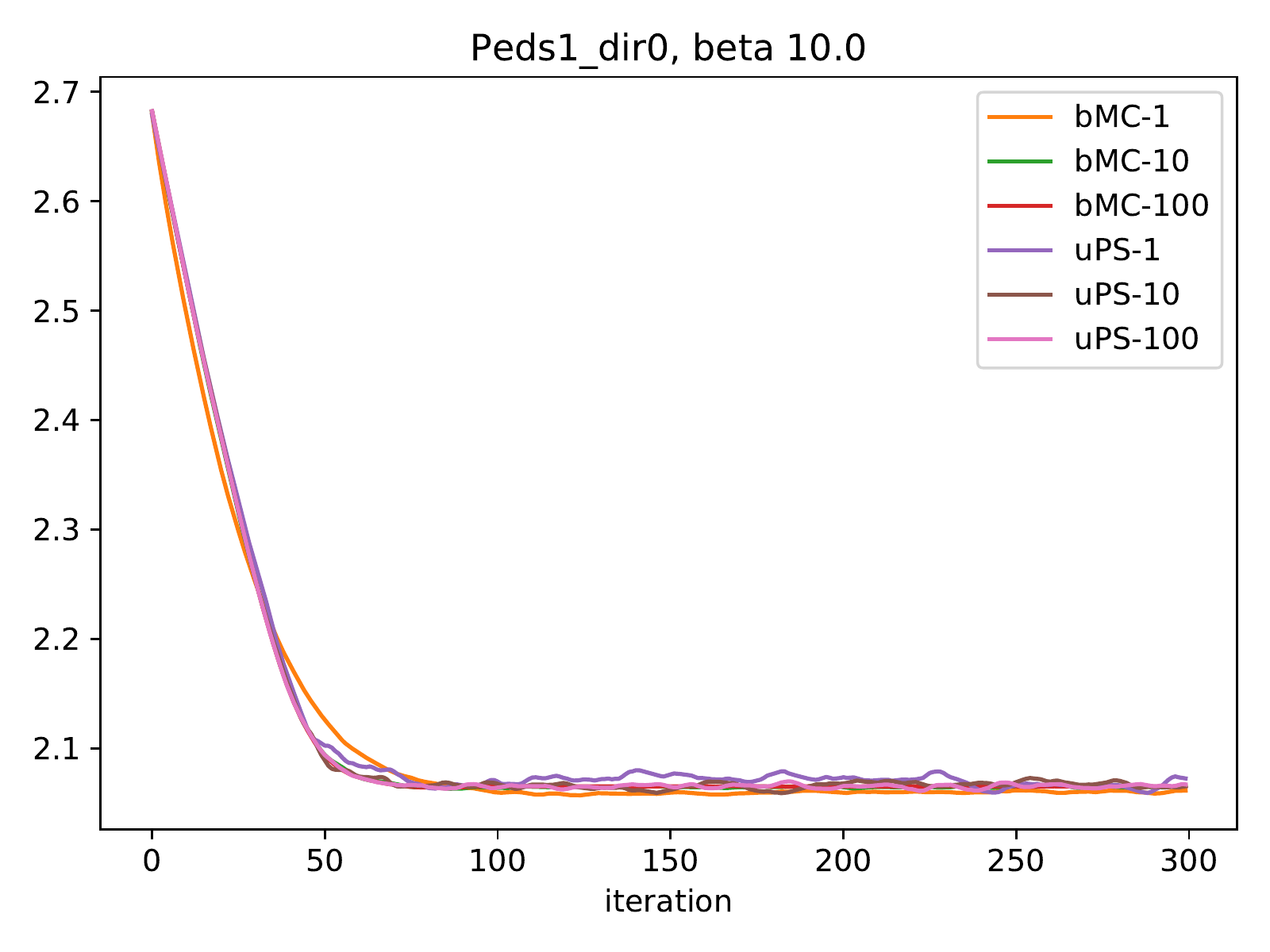}
\end{center}
\end{minipage}
\end{center}
\caption{Comparison of uPS and bMC on two datasets for Count Prediction.}
\label{#1}
\end{figure*}

}

\newcommand{\PutPoissonSmoothBMCLearningCurve}[1]{
\begin{figure*}[h]
\begin{center}
\begin{minipage}{0.310\linewidth}
\begin{center}
    \includegraphics[width=0.99\linewidth]{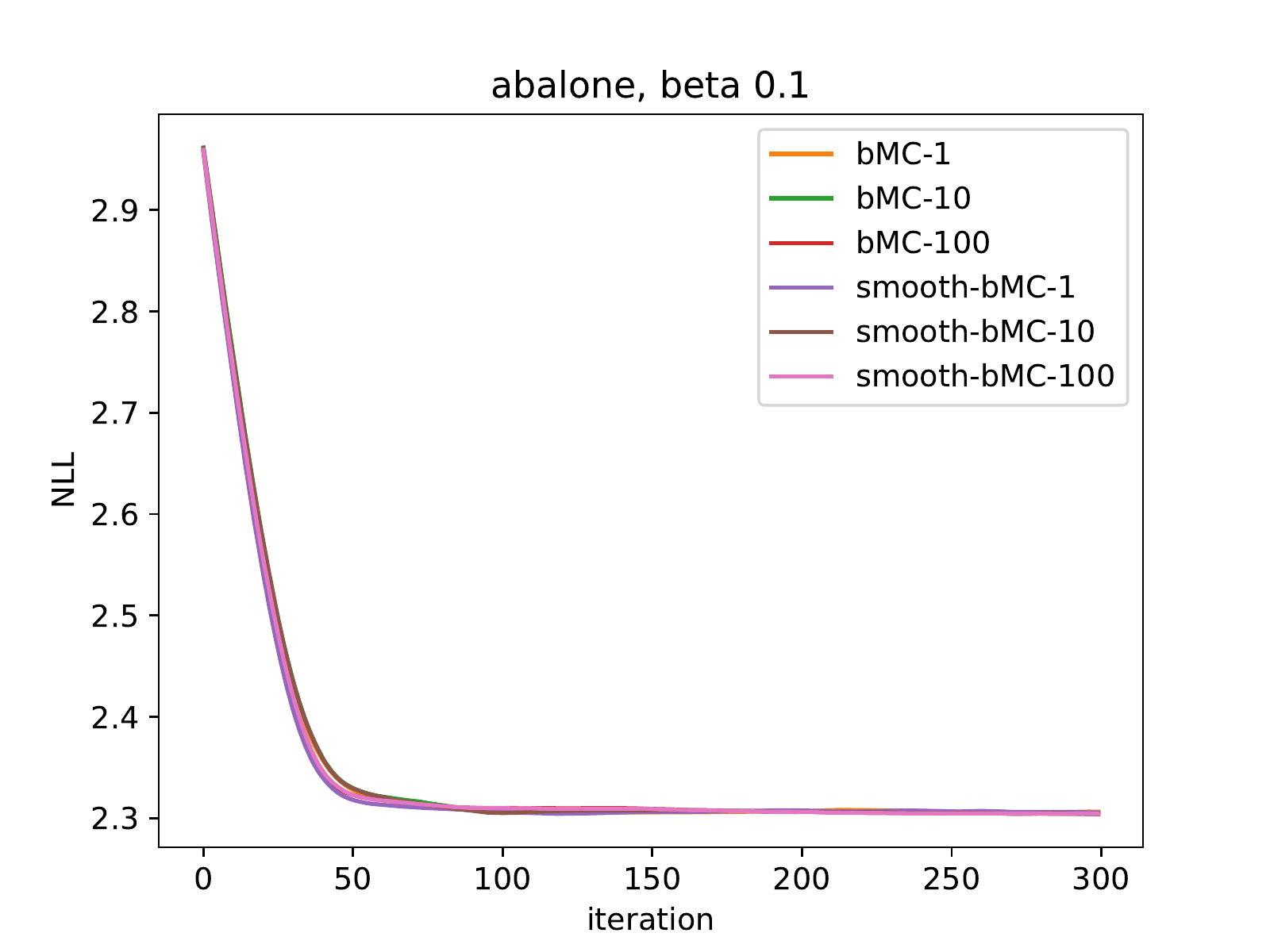}
\end{center}
\end{minipage}\quad
\begin{minipage}{0.310\linewidth}
\begin{center}
    \includegraphics[width=0.99\linewidth]{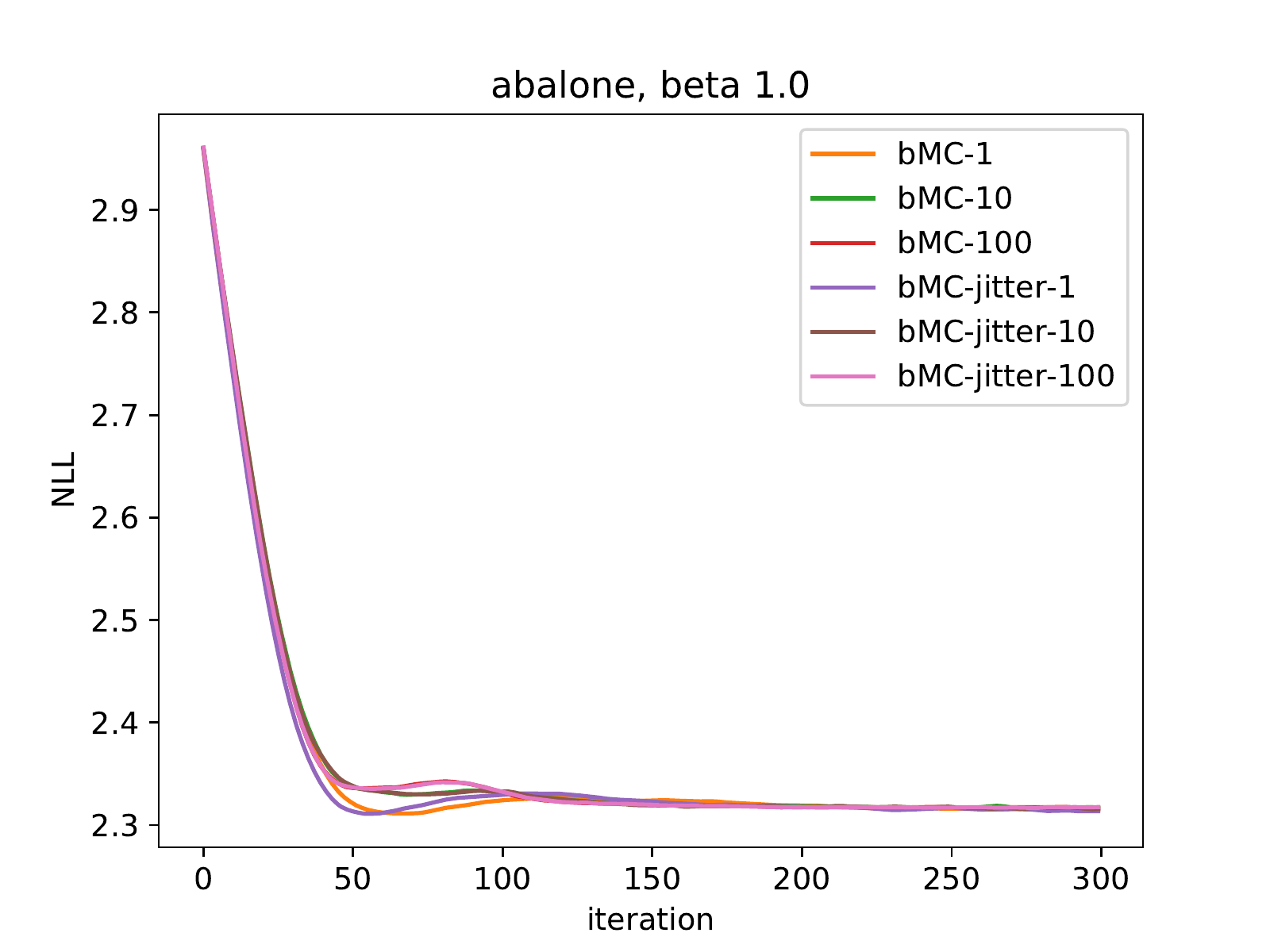}
\end{center}
\end{minipage}\quad
\begin{minipage}{0.310\linewidth}
\begin{center}
\includegraphics[width=0.99\linewidth]{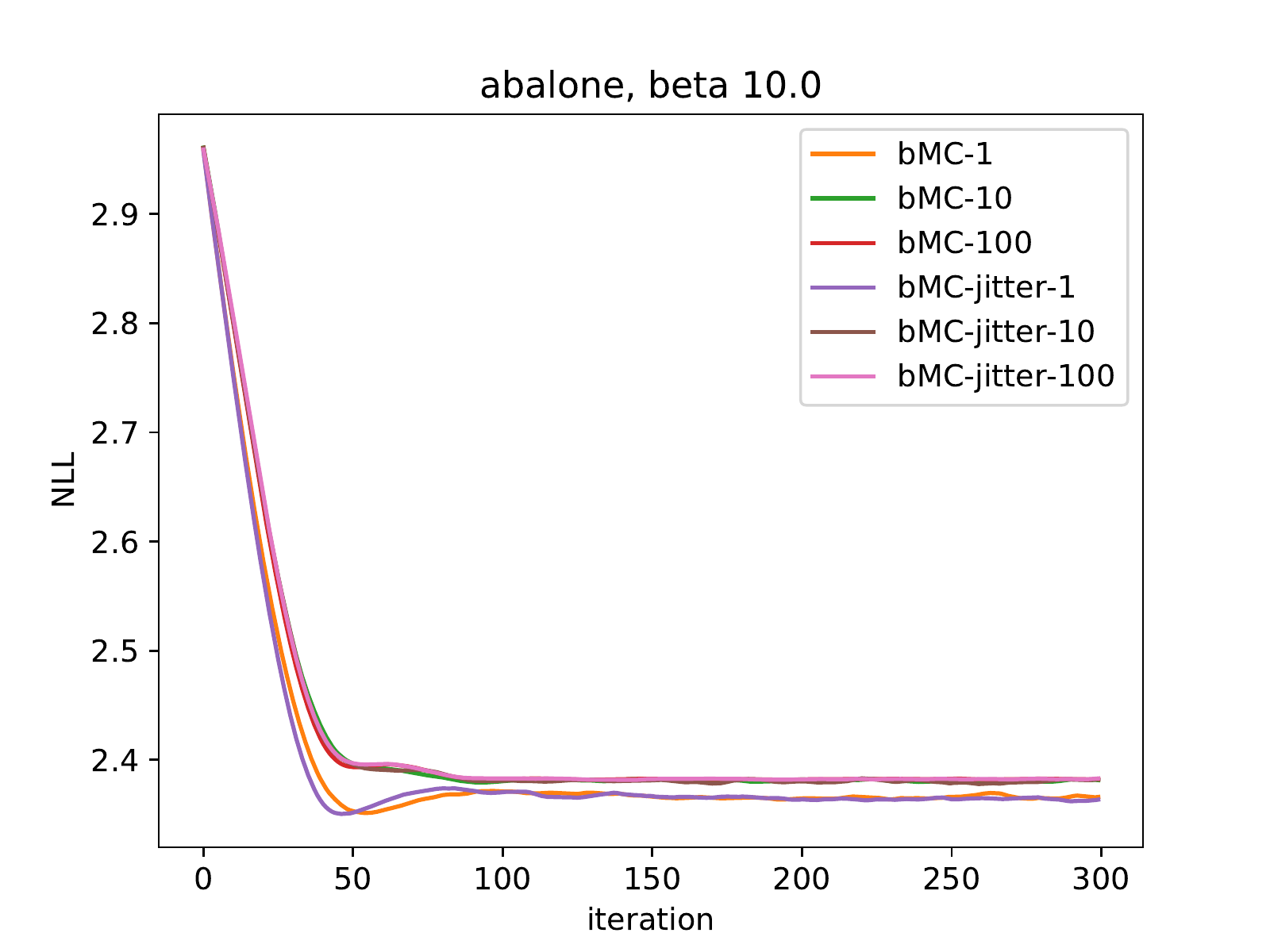}
\end{center}
\end{minipage}
\end{center}
\begin{center}
\begin{minipage}{0.310\linewidth}
\begin{center}
    \includegraphics[width=0.99\linewidth]{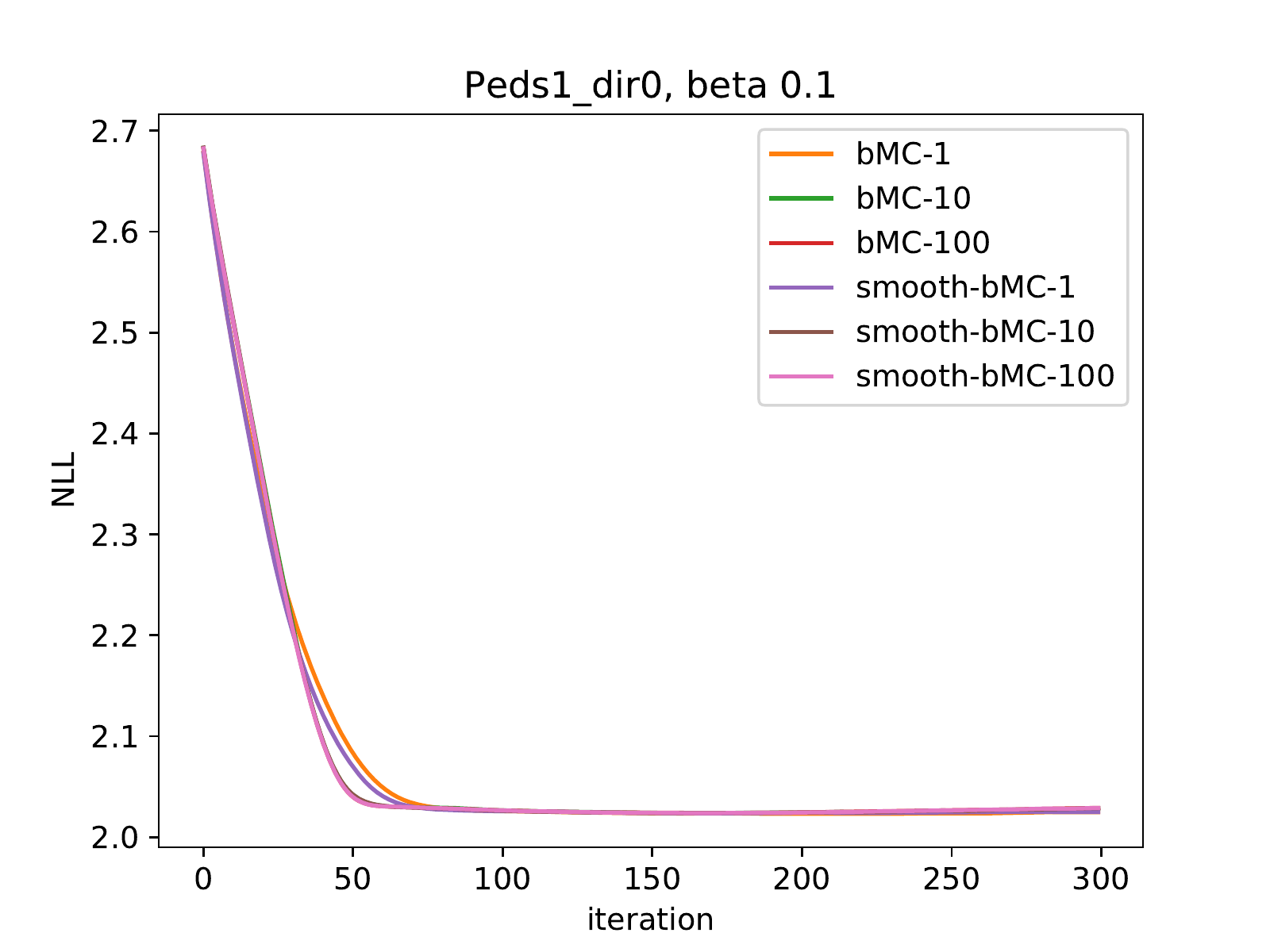}
\end{center}
\end{minipage}\quad
\begin{minipage}{0.310\linewidth}
\begin{center}
    \includegraphics[width=0.99\linewidth]{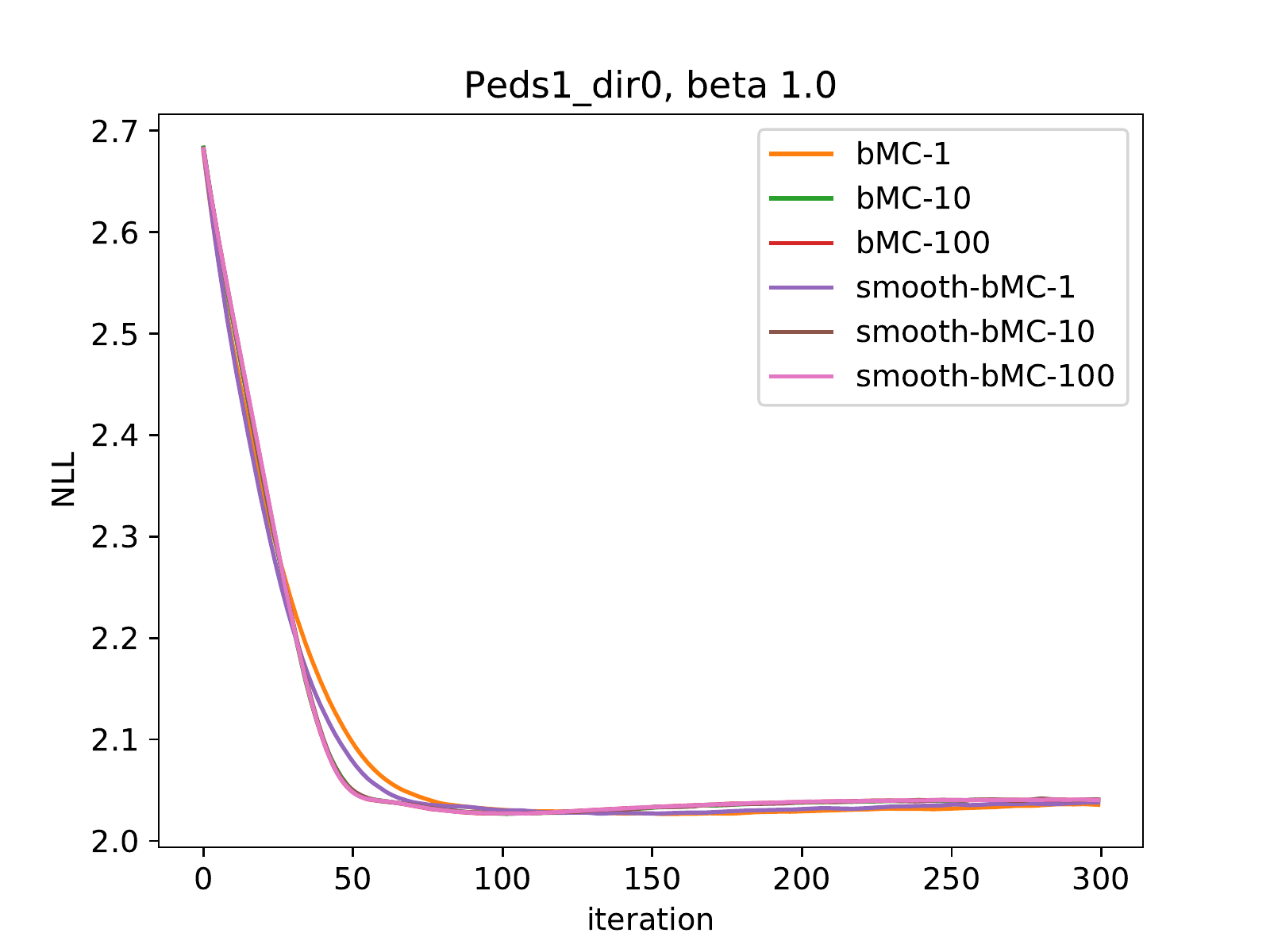}
\end{center}
\end{minipage}\quad
\begin{minipage}{0.310\linewidth}
\begin{center}
\includegraphics[width=0.99\linewidth]{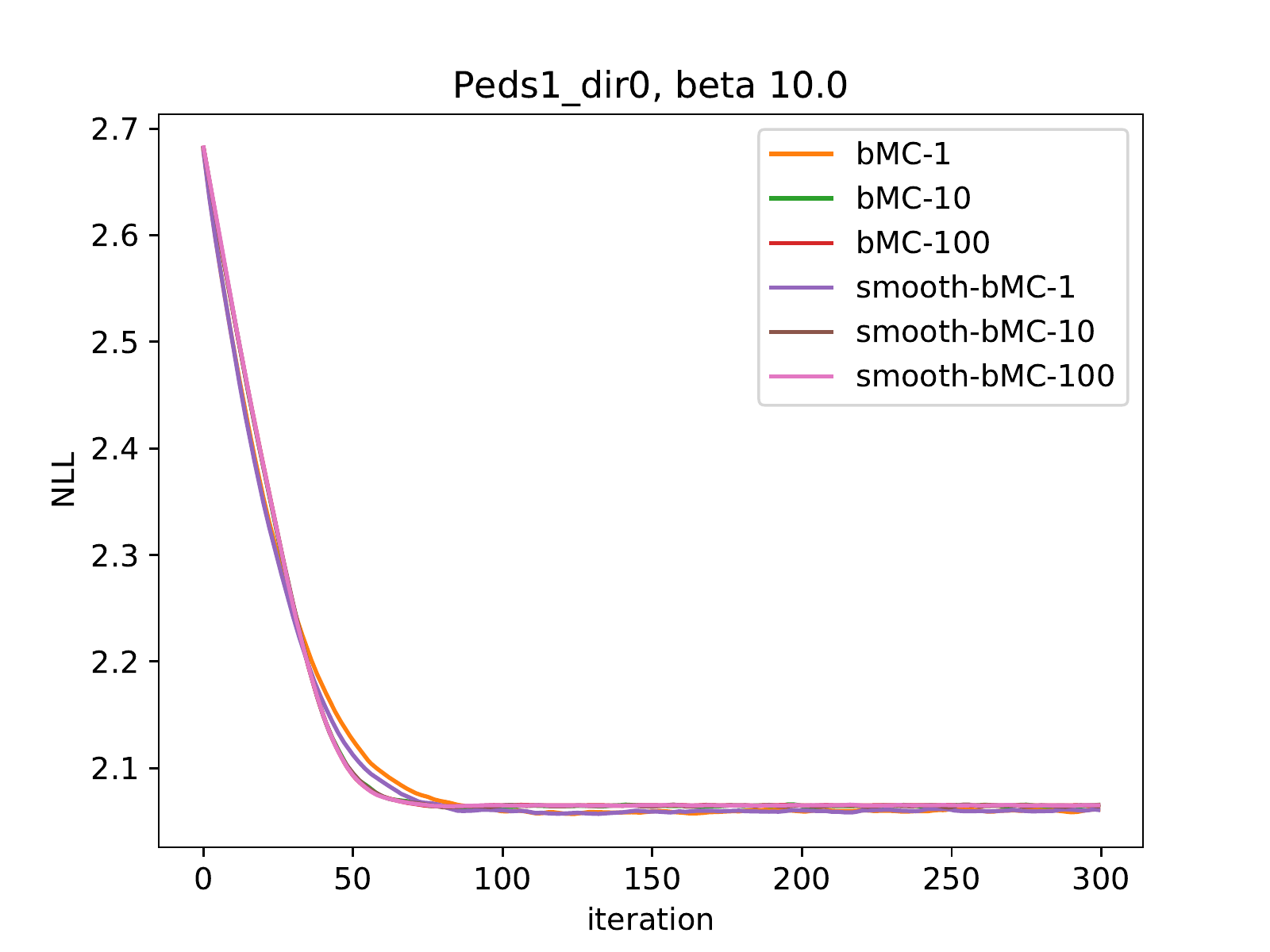}
\end{center}
\end{minipage}
\end{center}
\caption{Comparison of bMC and smooth-bMC on two datasets for Count Prediction.}
\label{#1}
\end{figure*}

}
\newcommand{\PutPoissonGradientsMean}[1]{
\begin{figure*}[h]

\begin{center}
\begin{minipage}{0.31\linewidth}
\begin{center}
    \includegraphics[width=0.99\linewidth]{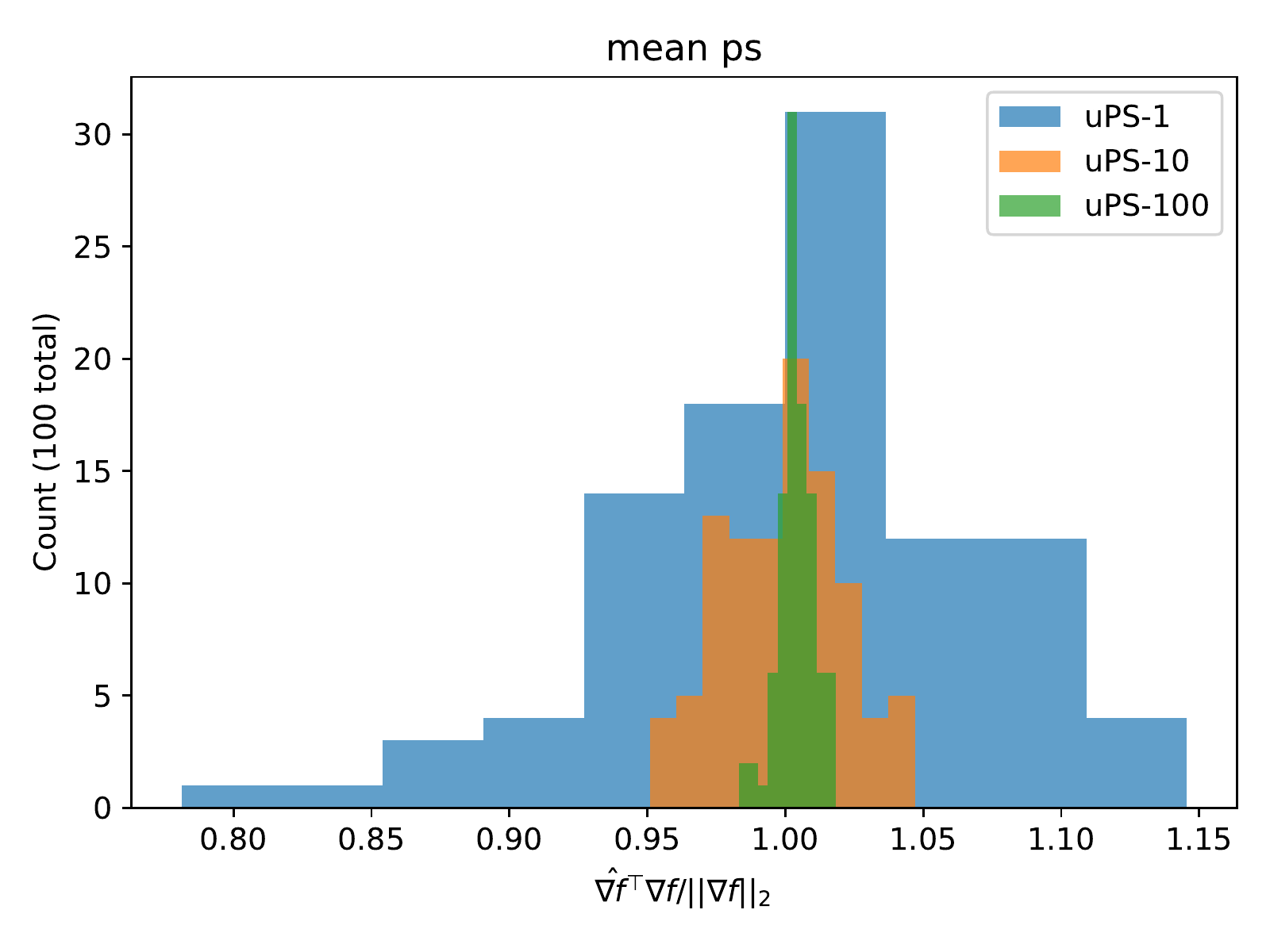}
\end{center}
\end{minipage}\quad
\begin{minipage}{0.31\linewidth}
\begin{center}
    \includegraphics[width=0.99\linewidth]{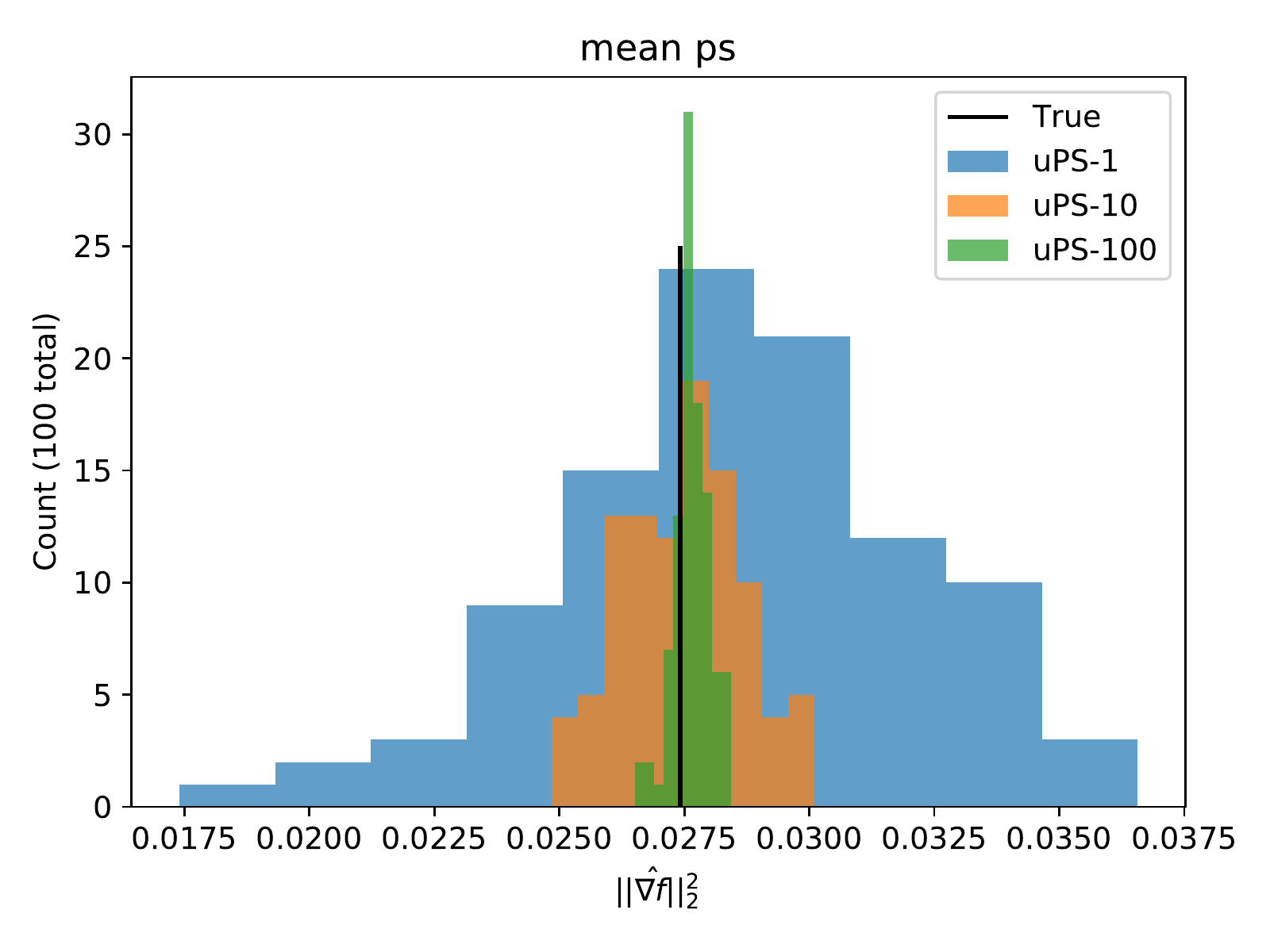}
\end{center}
\end{minipage}\quad
\begin{minipage}{0.31\linewidth}
\begin{center}
\includegraphics[width=0.99\linewidth]{figs/err-abalone-ps-m.pdf}
\end{center}
\end{minipage}
\end{center}

\begin{center}
\begin{minipage}{0.31\linewidth}
\begin{center}
    \includegraphics[width=0.99\linewidth]{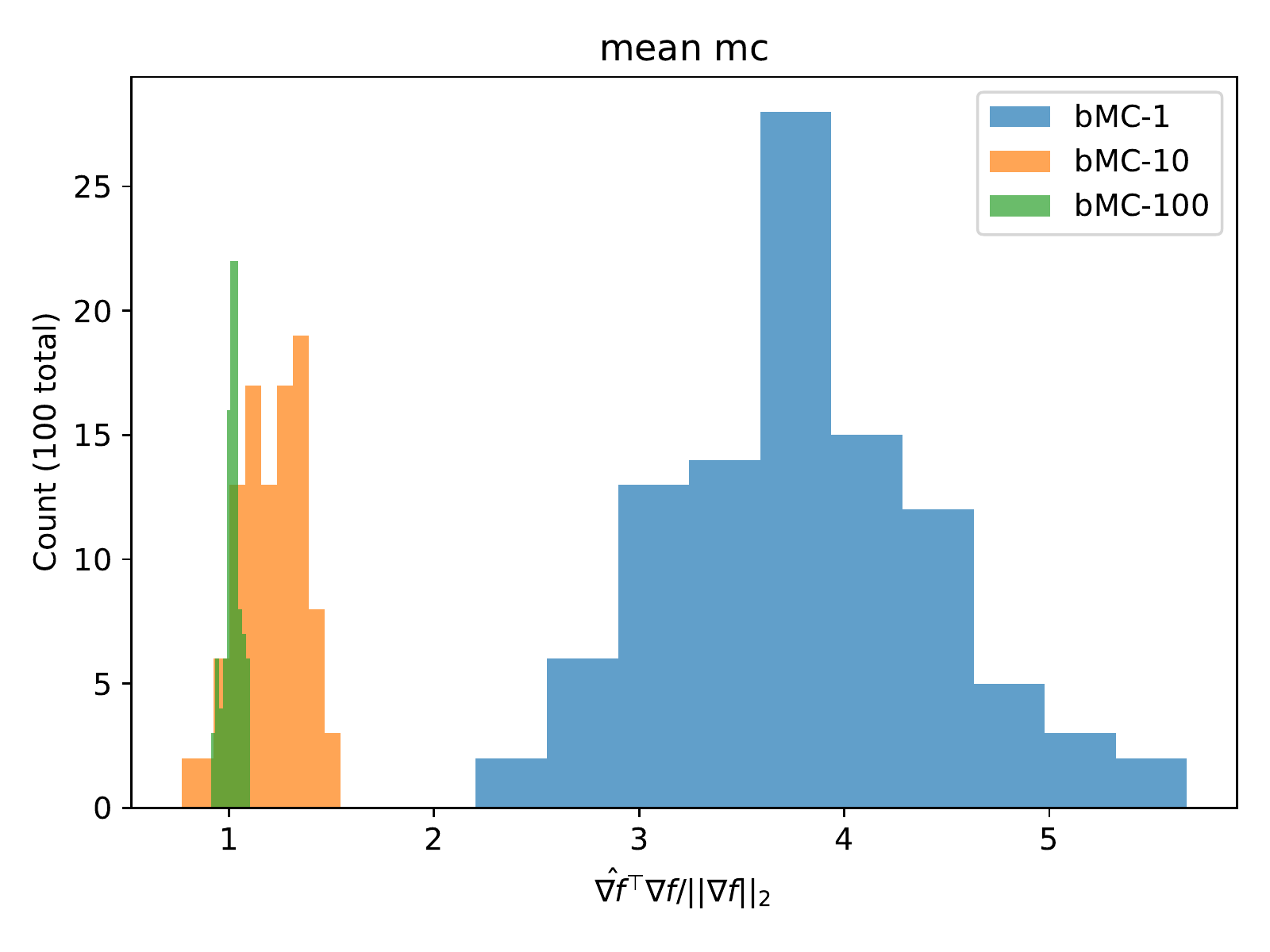}
\end{center}
\end{minipage}\quad
\begin{minipage}{0.31\linewidth}
\begin{center}
    \includegraphics[width=0.99\linewidth]{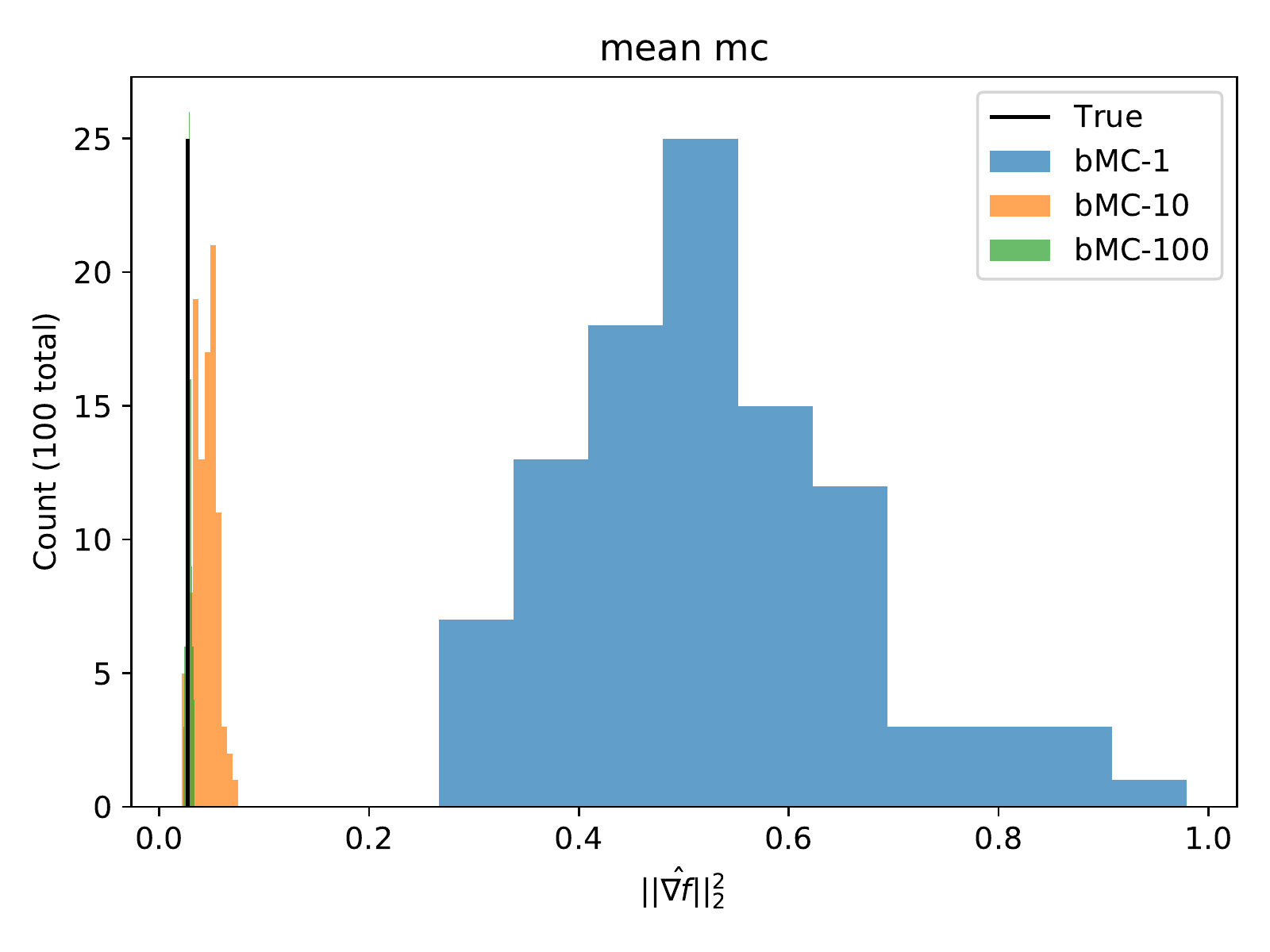}
\end{center}
\end{minipage}\quad
\begin{minipage}{0.31\linewidth}
\begin{center}
\includegraphics[width=0.99\linewidth]{figs/err-abalone-mc-m.pdf}
\end{center}
\end{minipage}
\end{center}

\begin{center}
\begin{minipage}{0.31\linewidth}
\begin{center}
    \includegraphics[width=0.99\linewidth]{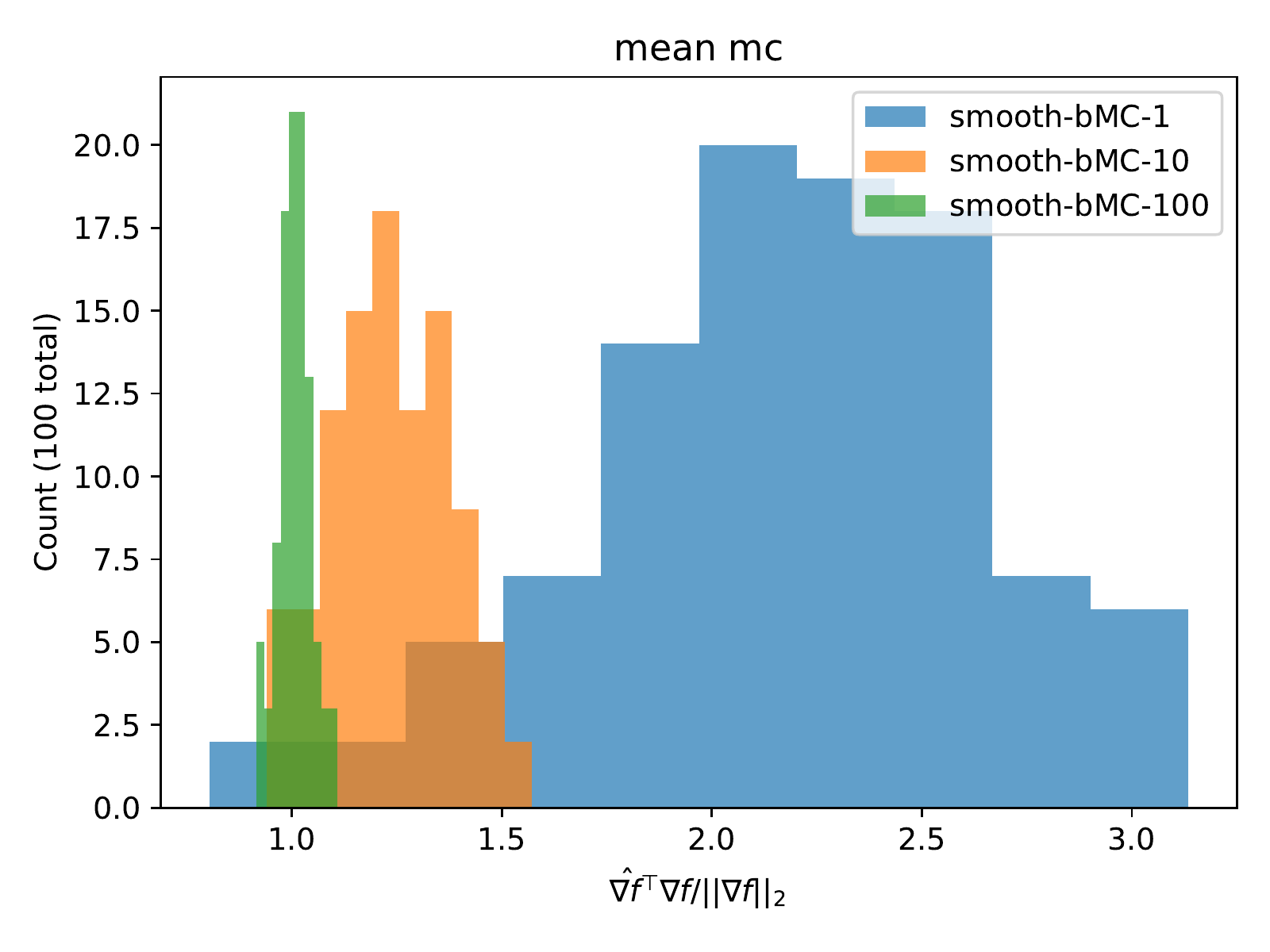}
\end{center}
\end{minipage}\quad
\begin{minipage}{0.31\linewidth}
\begin{center}
    \includegraphics[width=0.99\linewidth]{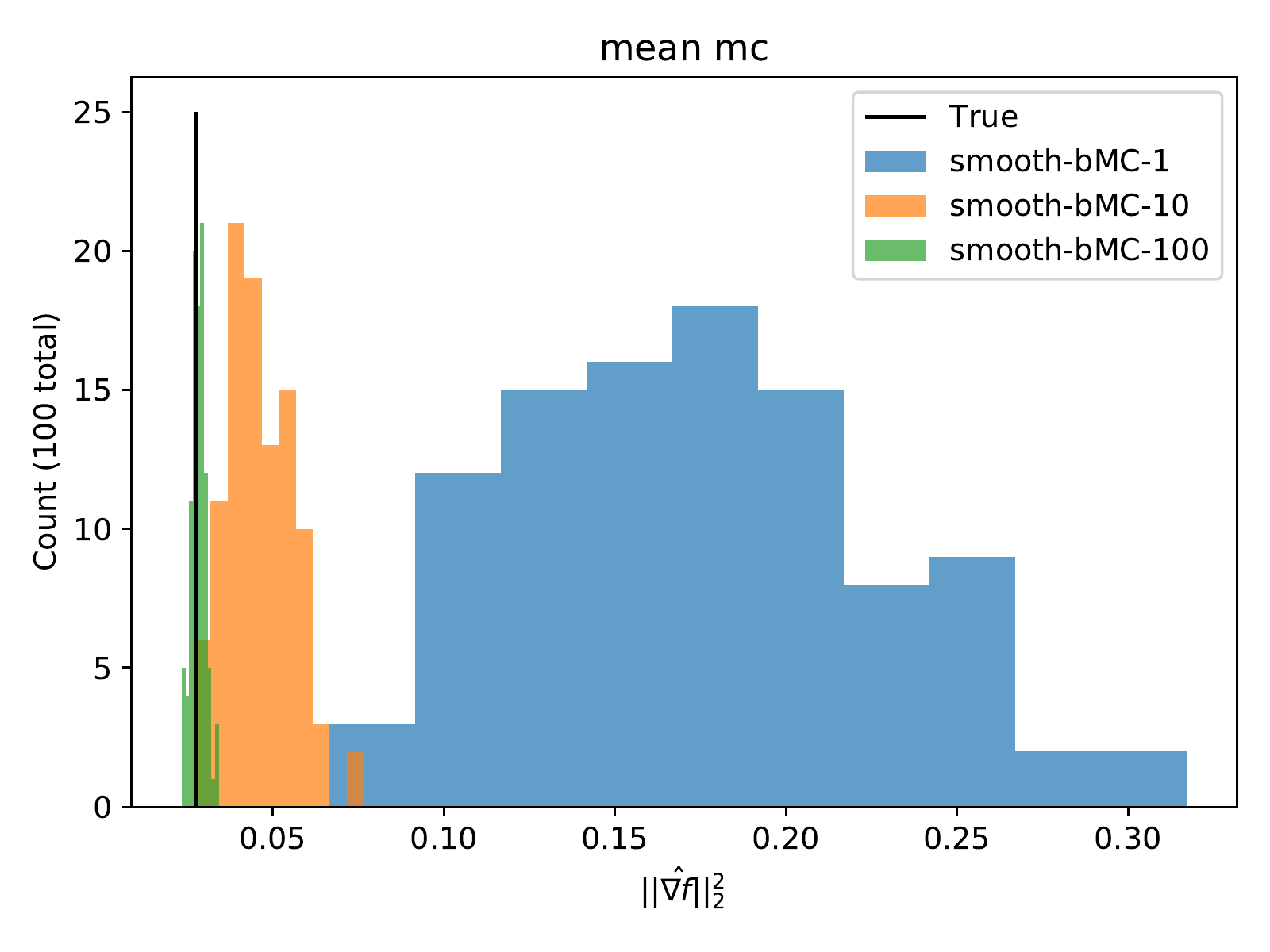}
\end{center}
\end{minipage}\quad
\begin{minipage}{0.31\linewidth}
\begin{center}
\includegraphics[width=0.99\linewidth]{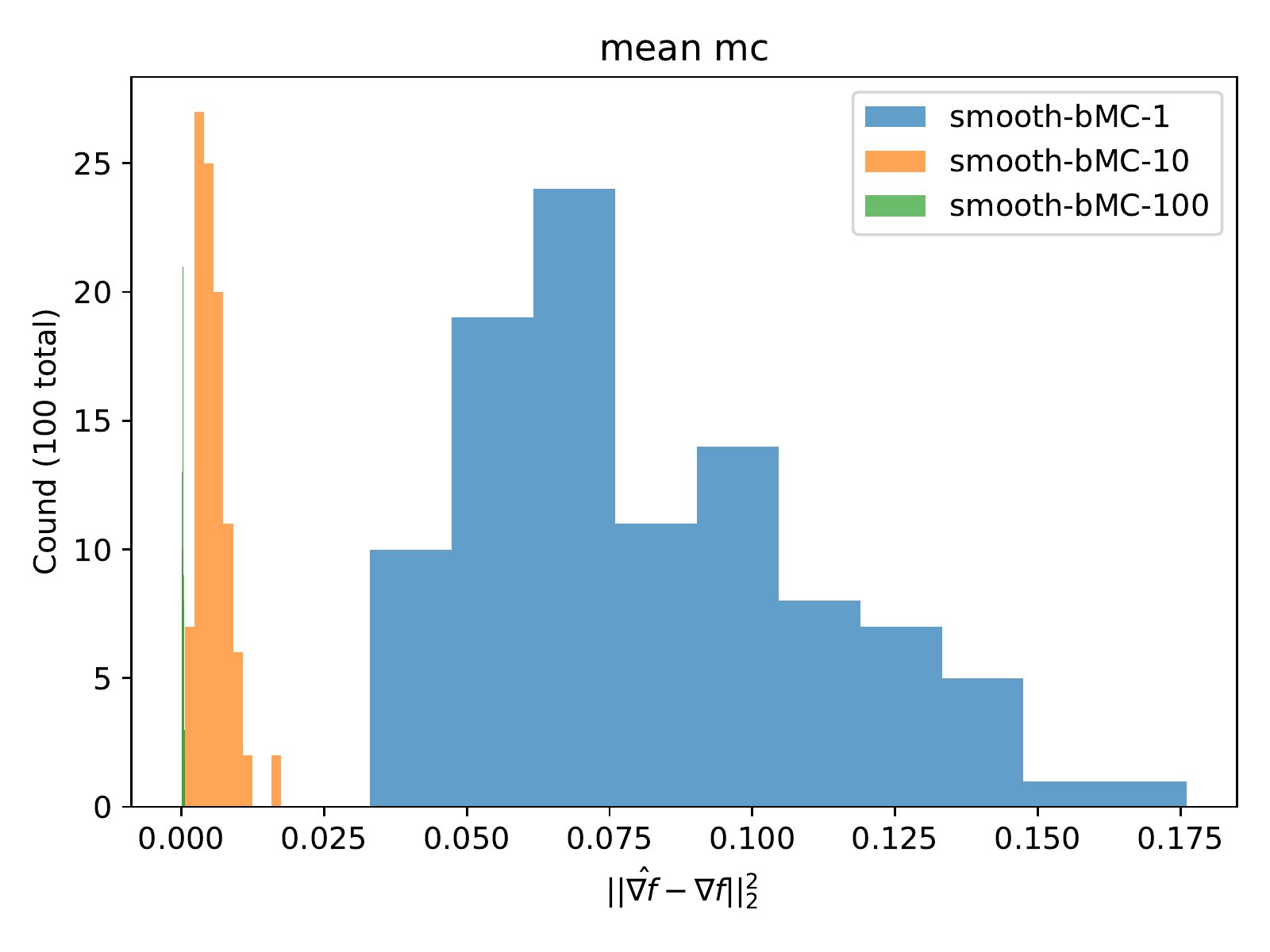}
\end{center}
\end{minipage}
\end{center}

\caption{Statistics for calculation of biased gradients for the mean parameter for Count prediction in the Abalone dataset.
First row uPS, second row bMC, and third row smooth-bMC($\nu=10^{-4}$).
Left: condition (i). Middle: condition (ii). Right: estimate of bias. Exact gradients estimated from 10000 bMC samples.}
\label{#1}
\end{figure*}
}

\newcommand{\PutAirlineA}[1]{
\begin{figure*}[h]
\begin{center}
\begin{minipage}{0.433\linewidth}
\begin{center}
    \includegraphics[width=0.99\linewidth]{figs/airline-20-logloss.pdf}
\end{center}
\end{minipage}\quad
\begin{minipage}{0.433\linewidth}
\begin{center}
\includegraphics[width=0.99\linewidth]{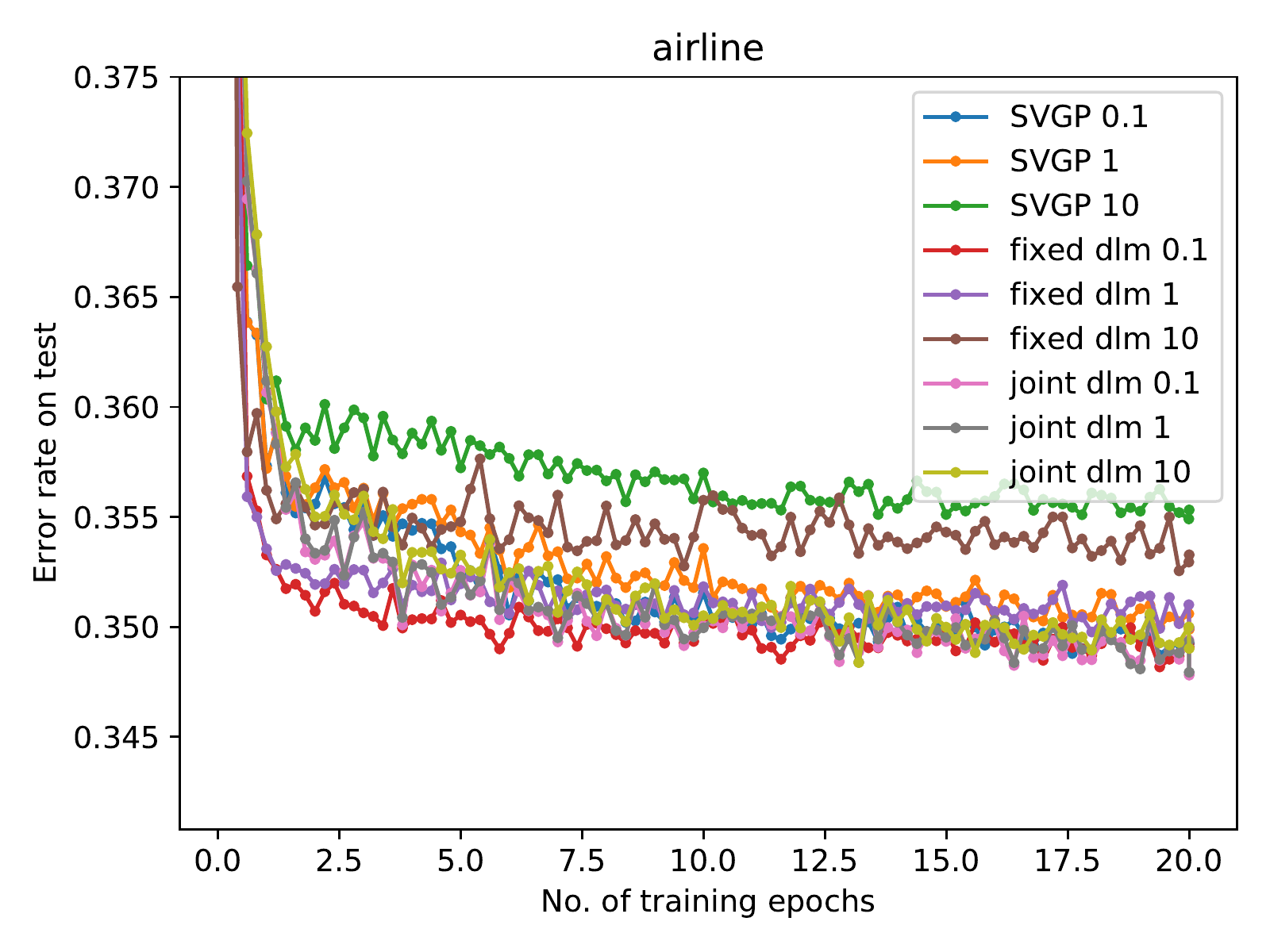}
\end{center}
\end{minipage}
\end{center}
\caption{Comparison of SVGP and DLM with exact gradients on the binary classification airline dataset. On the left is mean NLL and on the right is mean error. In both plots, lower values imply better performance.}
\label{#1}
\end{figure*}

}

\newcommand{\PutAirlineBA}[1]{
\begin{figure*}[h]
\begin{center}
\begin{minipage}{0.433\linewidth}
\begin{center}
    \includegraphics[width=0.99\linewidth]{figs/airline-20-dlm_ps-beta0.1-logloss.pdf}
\end{center}
\end{minipage}\quad
\begin{minipage}{0.433\linewidth}
\begin{center}
\includegraphics[width=0.99\linewidth]{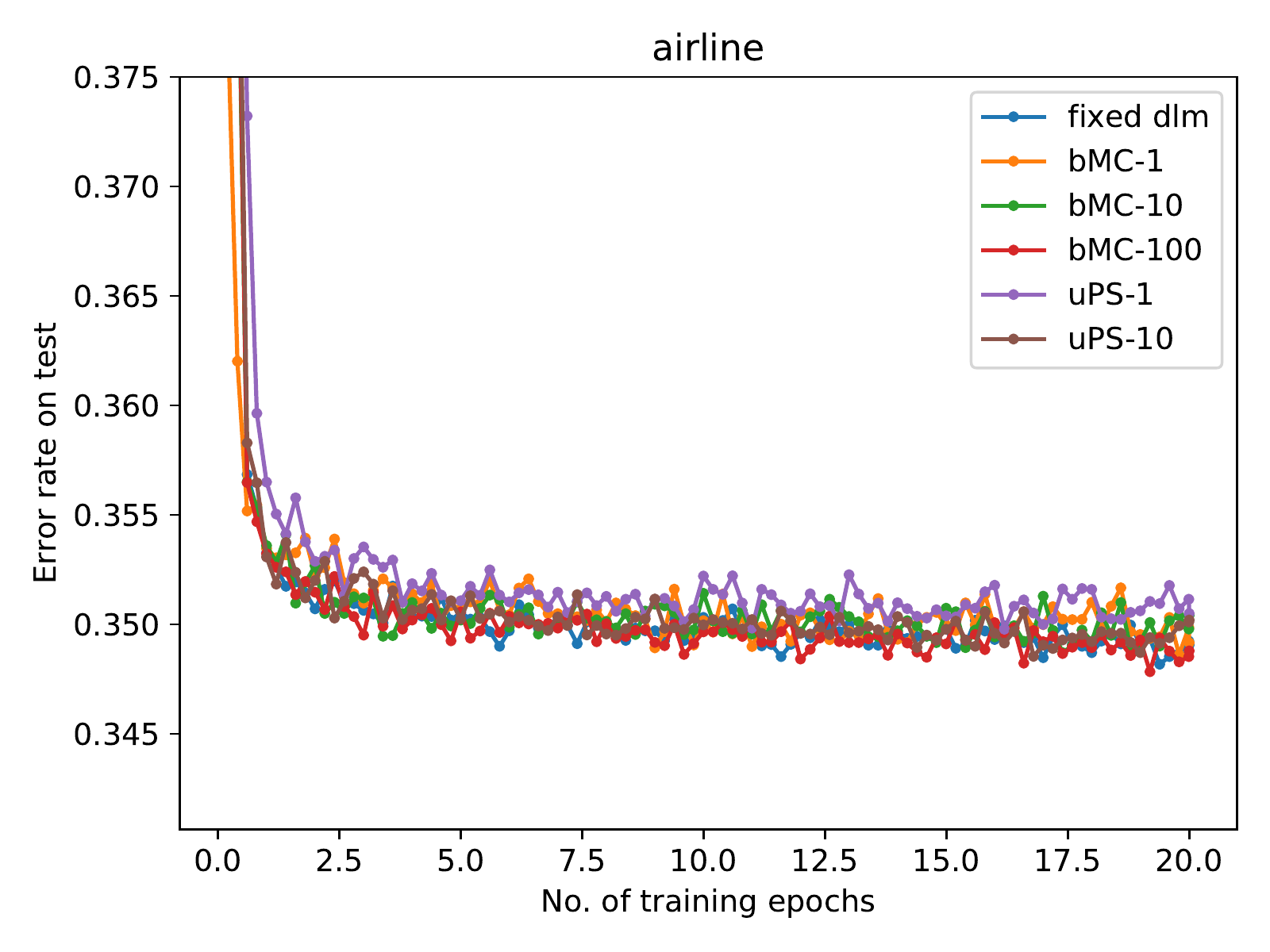}
\end{center}
\end{minipage}
\end{center}
\caption{Comparison of DLM with exact gradients, bMC gradients and uPS gradients with $\beta=0.1$ on the binary classification airline dataset. On the left is mean NLL and on the right is mean error. In both plots, lower values imply better performance.}
\label{#1}
\end{figure*}

}

\newcommand{\PutAirlineBB}[1]{
\begin{figure*}[h]
\begin{center}
\begin{minipage}{0.433\linewidth}
\begin{center}
    \includegraphics[width=0.99\linewidth]{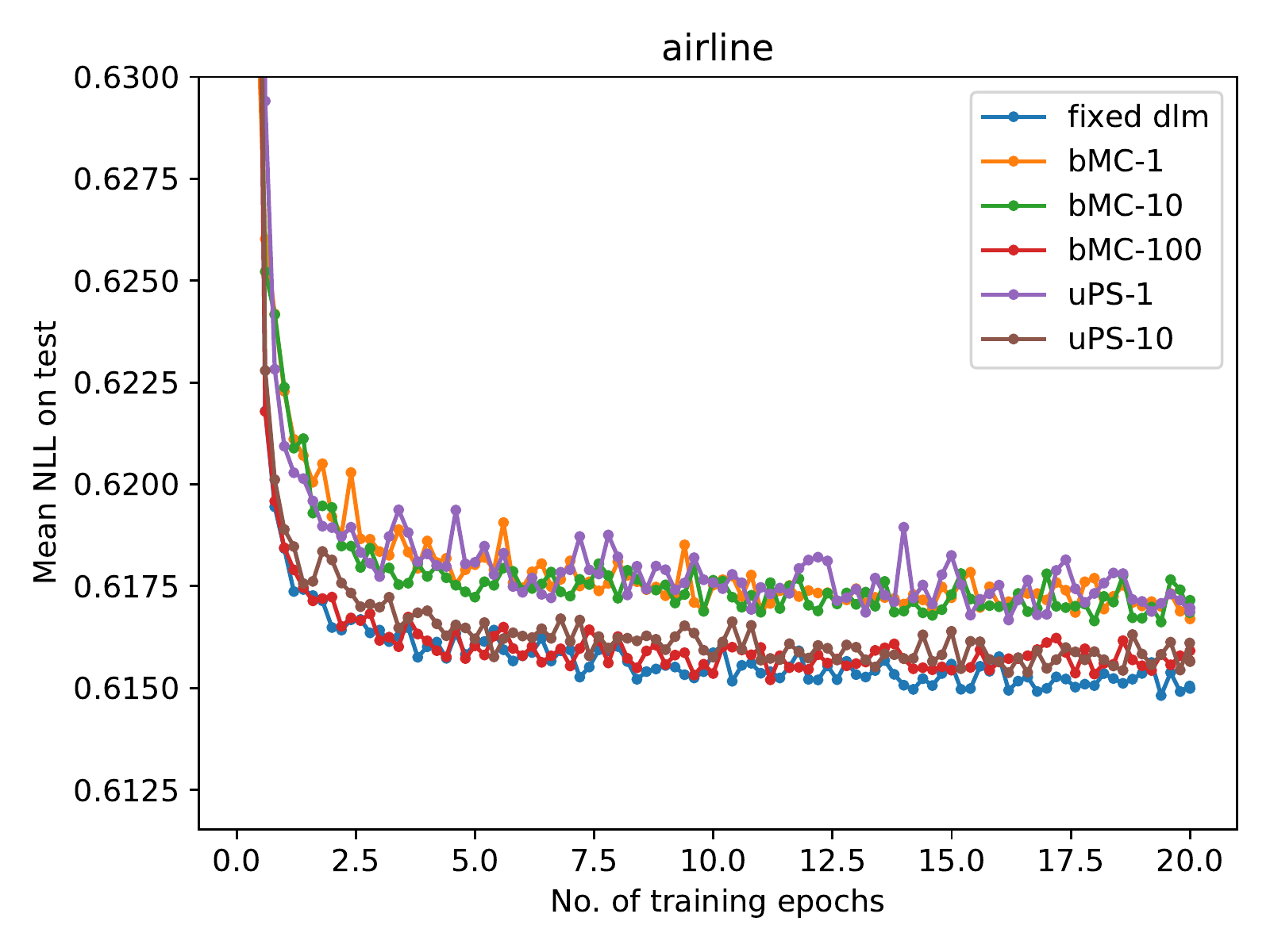}
\end{center}
\end{minipage}\quad
\begin{minipage}{0.433\linewidth}
\begin{center}
\includegraphics[width=0.99\linewidth]{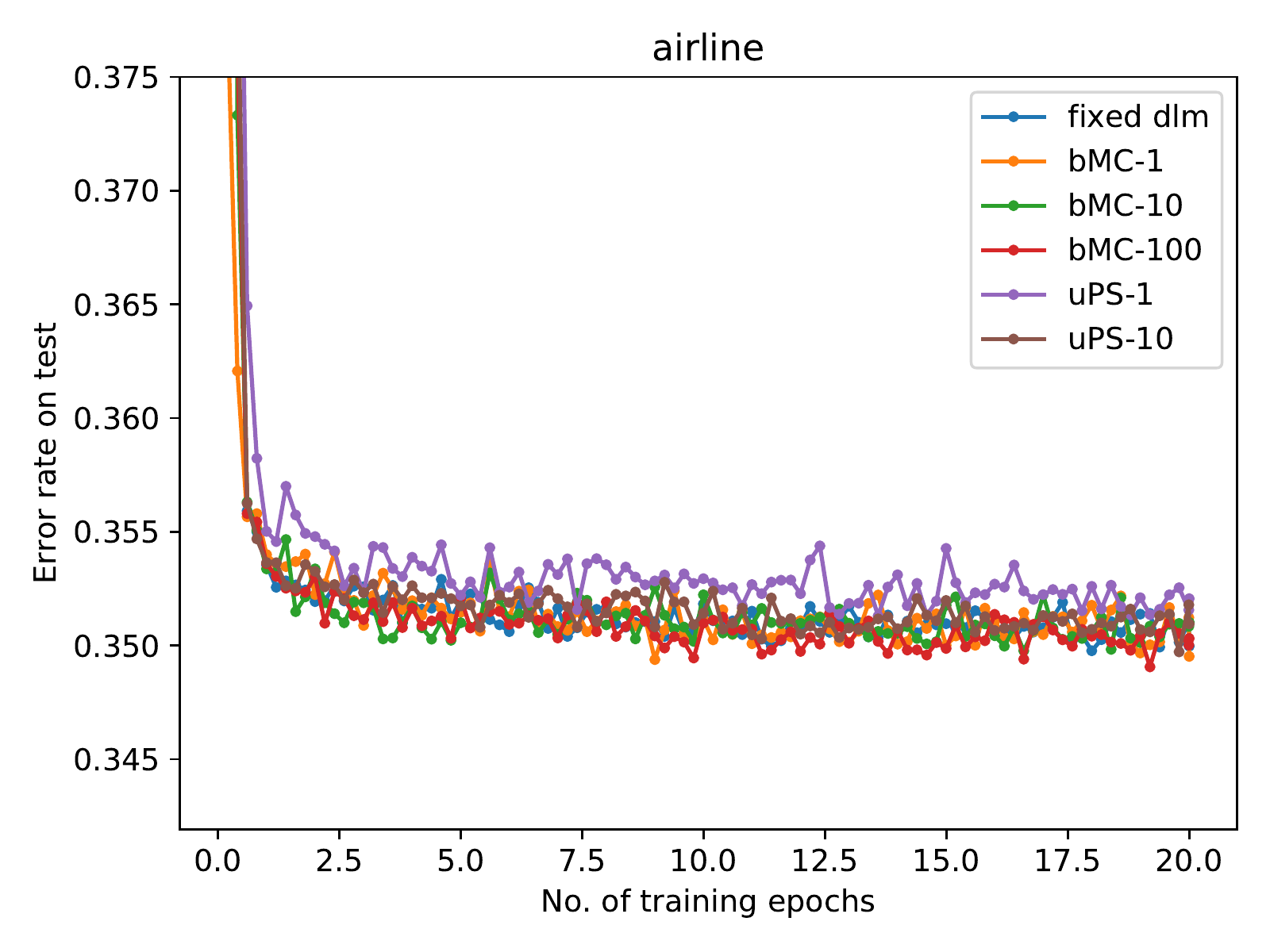}
\end{center}
\end{minipage}
\end{center}
\caption{Comparison of DLM with exact gradients, bMC gradients and uPS gradients with $\beta=1$ on the binary classification airline dataset. On the left is mean NLL and on the right is mean error. In both plots, lower values imply better performance.}
\label{#1}
\end{figure*}

}

\newcommand{\PutAirlineBC}[1]{
\begin{figure*}[h]
\begin{center}
\begin{minipage}{0.433\linewidth}
\begin{center}
    \includegraphics[width=0.99\linewidth]{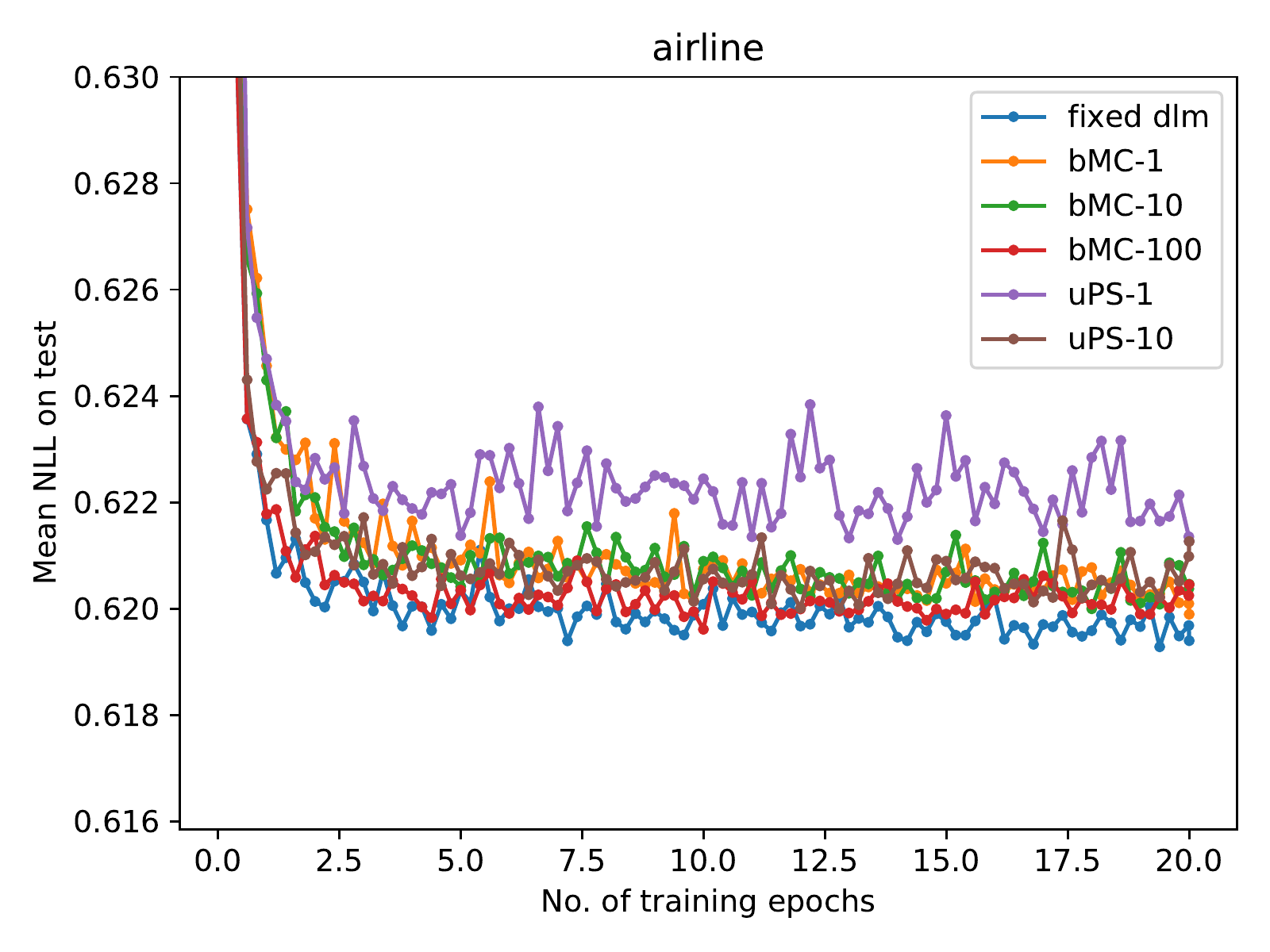}
\end{center}
\end{minipage}\quad
\begin{minipage}{0.433\linewidth}
\begin{center}
\includegraphics[width=0.99\linewidth]{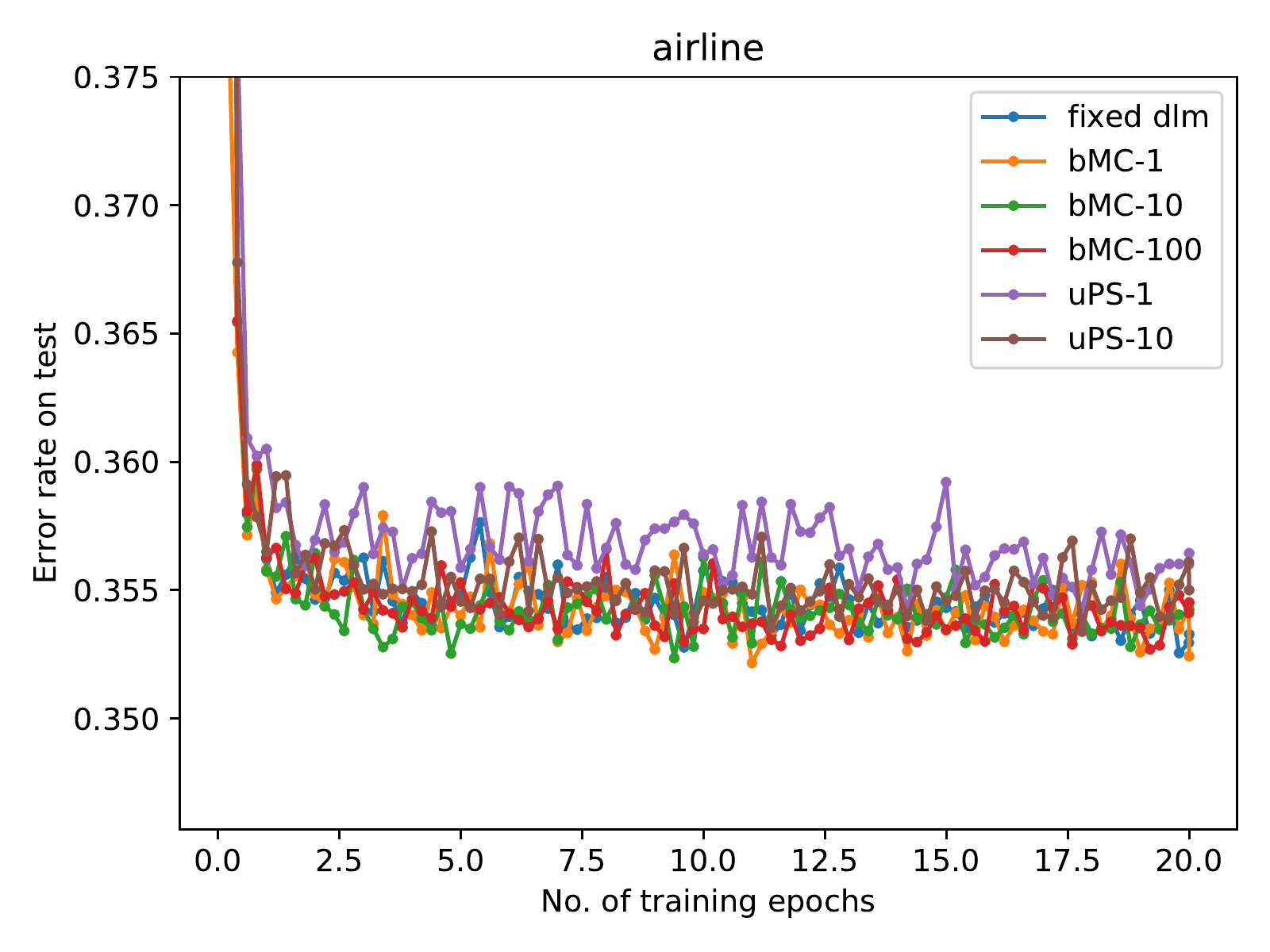}
\end{center}
\end{minipage}
\end{center}
\caption{Comparison of DLM with exact gradients, bMC gradients and uPS gradients with $\beta=10$ on the binary classification airline dataset. On the left is mean NLL and on the right is mean error. In both plots, lower values imply better performance.}
\label{#1}
\end{figure*}

}

\newcommand{\PutAirlineMain}[1]{
\begin{figure*}[h]
\begin{center}
\begin{minipage}{0.433\linewidth}
\begin{center}
    \includegraphics[width=0.99\linewidth]{figs/airline-logloss.pdf}
\end{center}
\end{minipage}\quad
\begin{minipage}{0.433\linewidth}
\begin{center}
\includegraphics[width=0.99\linewidth]{figs/airline-mis.pdf}
\end{center}
\end{minipage}
\end{center}
\caption{Comparison of SVGP and DLM on the binary classification airline dataset. On the left is negative log predictive probability and on the right is error rate. In both plots, lower values imply better performance.}
\label{#1}
\end{figure*}

}

\newcommand{\PutAirlineMainMCA}[1]{
\begin{figure*}[h]
\begin{center}
\begin{minipage}{0.433\linewidth}
\begin{center}
    \includegraphics[width=0.99\linewidth]{figs/airline-mc1-logloss.pdf}
\end{center}
\end{minipage}\quad
\begin{minipage}{0.433\linewidth}
\begin{center}
\includegraphics[width=0.99\linewidth]{figs/airline-mc1-mis.pdf}
\end{center}
\end{minipage}
\end{center}
\caption{Comparison of SVGP and DLM with one Monte Carlo sample per datapoint on the binary classification airline dataset. On the left is negative log predictive probability and on the right is error rate. In both plots, lower values imply better performance.}
\label{#1}
\end{figure*}

}

\newcommand{\PutAirlineMainMCB}[1]{
\begin{figure*}[h]
\begin{center}
\begin{minipage}{0.433\linewidth}
\begin{center}
    \includegraphics[width=0.99\linewidth]{figs/airline-mc10-logloss.pdf}
\end{center}
\end{minipage}\quad
\begin{minipage}{0.433\linewidth}
\begin{center}
\includegraphics[width=0.99\linewidth]{figs/airline-mc10-mis.pdf}
\end{center}
\end{minipage}
\end{center}
\caption{Comparison of SVGP and DLM with ten Monte Carlo samples per datapoint on the binary classification airline dataset. On the left is negative log predictive probability and on the right is error rate. In both plots, lower values imply better performance.}
\label{#1}
\end{figure*}

}

\newcommand{\PutAirlineMainMCC}[1]{
\begin{figure*}[h]
\begin{center}
\begin{minipage}{0.433\linewidth}
\begin{center}
    \includegraphics[width=0.99\linewidth]{figs/airline-mc100-logloss.pdf}
\end{center}
\end{minipage}\quad
\begin{minipage}{0.433\linewidth}
\begin{center}
\includegraphics[width=0.99\linewidth]{figs/airline-mc100-mis.pdf}
\end{center}
\end{minipage}
\end{center}
\caption{Comparison of SVGP and DLM with one hundred Monte Carlo samples per datapoint on the binary classification airline dataset. On the left is negative log predictive probability and on the right is error rate. In both plots, lower values imply better performance.}
\label{#1}
\end{figure*}

}

\newcommand{\PutAirlineMainMCTenHundred}[1]{
\begin{figure*}[h]
\begin{center}
\begin{minipage}{0.433\linewidth}
\begin{center}
    \includegraphics[width=0.99\linewidth]{figs/airline-mc10-logloss.pdf}
\end{center}
\end{minipage}\quad
\begin{minipage}{0.433\linewidth}
\begin{center}
\includegraphics[width=0.99\linewidth]{figs/airline-mc100-logloss.pdf}
\end{center}
\end{minipage}
\end{center}
\caption{Comparison of negative log predictive probability for SVGP and DLM with 10 (left) and 100 (right) Monte Carlo samples per datapoint on the binary classification airline dataset.}
\label{#1}
\end{figure*}

}

\newcommand{\PutAirlineGradientsCov}[1]{
\begin{figure*}[h]
\begin{center}
\begin{minipage}{0.31\linewidth}
\begin{center}
    \includegraphics[width=0.99\linewidth]{figs/grad_stats_airline-0.1-seed0-lr0.001-dlm-beta0.1-20-ard_rbf-M200-B1000-seed0-lr0.001-svgp_condc_hists_cov.pdf}
\end{center}
\end{minipage}\quad
\begin{minipage}{0.31\linewidth}
\begin{center}
    \includegraphics[width=0.99\linewidth]{figs/grad_stats_airline-0.1-seed0-lr0.001-dlm-beta0.1-20-ard_rbf-M200-B1000-seed0-lr0.001-svgp_condd_hists_cov.pdf}
\end{center}
\end{minipage}\quad
\begin{minipage}{0.31\linewidth}
\begin{center}
\includegraphics[width=0.99\linewidth]{figs/grad_stats_airline-0.1-seed0-lr0.001-dlm-beta0.1-20-ard_rbf-M200-B1000-seed0-lr0.001-svgp_err_hists_cov.pdf}
\end{center}
\end{minipage}
\end{center}
\caption{Statistics for calculation of biased gradients for the Cholesky parameter for binary classification.
Left: condition c. Middle: condition d. Right: estimate of bias.
(Note, $\hat{\nabla f}$ denotes the $s_t$ variable described in the main text.)
}
\label{#1}
\end{figure*}
}

\newcommand{\PutAirlineGradientsIND}[1]{
\begin{figure*}[h]
\begin{center}
\begin{minipage}{0.31\linewidth}
\begin{center}
    \includegraphics[width=0.99\linewidth]{figs/grad_stats_airline-0.1-seed0-lr0.001-dlm-beta0.1-20-ard_rbf-M200-B1000-seed0-lr0.001-svgp_condc_hists_ind.pdf}
\end{center}
\end{minipage}\quad
\begin{minipage}{0.31\linewidth}
\begin{center}
    \includegraphics[width=0.99\linewidth]{figs/grad_stats_airline-0.1-seed0-lr0.001-dlm-beta0.1-20-ard_rbf-M200-B1000-seed0-lr0.001-svgp_condd_hists_ind.pdf}
\end{center}
\end{minipage}\quad
\begin{minipage}{0.31\linewidth}
\begin{center}
\includegraphics[width=0.99\linewidth]{figs/grad_stats_airline-0.1-seed0-lr0.001-dlm-beta0.1-20-ard_rbf-M200-B1000-seed0-lr0.001-svgp_err_hists_ind.pdf}
\end{center}
\end{minipage}
\end{center}
\caption{Statistics for calculation of biased gradients for the inducing locations parameter for binary classification.
Left: condition c. Middle: condition d. Right: estimate of bias.}
\label{#1}
\end{figure*}
}

\newcommand{\PutAirlineGradientsMean}[1]{
\begin{figure*}[h]
\begin{center}
\begin{minipage}{0.31\linewidth}
\begin{center}
    \includegraphics[width=0.99\linewidth]{figs/grad_stats_airline-0.1-seed0-lr0.001-dlm-beta0.1-20-ard_rbf-M200-B1000-seed0-lr0.001-svgp_condc_hists_mean.pdf}
\end{center}
\end{minipage}\quad
\begin{minipage}{0.31\linewidth}
\begin{center}
    \includegraphics[width=0.99\linewidth]{figs/grad_stats_airline-0.1-seed0-lr0.001-dlm-beta0.1-20-ard_rbf-M200-B1000-seed0-lr0.001-svgp_condd_hists_mean.pdf}
\end{center}
\end{minipage}\quad
\begin{minipage}{0.31\linewidth}
\begin{center}
\includegraphics[width=0.99\linewidth]{figs/grad_stats_airline-0.1-seed0-lr0.001-dlm-beta0.1-20-ard_rbf-M200-B1000-seed0-lr0.001-svgp_err_hists_mean.pdf}
\end{center}
\end{minipage}
\end{center}
\caption{Statistics for calculation of biased gradients for the mean parameter for binary classification.
Left: condition c. Middle: condition d. Right: estimate of bias.}
\label{#1}
\end{figure*}
}

\newcommand{\PutAirlineProductSamplingLearningCurve}[1]{
\begin{figure*}[h]
\begin{center}
\begin{minipage}{0.433\linewidth}
\begin{center}
    \includegraphics[width=0.99\linewidth]{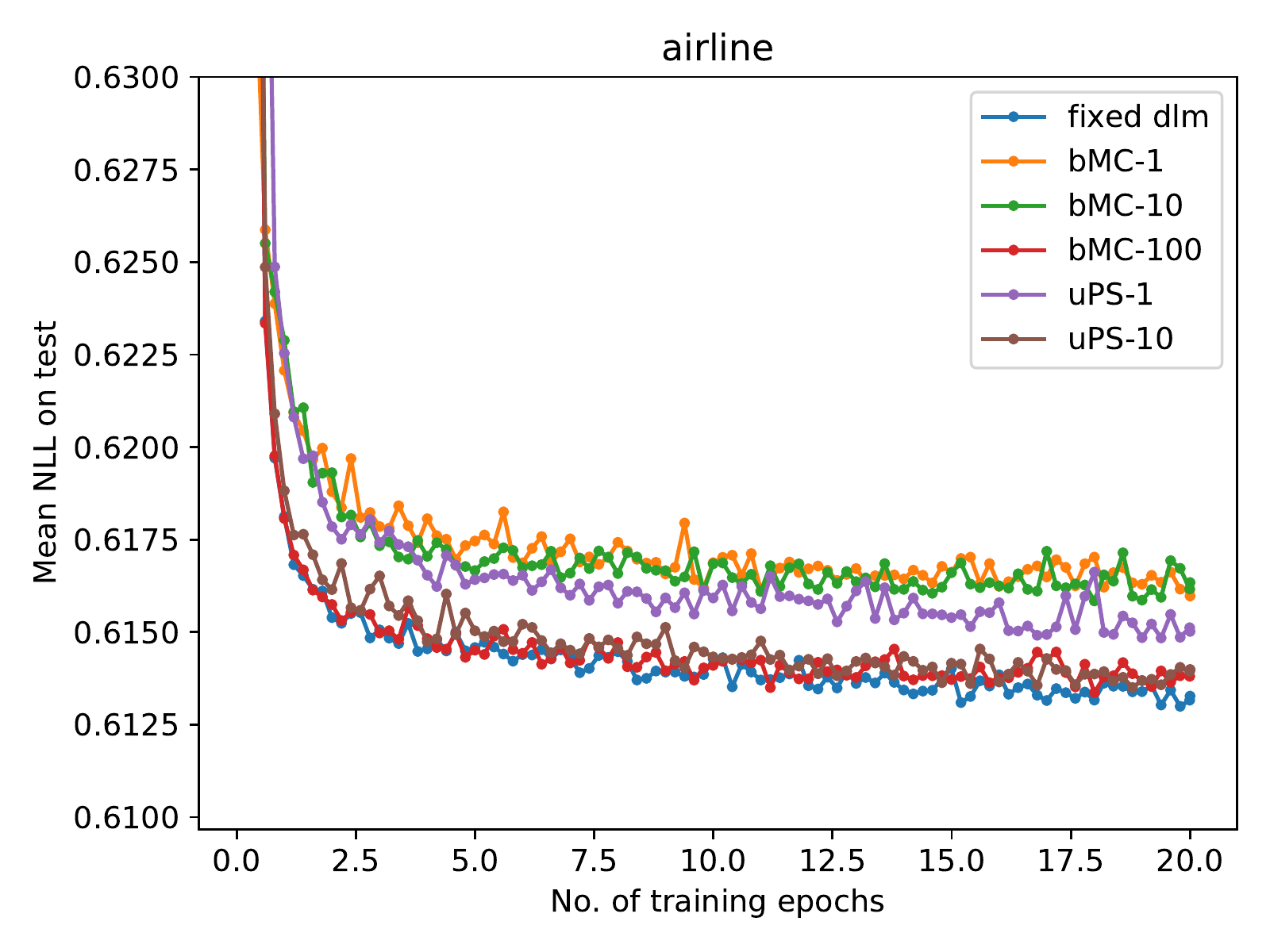}
\end{center}
\end{minipage}\quad
\begin{minipage}{0.433\linewidth}
\begin{center}
    \includegraphics[width=0.99\linewidth]{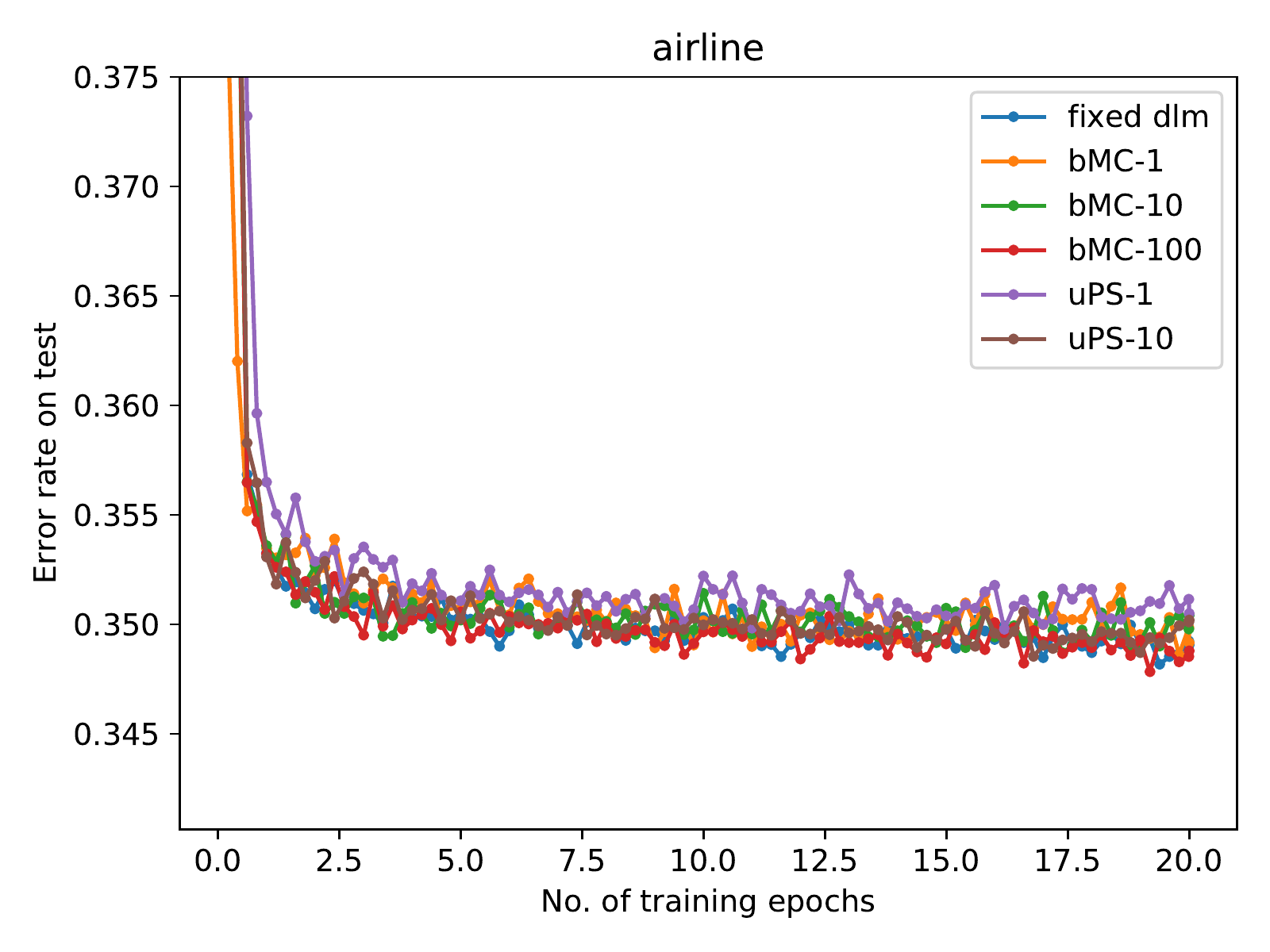}
\end{center}
\end{minipage}
\end{center}
\caption{Learning curve of bMC and uPS.}
\label{#1}
\end{figure*}
}

\newcommand{\PutAirlineSmoothBMCLearningCurve}[1]{
\begin{figure*}[h]
\begin{center}
\begin{minipage}{0.433\linewidth}
\begin{center}
    \includegraphics[width=0.99\linewidth]{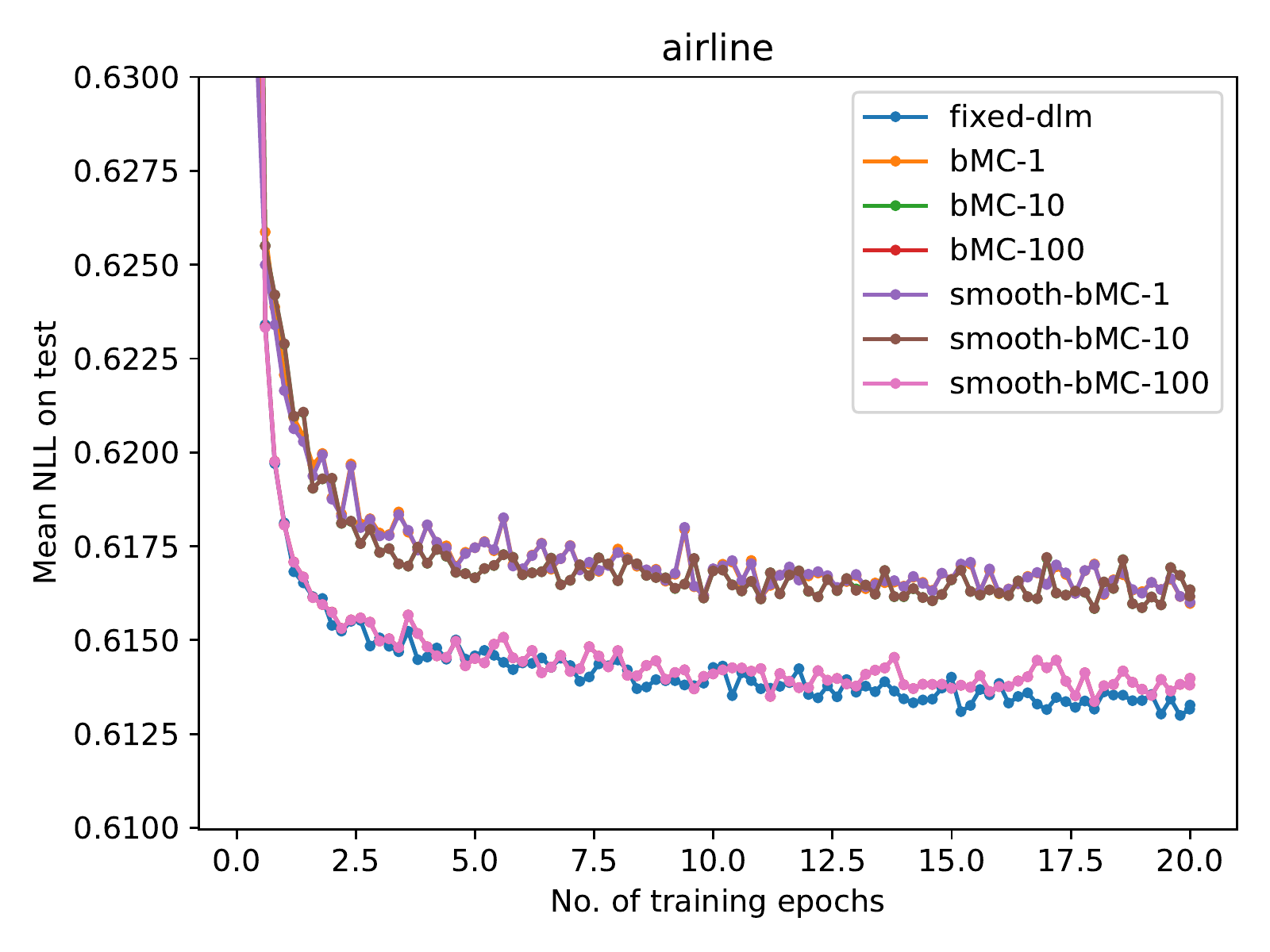}
\end{center}
\end{minipage}\quad
\begin{minipage}{0.433\linewidth}
\begin{center}
    \includegraphics[width=0.99\linewidth]{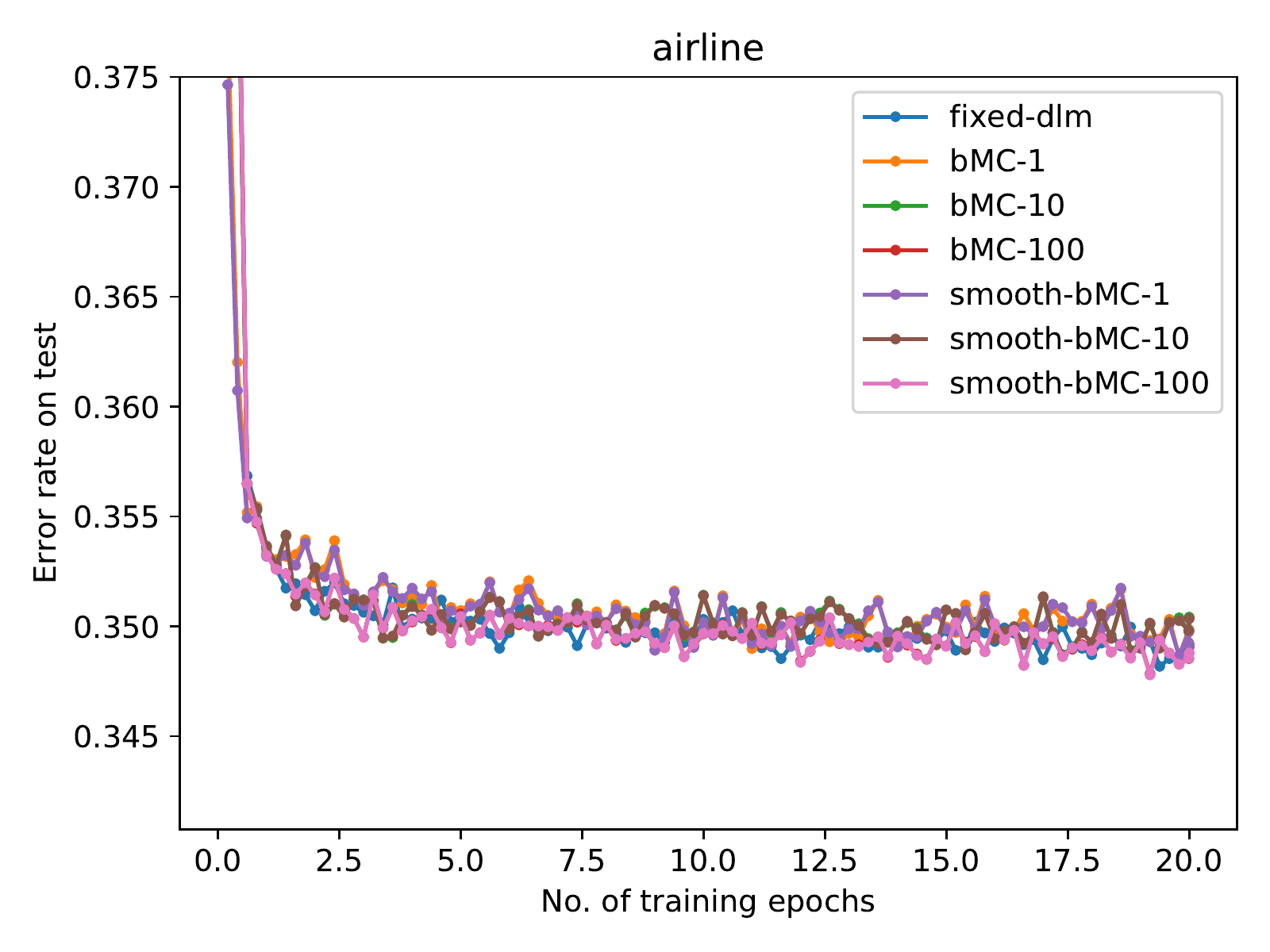}
\end{center}
\end{minipage}
\end{center}
\caption{Learning curve of bMC and smooth-bMC, $\beta=1$.}
\label{#1}
\end{figure*}
}

\newcommand{\PutRegressionLog}[1]{
\begin{figure*}[h]
\begin{center}
\begin{minipage}{0.310\linewidth}
\begin{center}
    \includegraphics[width=0.99\linewidth]{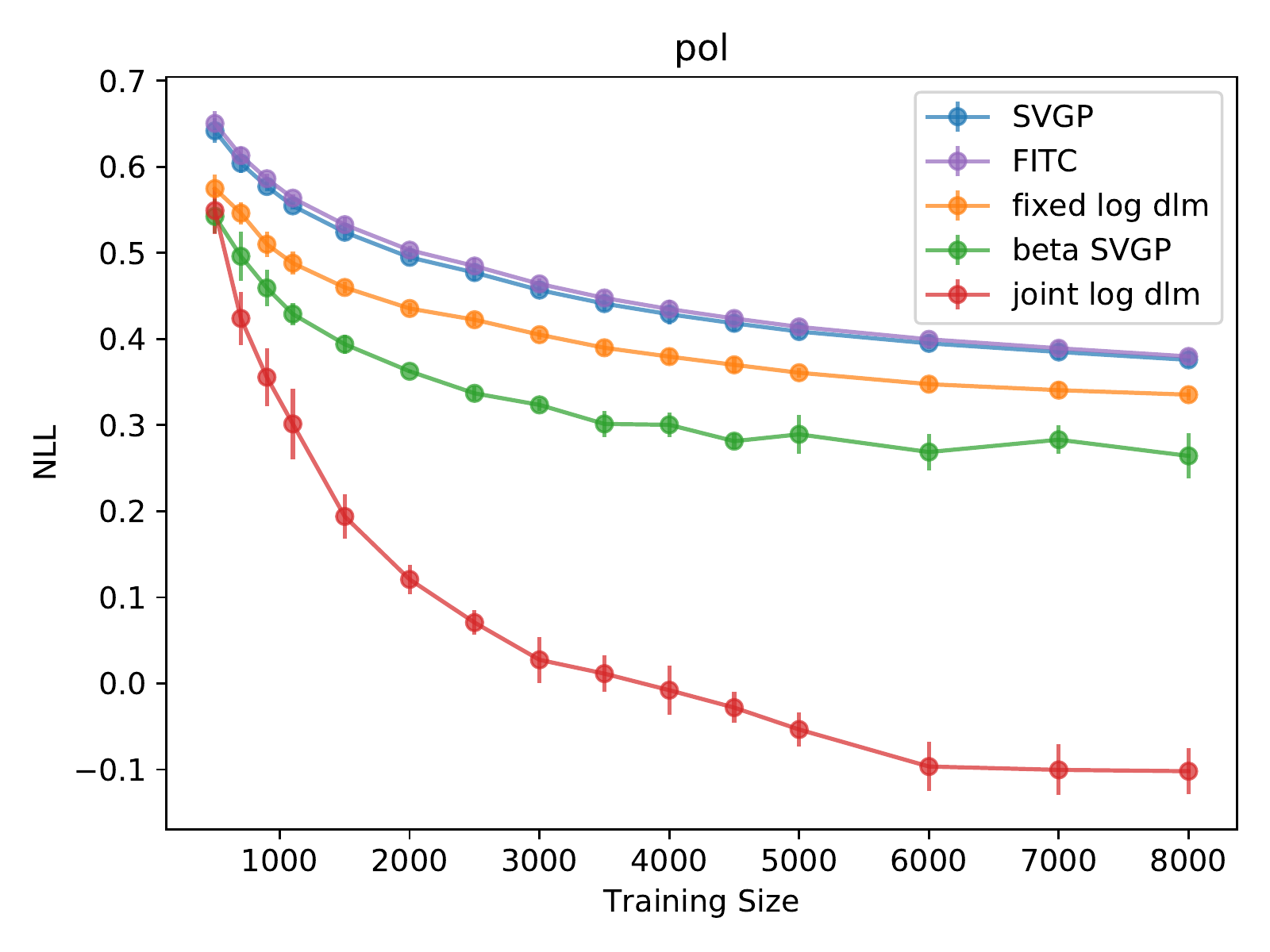}
\end{center}
\end{minipage}\quad
\begin{minipage}{0.310\linewidth}
\begin{center}
    \includegraphics[width=0.99\linewidth]{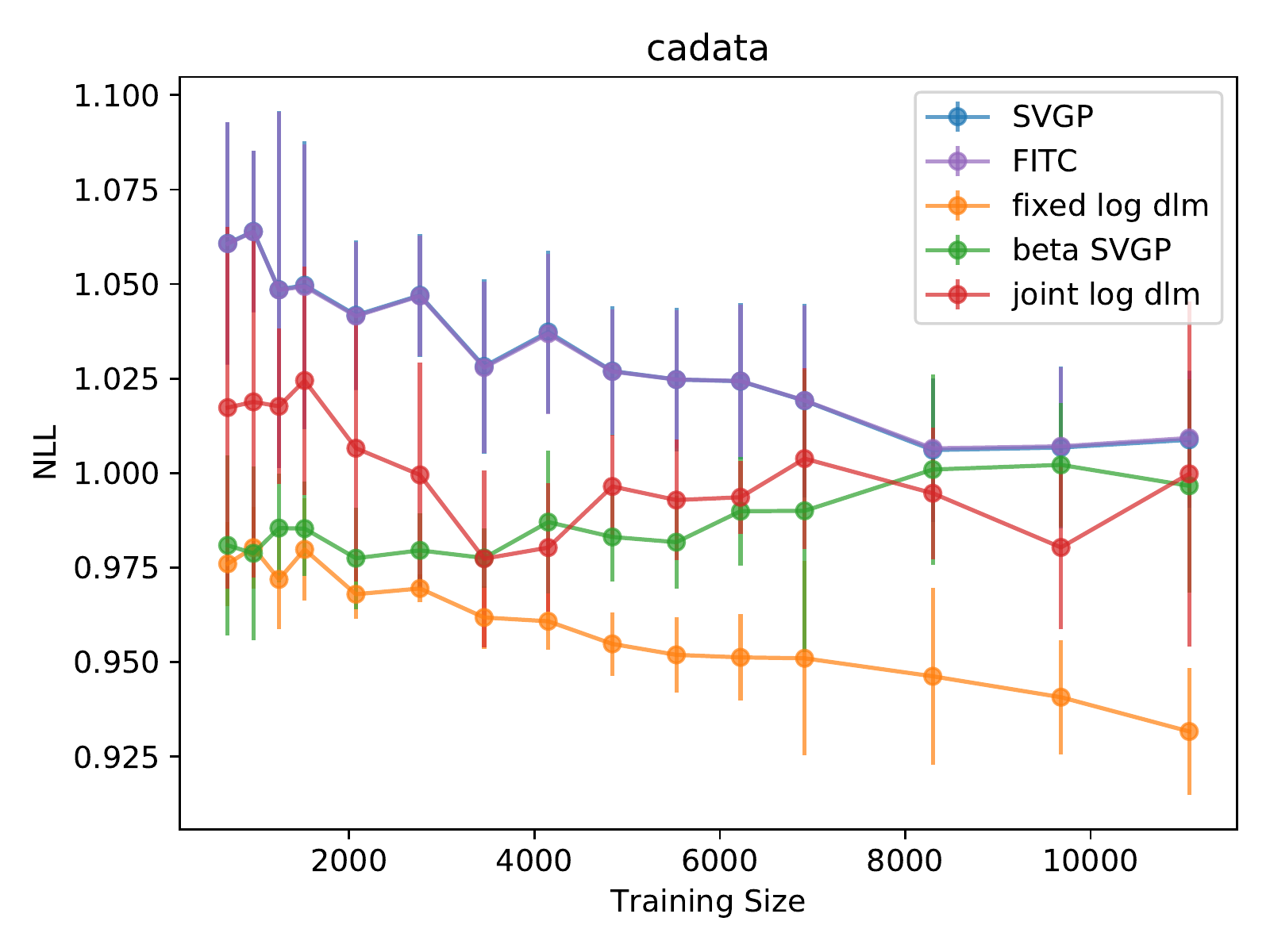}
\end{center}
\end{minipage}\quad
\begin{minipage}{0.310\linewidth}
\begin{center}
    \includegraphics[width=0.99\linewidth]{figs/beta_cadata_size691.pdf}
\end{center}
\end{minipage}\quad
\begin{minipage}{0.310\linewidth}
\begin{center}
    \includegraphics[width=0.99\linewidth]{figs/log_sarcos.pdf}
\end{center}
\end{minipage}\quad
\begin{minipage}{0.310\linewidth}
\begin{center}
    \includegraphics[width=0.99\linewidth]{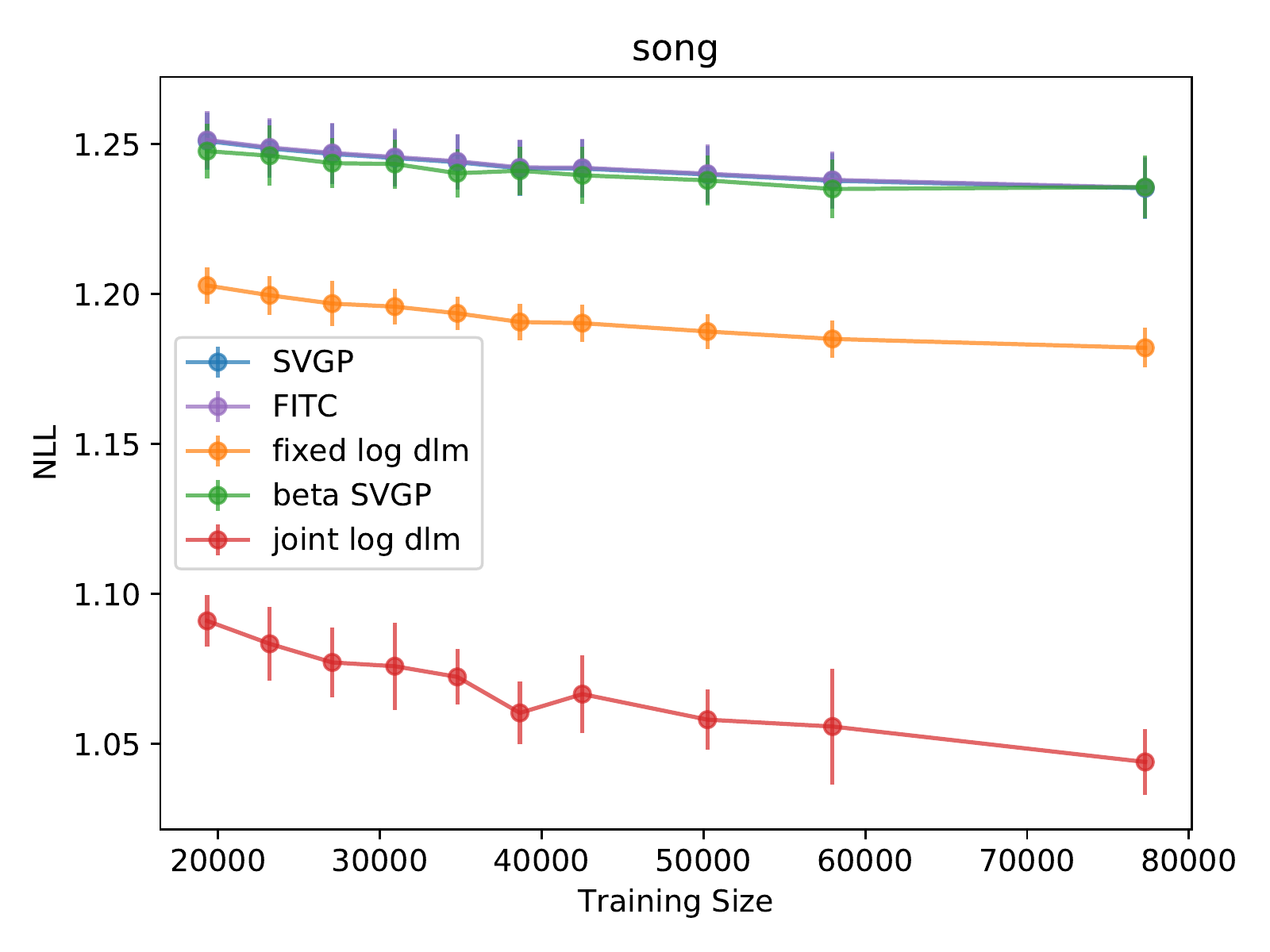}
\end{center}
\end{minipage}
\begin{minipage}{0.310\linewidth}
\begin{center}
    \includegraphics[width=0.99\linewidth]{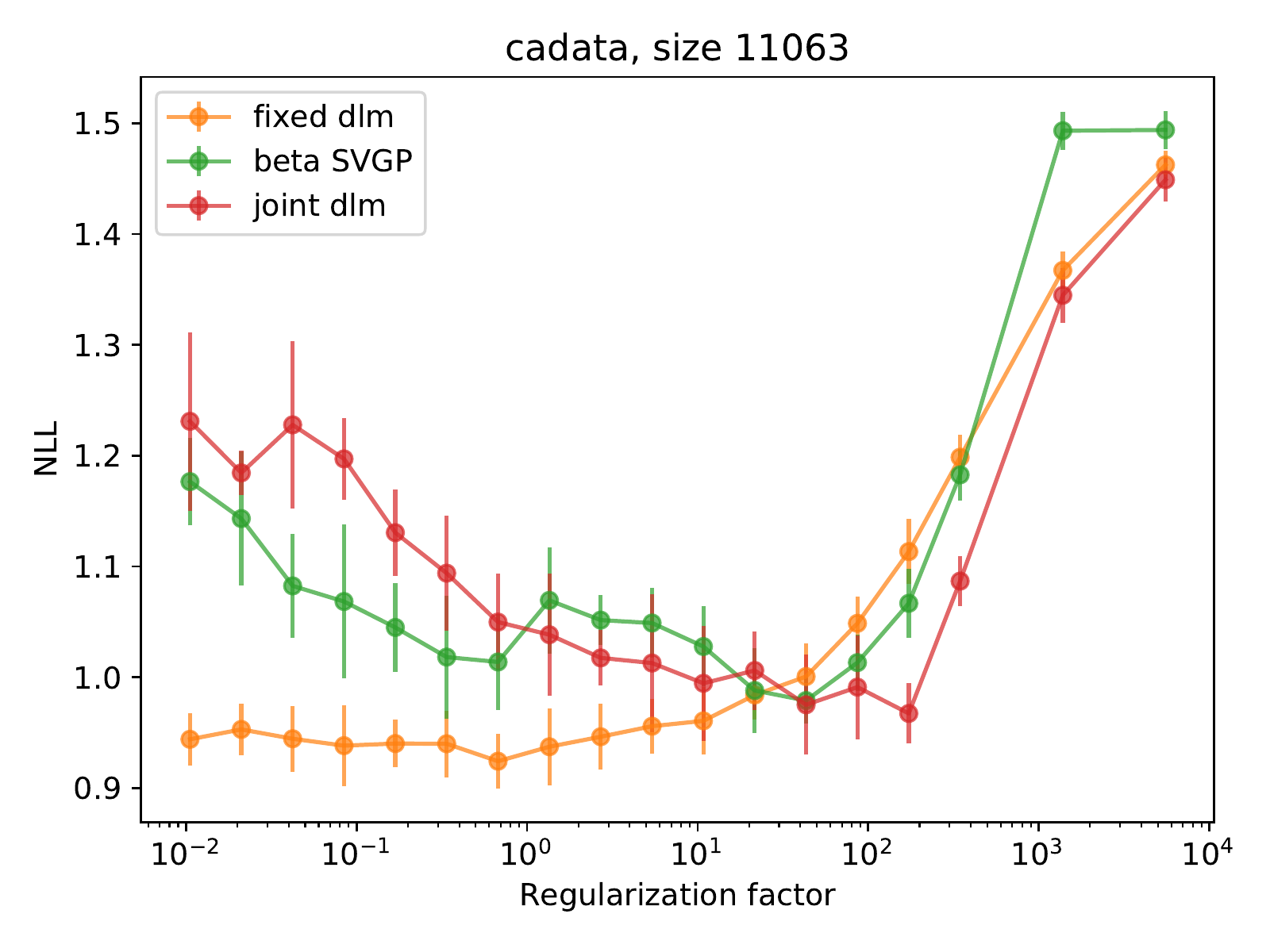}
\end{center}
\end{minipage}
\end{center}
\caption{sGP Regression: Left and middle columns show a comparison of SVGP, FITC and DLM on mean test NLL in 4 datasets. The right column shows NLL as a function of $\beta$ for cadata for a small training size and a large training size. In all plots, lower values imply better performance. }
\label{#1}
\end{figure*}

}

\newcommand{\PutRegressionSq}[1]{
\begin{figure*}[h]
\begin{center}
\begin{minipage}{0.433\linewidth}
\begin{center}
    \includegraphics[width=0.99\linewidth]{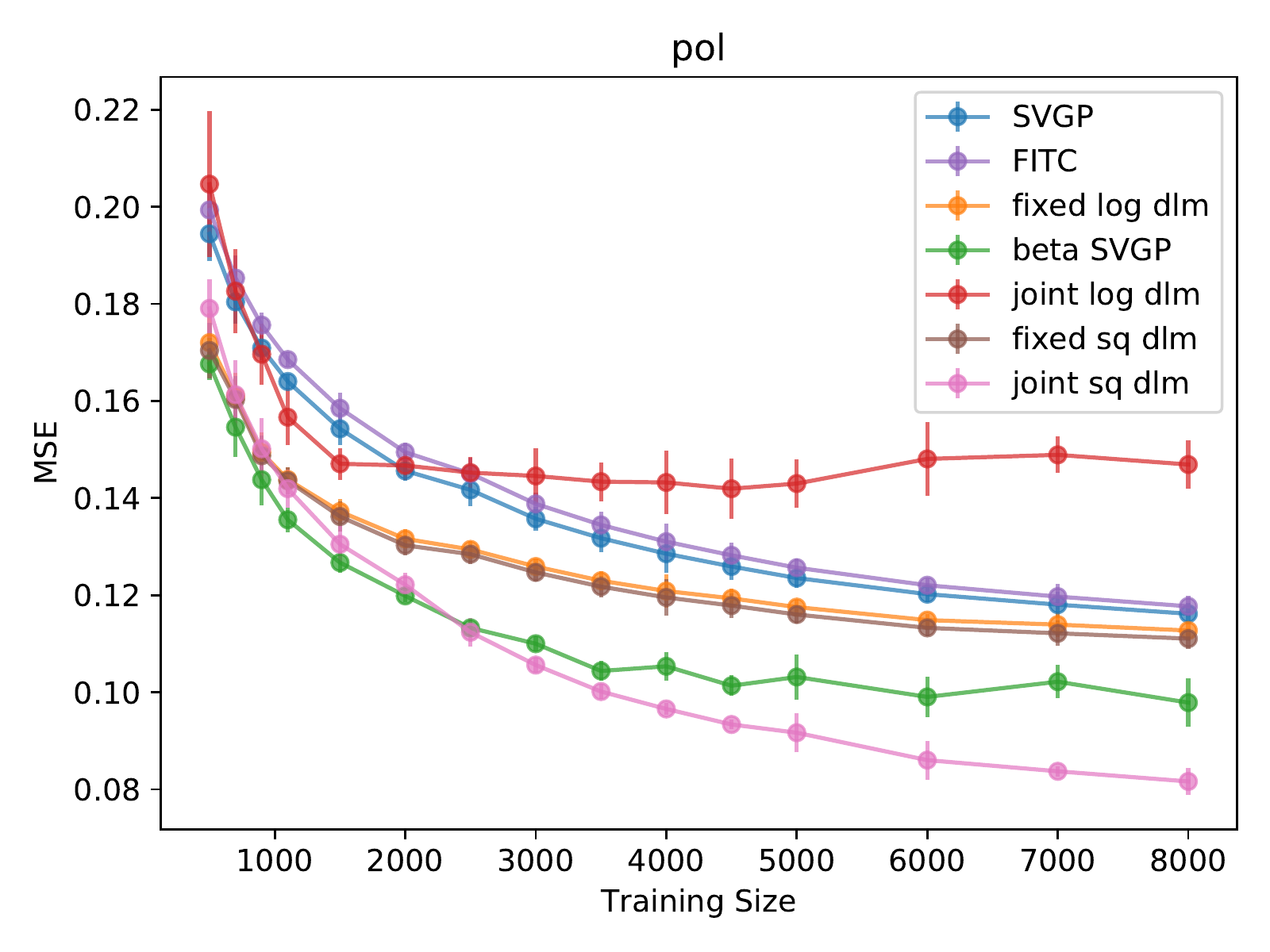}
\end{center}
\end{minipage}\quad
\begin{minipage}{0.433\linewidth}
\begin{center}
    \includegraphics[width=0.99\linewidth]{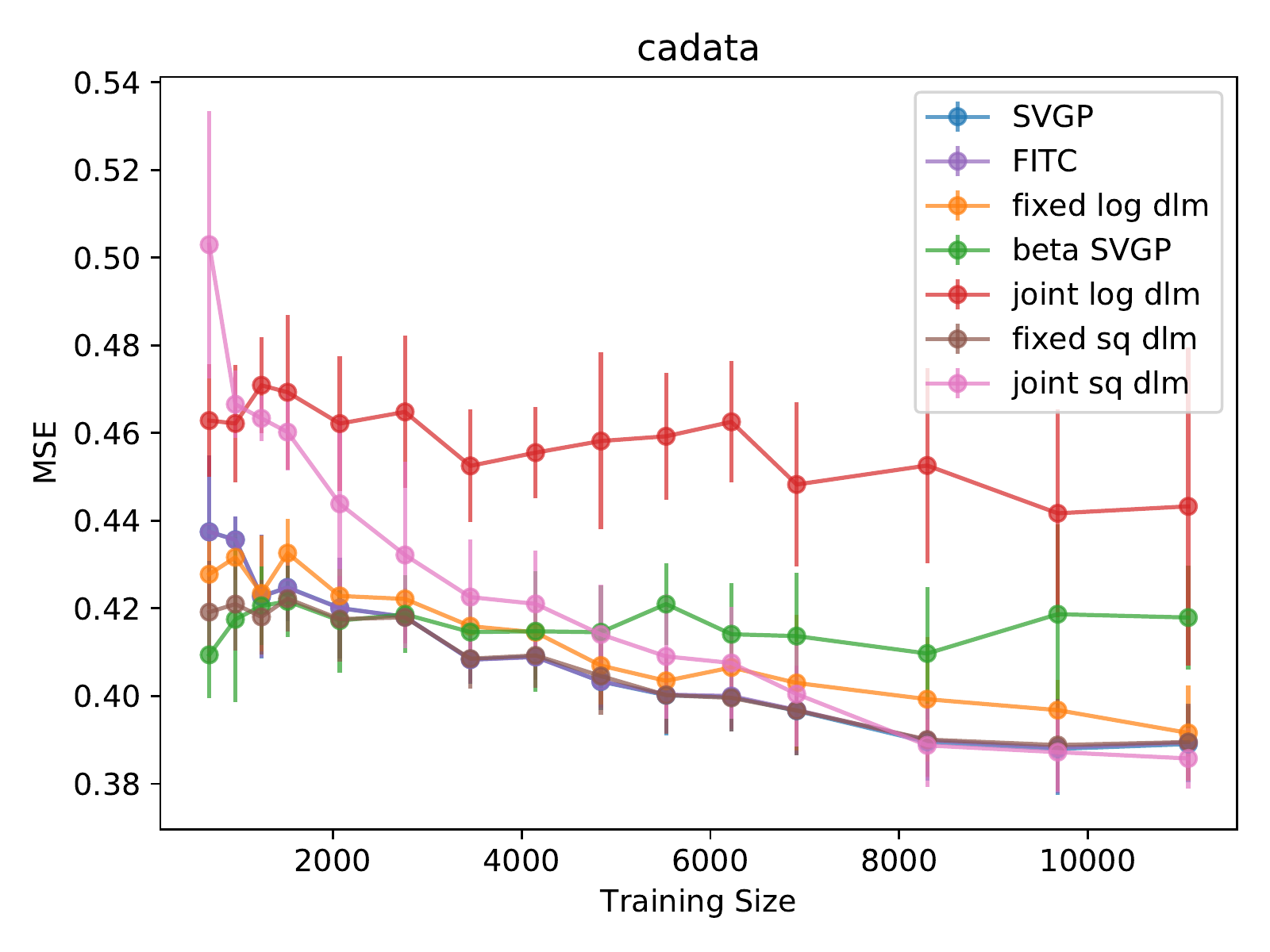}
\end{center}
\end{minipage}\quad
\begin{minipage}{0.433\linewidth}
\begin{center}
    \includegraphics[width=0.99\linewidth]{figs/mse_sarcos.pdf}
\end{center}
\end{minipage}\quad
\begin{minipage}{0.433\linewidth}
\begin{center}
    \includegraphics[width=0.99\linewidth]{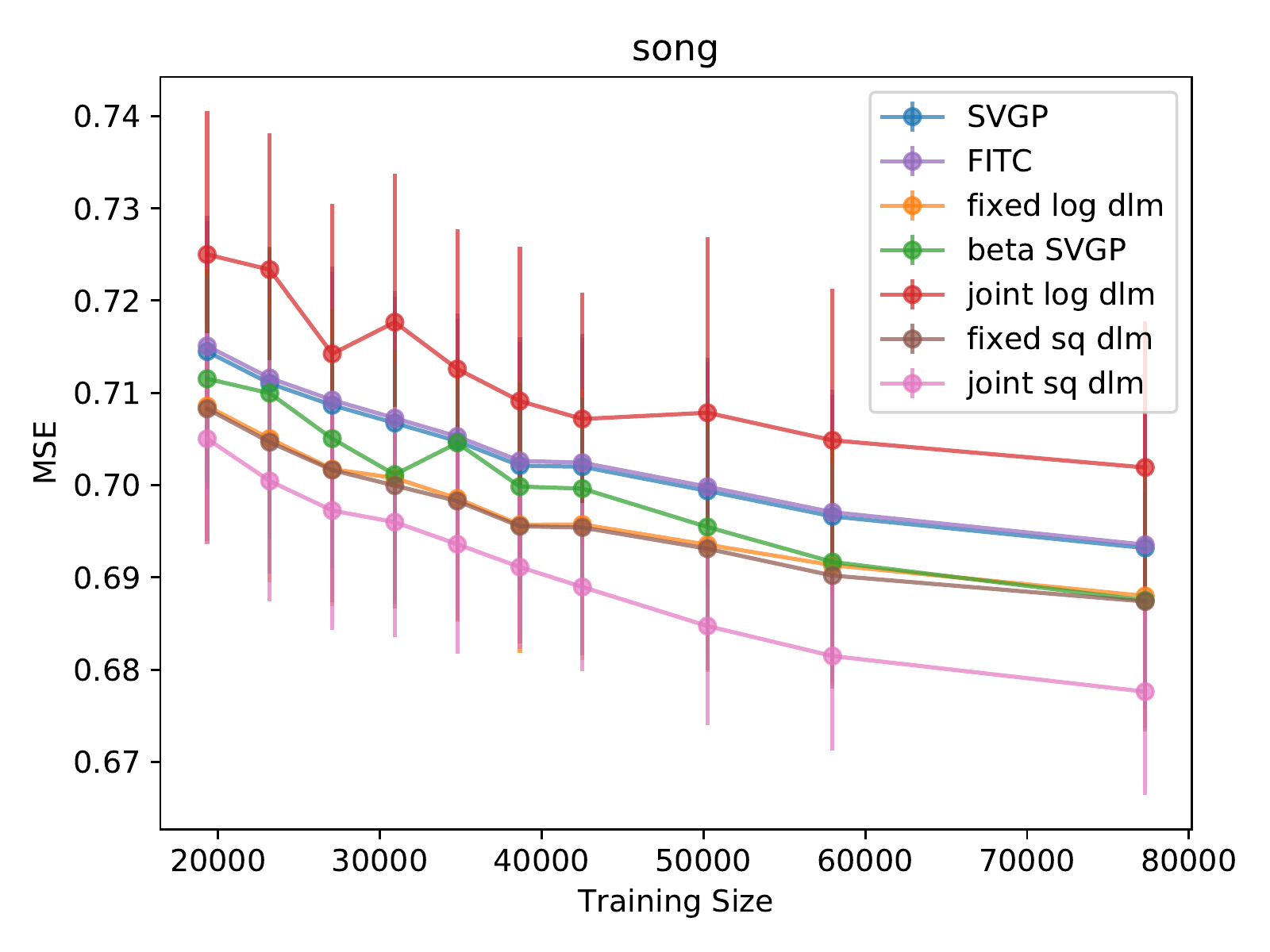}
\end{center}
\end{minipage}
\end{center}
\caption{Square loss in sGP Regression: Comparison of SVGP, FITC, DLM and SQ\_DLM in MSE. In all plots, lower values imply better performance.}
\label{#1}
\end{figure*}

}

\newcommand{\PutClassificationLog}[1]{
\begin{figure*}[h]
\begin{center}
\begin{minipage}{0.433\linewidth}
\begin{center}
    \includegraphics[width=0.99\linewidth]{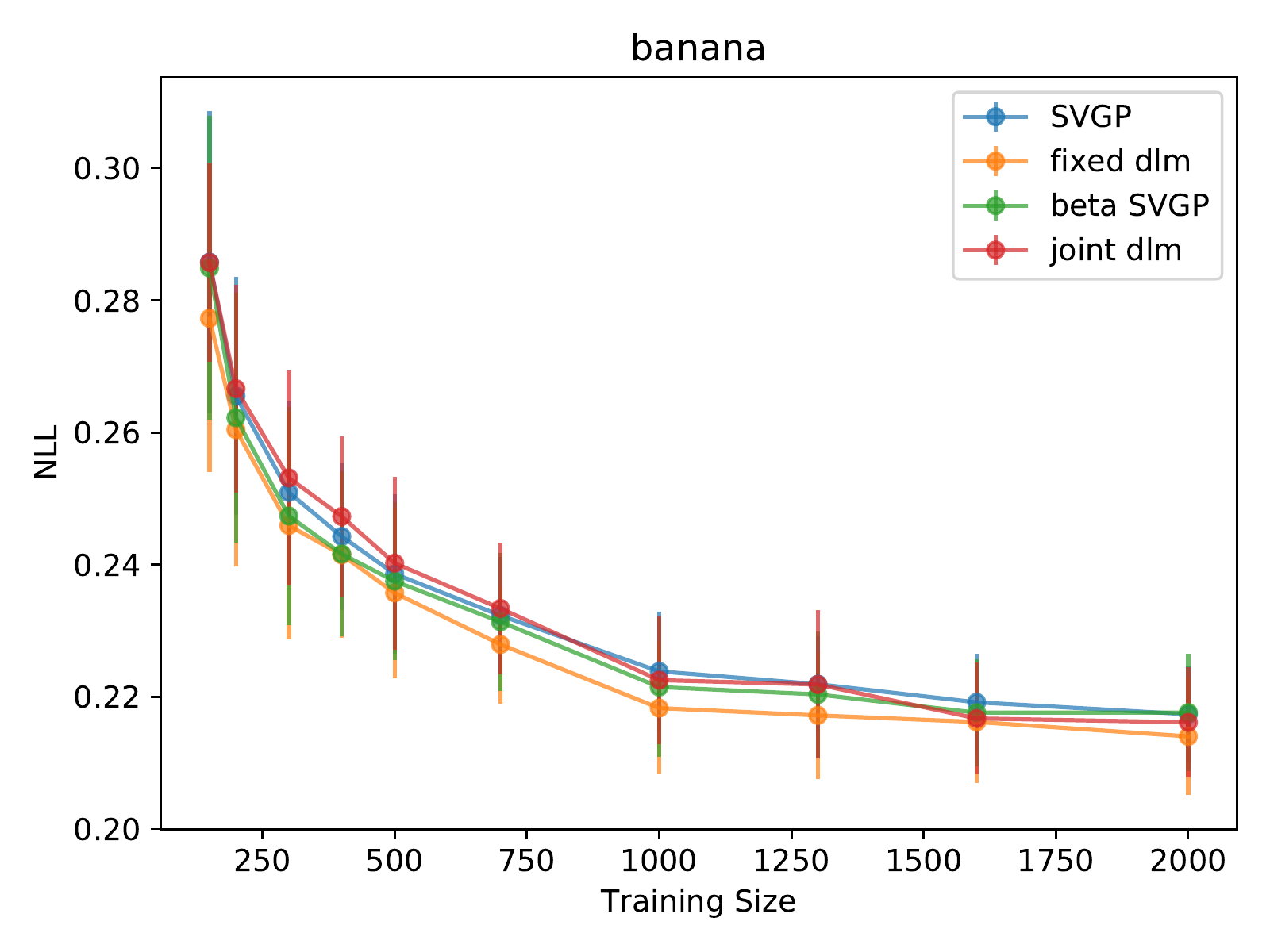}
\end{center}
\end{minipage}\quad
\begin{minipage}{0.433\linewidth}
\begin{center}
    \includegraphics[width=0.99\linewidth]{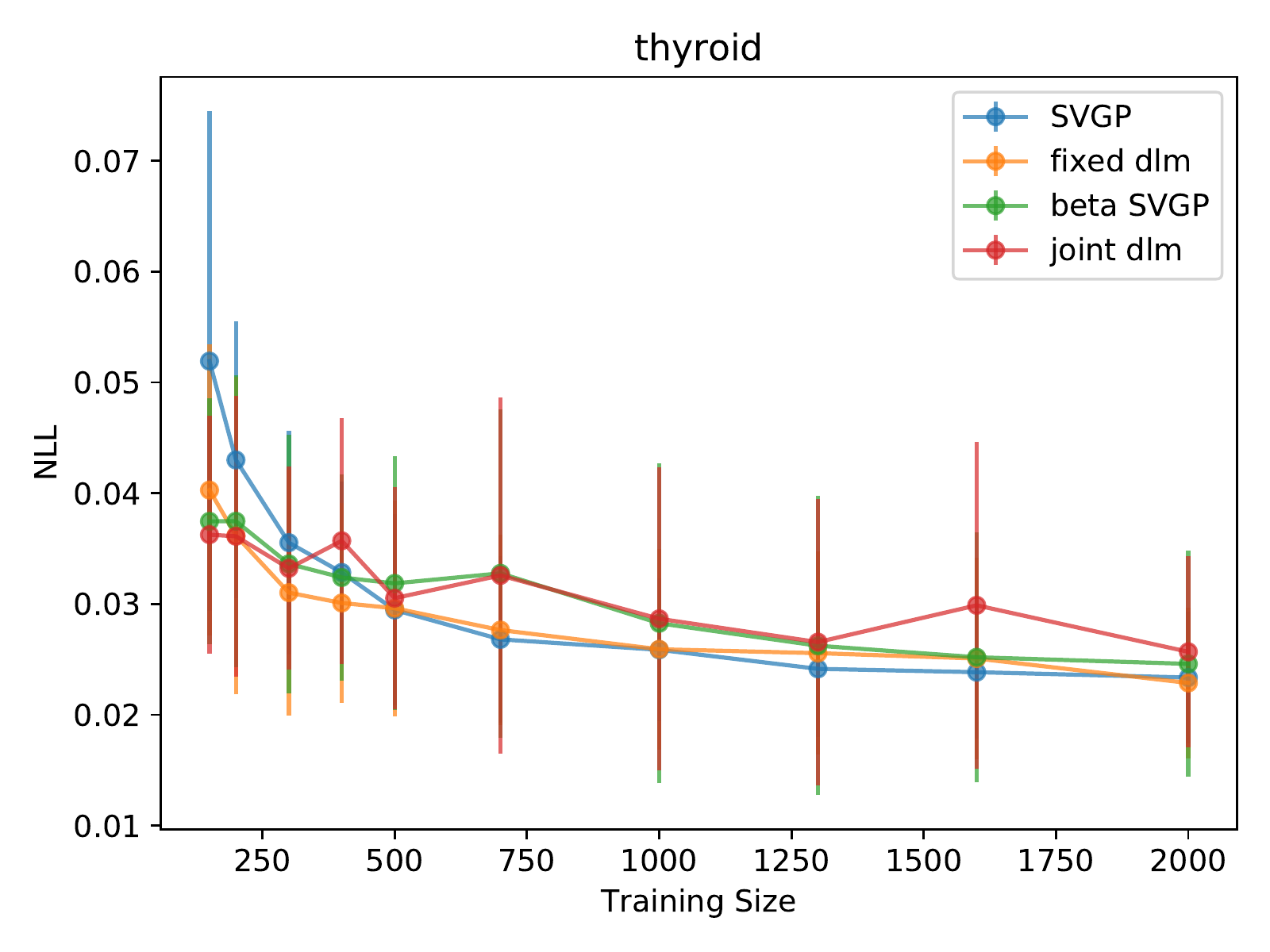}
\end{center}
\end{minipage}\quad
\begin{minipage}{0.433\linewidth}
\begin{center}
    \includegraphics[width=0.99\linewidth]{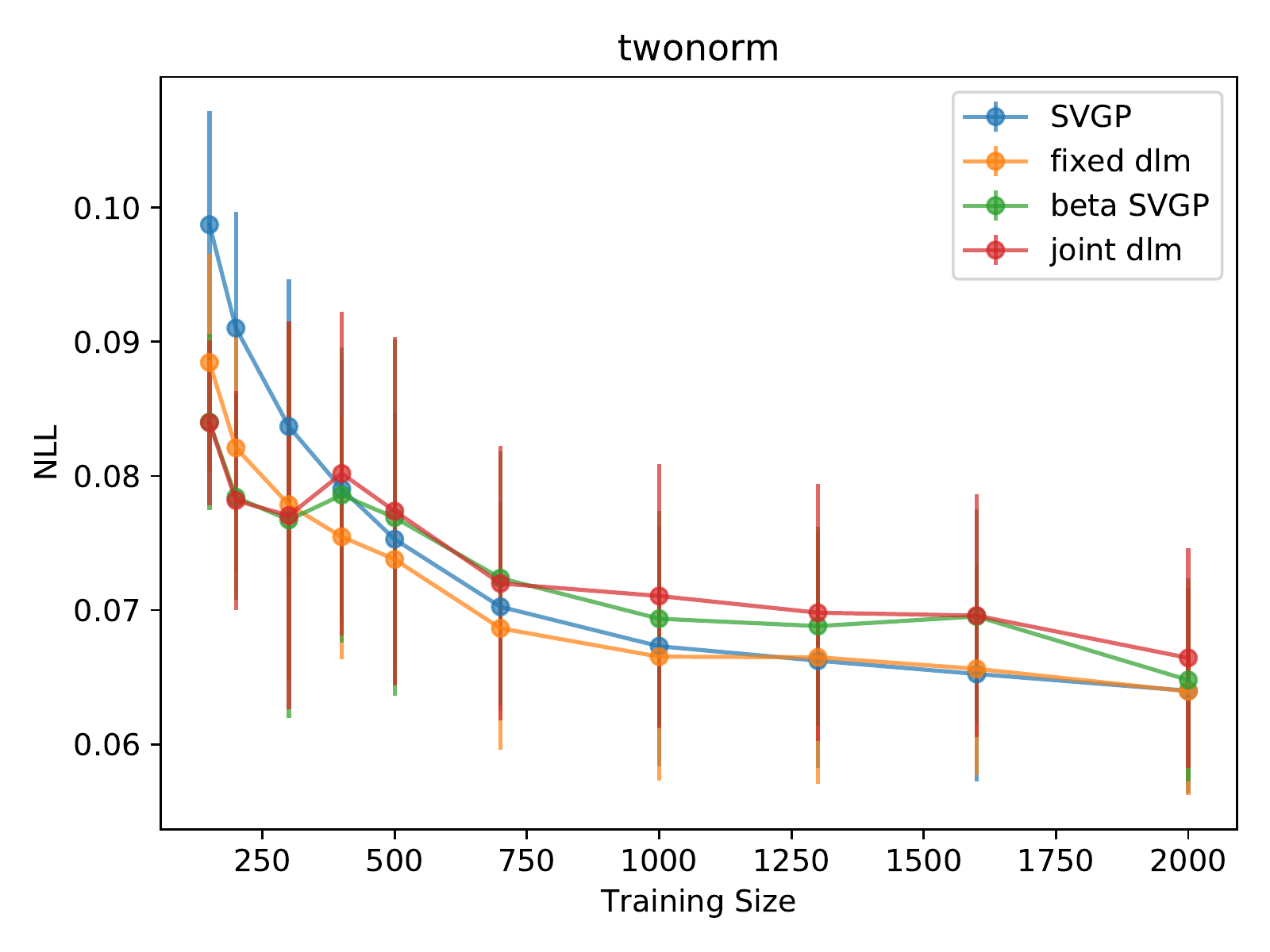}
\end{center}
\end{minipage}\quad
\begin{minipage}{0.433\linewidth}
\begin{center}
    \includegraphics[width=0.99\linewidth]{figs/log_ringnorm.pdf}
\end{center}
\end{minipage}
\end{center}
\caption{sGP Classification: Comparison of SVGP and DLM in mean NLL. In all plots, lower values imply better performance.}
\label{#1}
\end{figure*}

}

\newcommand{\PutClassificationErr}[1]{
\begin{figure*}[h]
\begin{center}
\begin{minipage}{0.433\linewidth}
\begin{center}
    \includegraphics[width=0.99\linewidth]{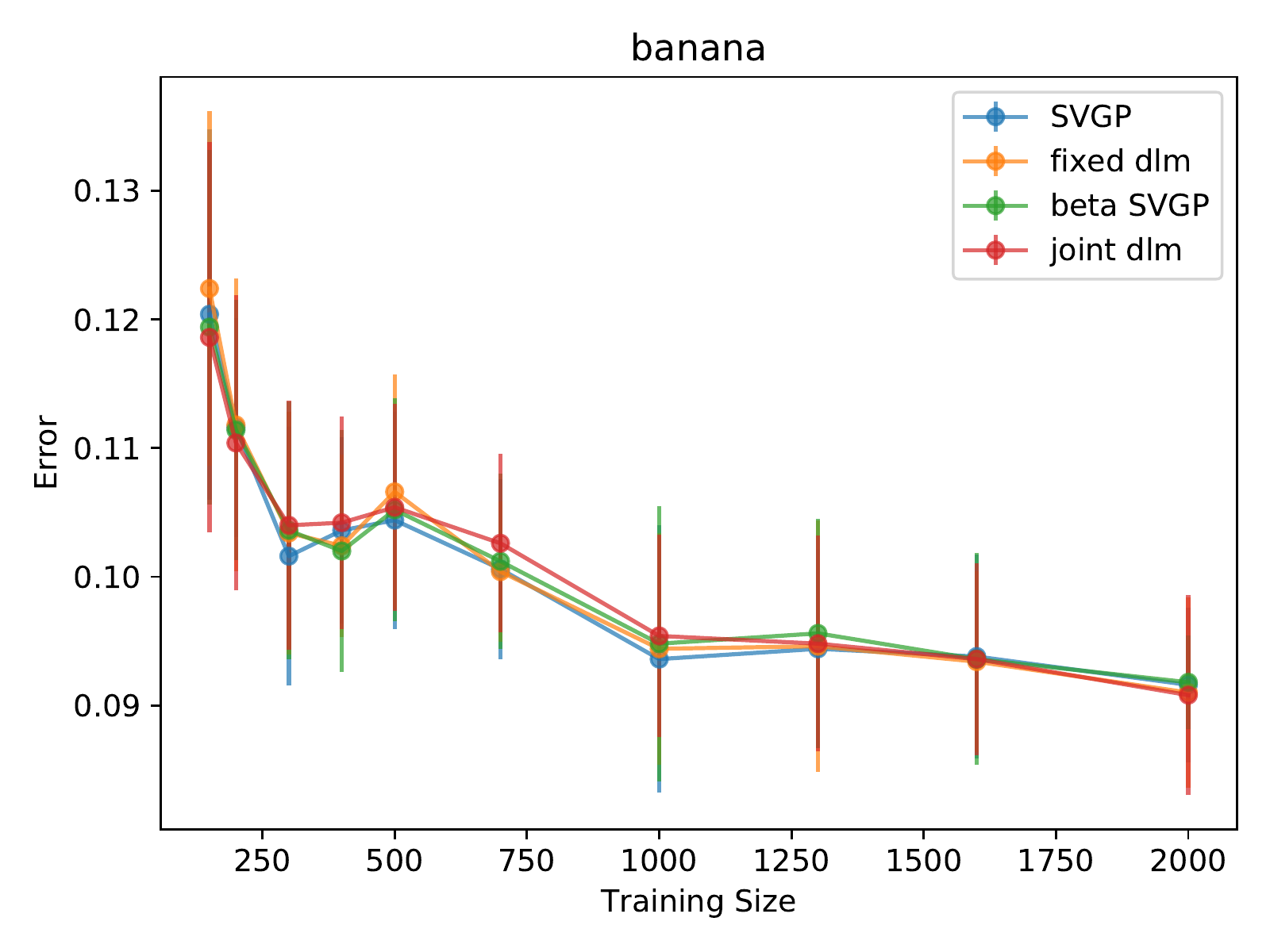}
\end{center}
\end{minipage}\quad
\begin{minipage}{0.433\linewidth}
\begin{center}
    \includegraphics[width=0.99\linewidth]{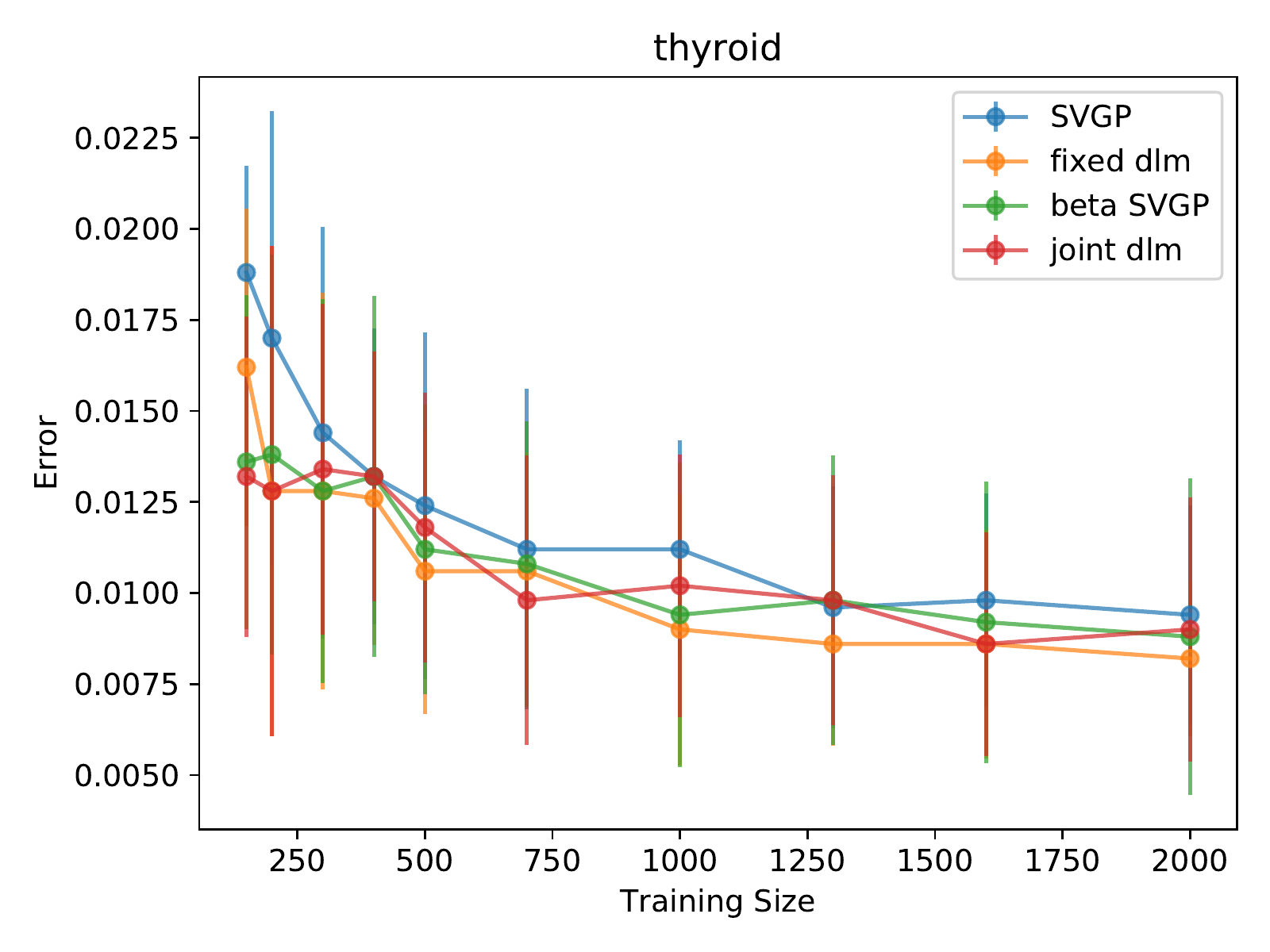}
\end{center}
\end{minipage}\quad
\begin{minipage}{0.433\linewidth}
\begin{center}
    \includegraphics[width=0.99\linewidth]{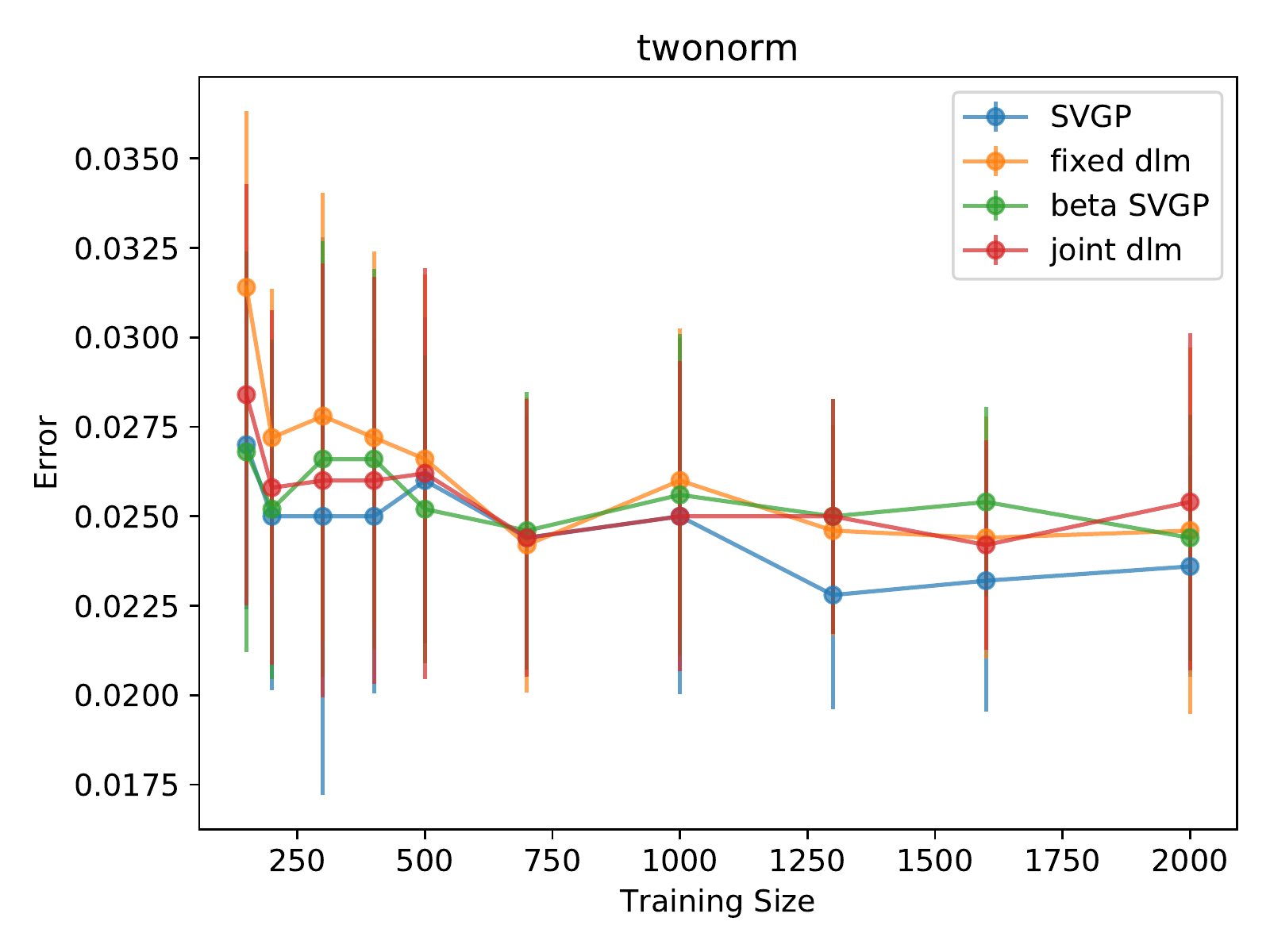}
\end{center}
\end{minipage}\quad
\begin{minipage}{0.433\linewidth}
\begin{center}
    \includegraphics[width=0.99\linewidth]{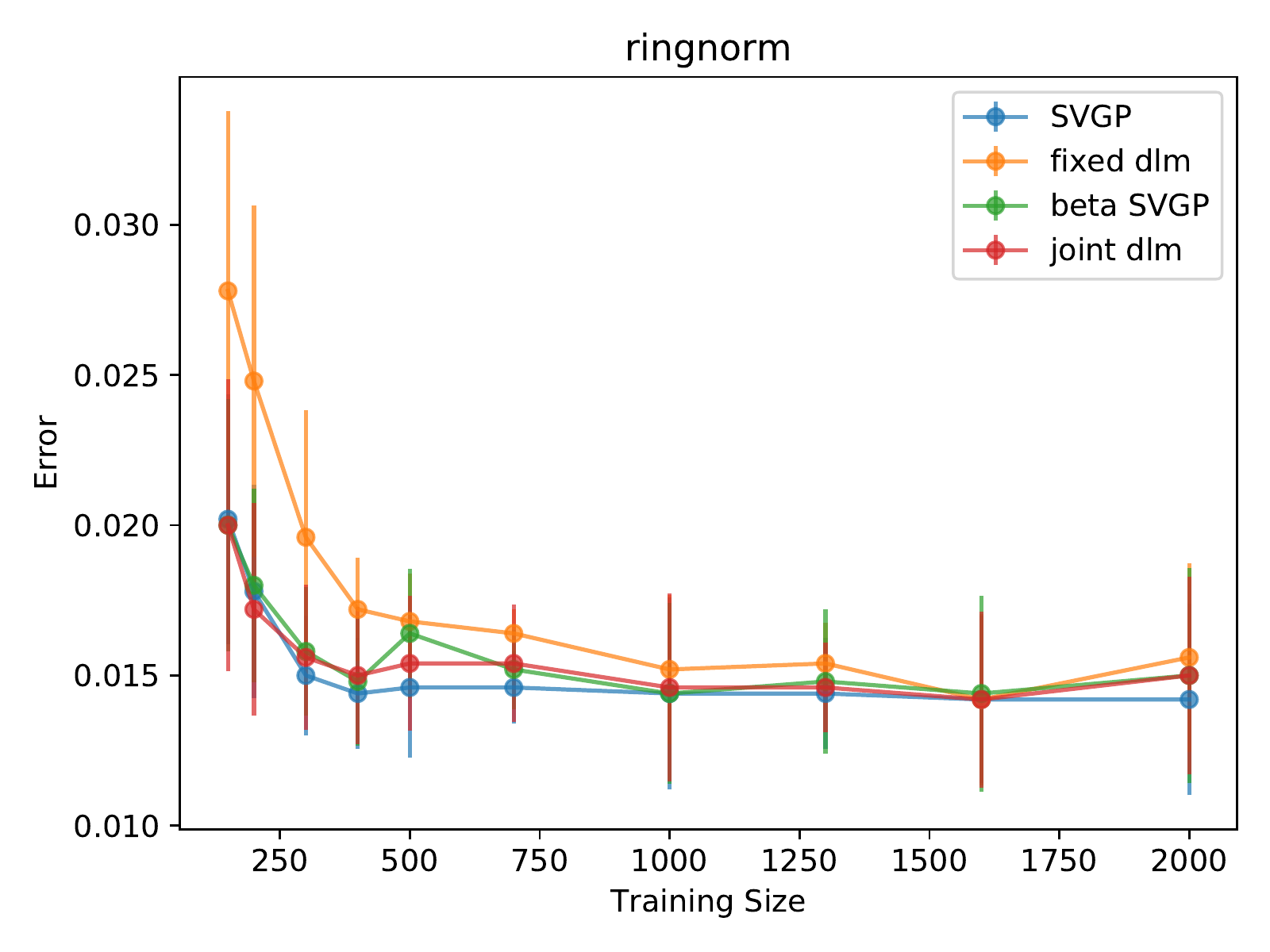}
\end{center}
\end{minipage}
\end{center}
\caption{sGP Classification: Comparison of SVGP and DLM  in term of mean error. In all plots, lower values imply better performance.}
\label{#1}
\end{figure*}

}

\newcommand{\PutSineDLM}[1]{
\begin{figure*}[th]
\begin{center}
\begin{minipage}{0.403\linewidth}
\begin{center}
    \includegraphics[width=0.99\linewidth]{figs/sine_opt_log_dlm.pdf}
\end{center}
\end{minipage}\quad
\begin{minipage}{0.403\linewidth}
\begin{center}
    \includegraphics[width=0.99\linewidth]{figs/sine_svgp_init_log_dlm.pdf}
\end{center}
\end{minipage}\quad
\begin{minipage}{0.403\linewidth}
\begin{center}
    \includegraphics[width=0.99\linewidth]{figs/sine_opt_log_dlm_reg.pdf}
\end{center}
\end{minipage}\quad
\begin{minipage}{0.403\linewidth}
\begin{center}
    \includegraphics[width=0.99\linewidth]{figs/sine_svgp_init_log_dlm_reg.pdf}
\end{center}
\end{minipage}
\end{center}
\caption{Comparison of DLM with two settings. Plots on the left show results with all parameters optimized by the DLM objective.
Plot on the right show results for DLM with hyperparameters chosen by SVGP.}
\label{#1}
\end{figure*}

}

\newcommand{\PutLossAndEta}[1]{
\begin{figure*}[h]
\begin{center}
\begin{minipage}{0.433\linewidth}
\begin{center}
    \includegraphics[width=0.99\linewidth]{figs/log_cadata_size691.pdf}
\end{center}
\end{minipage}\quad
\begin{minipage}{0.433\linewidth}
\begin{center}
    \includegraphics[width=0.99\linewidth]{figs/sq_cadata_size691.pdf}
\end{center}
\end{minipage}\quad
\begin{minipage}{0.433\linewidth}
\begin{center}
    \includegraphics[width=0.99\linewidth]{figs/log_cadata_size11063.pdf}
\end{center}
\end{minipage}\quad
\begin{minipage}{0.433\linewidth}
\begin{center}
    \includegraphics[width=0.99\linewidth]{figs/sq_cadata_size11063.pdf}
\end{center}
\end{minipage}
\end{center}
\caption{Comparison of validation error of different regularization parameters $\eta$. Error type (NLL or MSE) and training size are stated in titles.}
\label{#1}
\end{figure*}

}

\newcommand{\PutPoissonLoss}[1]{
\begin{figure*}[h]
\begin{center}
\begin{minipage}{0.433\linewidth}
\begin{center}
    \includegraphics[width=0.99\linewidth]{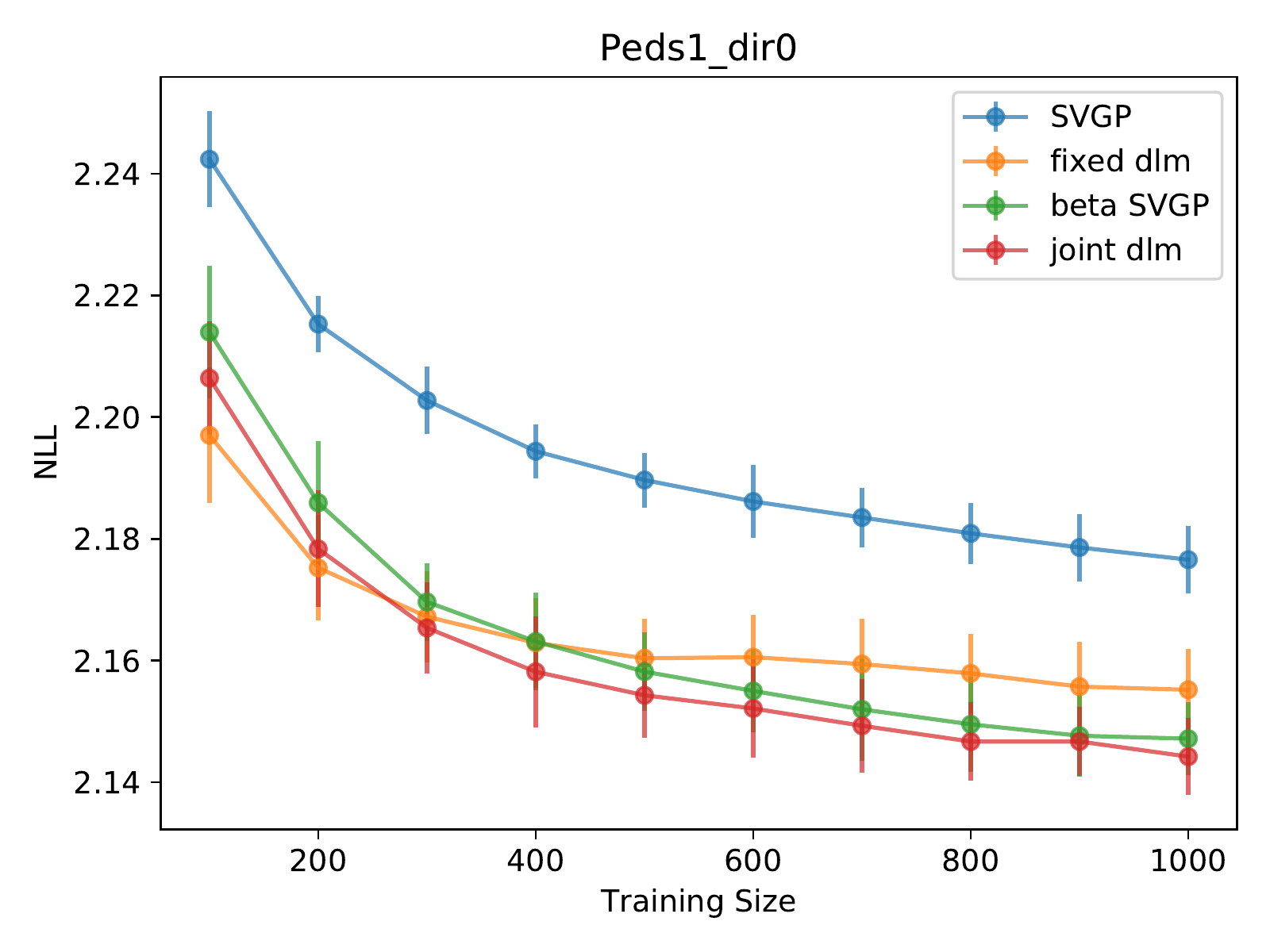}
\end{center}
\end{minipage}\quad
\begin{minipage}{0.433\linewidth}
\begin{center}
    \includegraphics[width=0.99\linewidth]{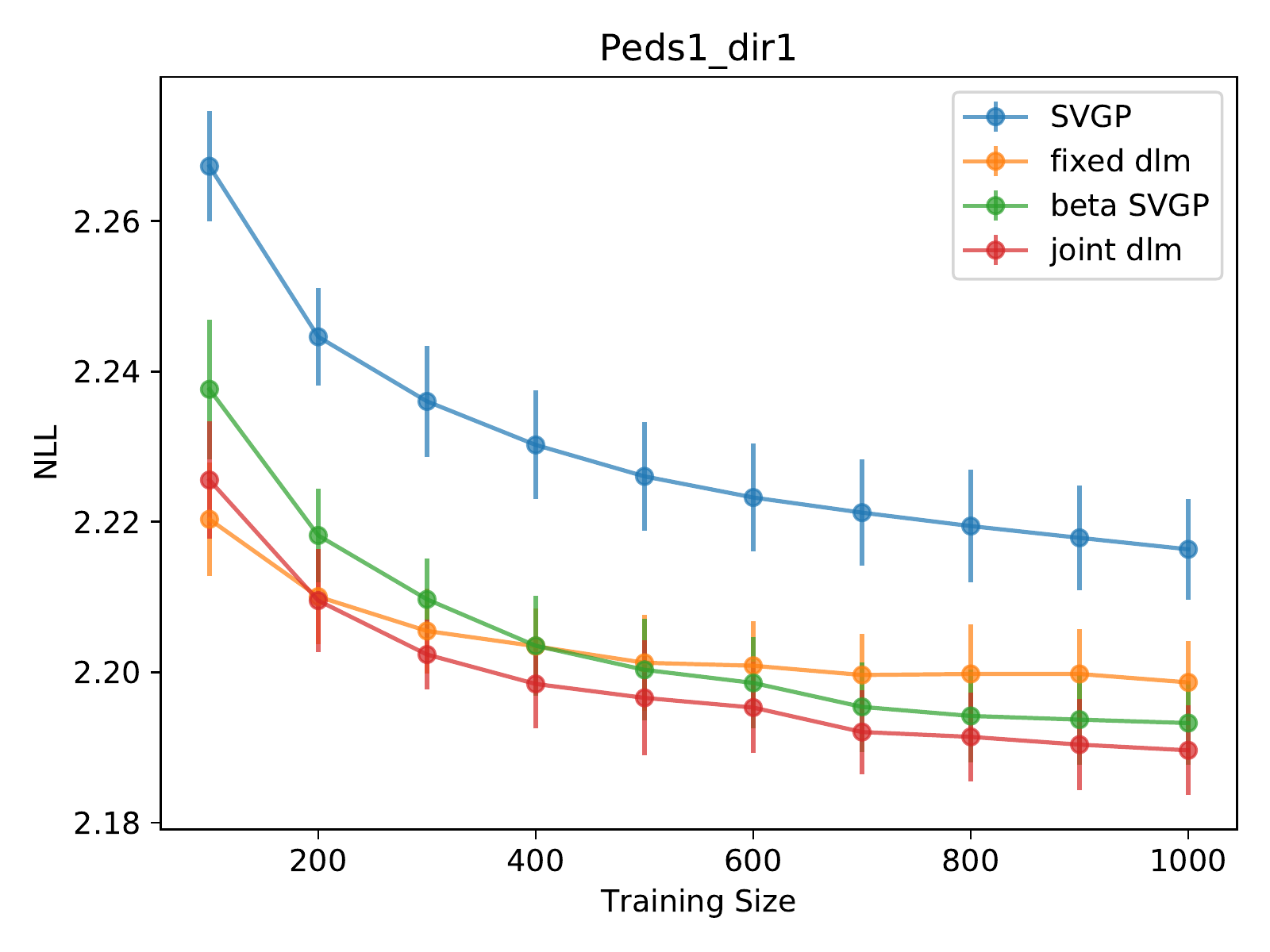}
\end{center}
\end{minipage}\quad
\begin{minipage}{0.433\linewidth}
\begin{center}
    \includegraphics[width=0.99\linewidth]{figs/log_Peds1_dir2.pdf}
\end{center}
\end{minipage}\quad
\begin{minipage}{0.433\linewidth}
\begin{center}
    \includegraphics[width=0.99\linewidth]{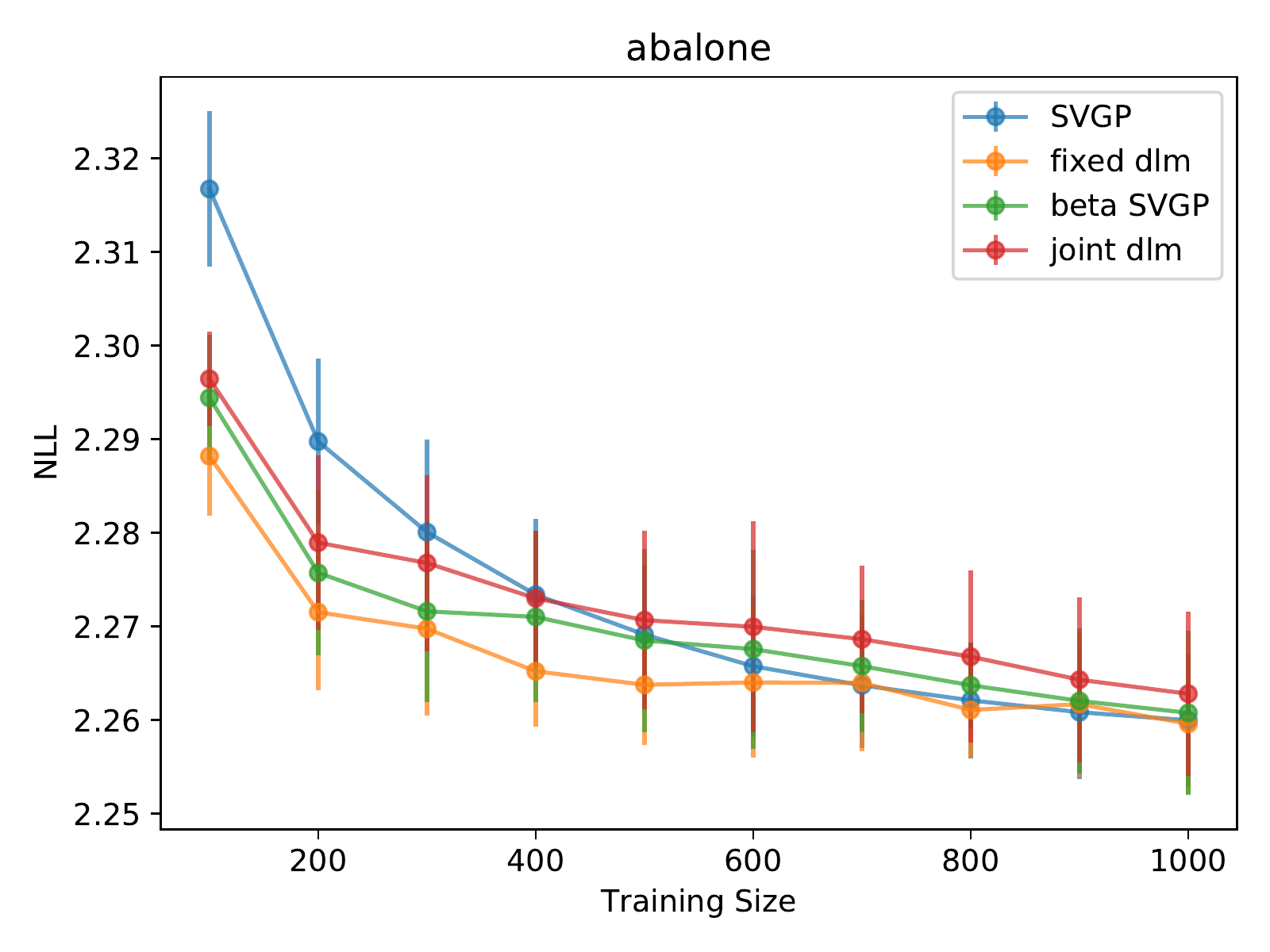}
\end{center}
\end{minipage}
\end{center}
\caption{sGP Count Prediction: Comparison of SVGP and DLM  with 10 MC samples in terms of mean NLL. In all plots, lower values imply better performance.}
\label{#1}
\end{figure*}

}

\newcommand{\PutPoissonMRE}[1]{
\begin{figure*}[h]
\begin{center}
\begin{minipage}{0.433\linewidth}
\begin{center}
    \includegraphics[width=0.99\linewidth]{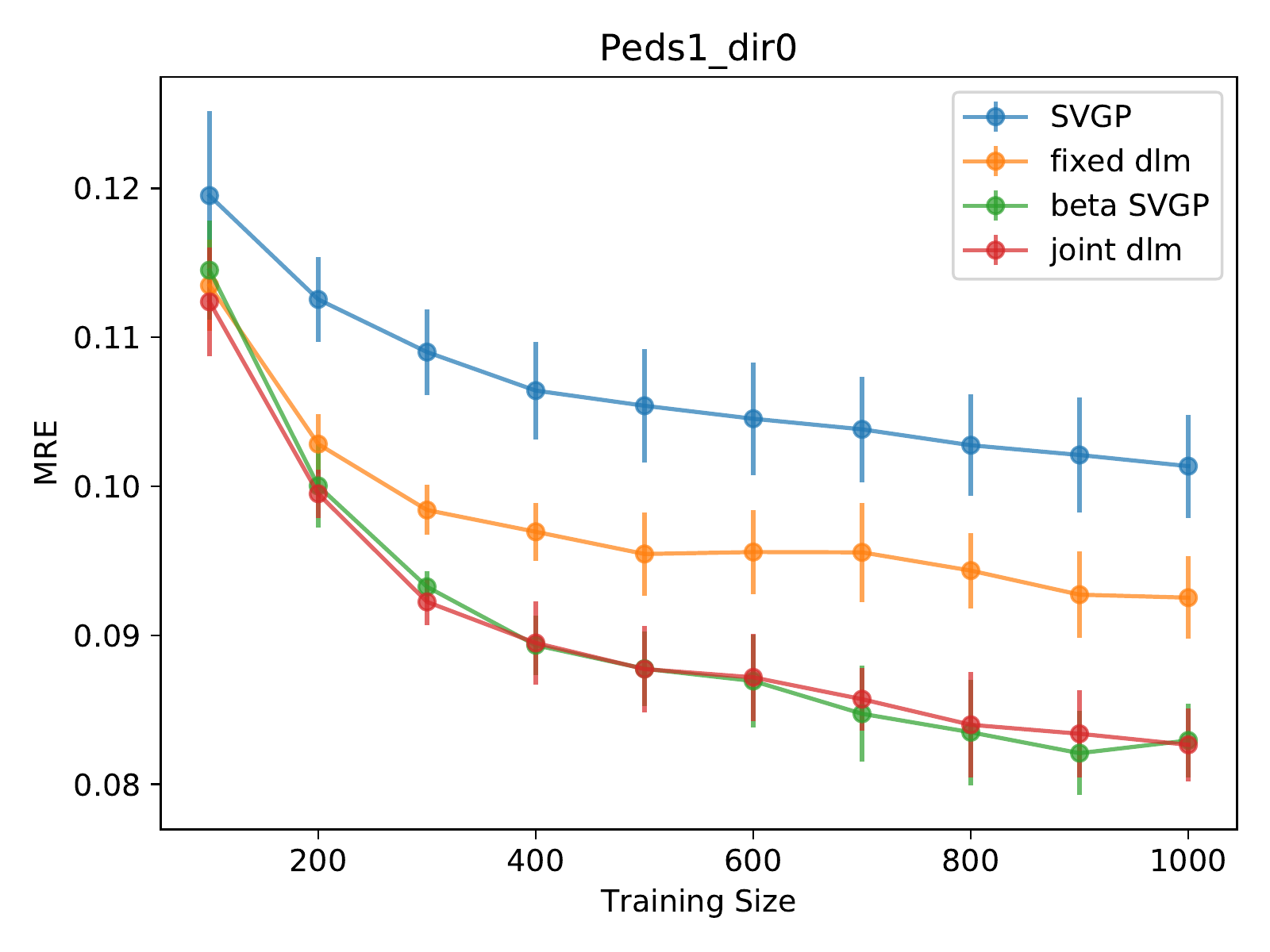}
\end{center}
\end{minipage}\quad
\begin{minipage}{0.433\linewidth}
\begin{center}
    \includegraphics[width=0.99\linewidth]{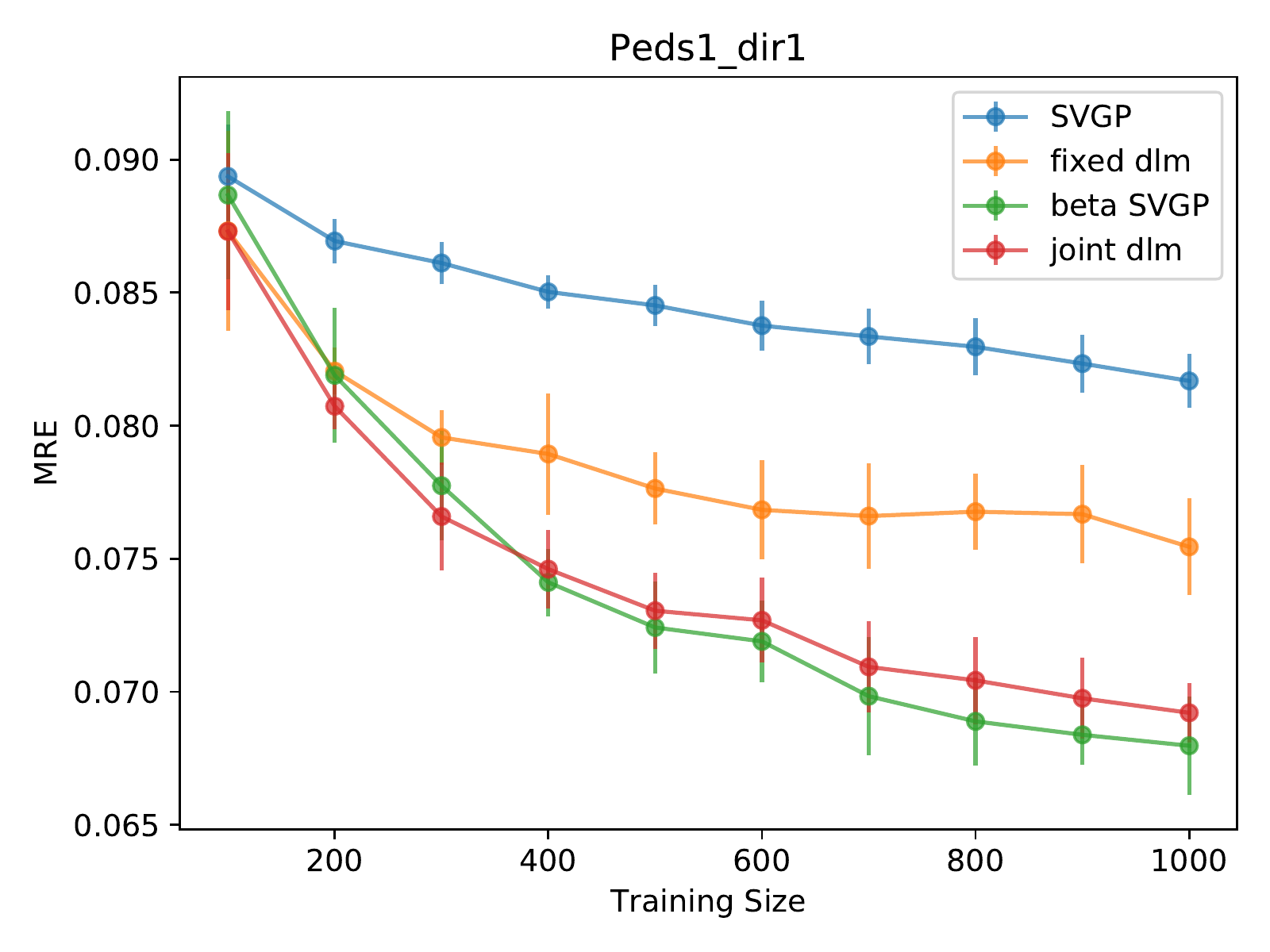}
\end{center}
\end{minipage}\quad
\begin{minipage}{0.433\linewidth}
\begin{center}
    \includegraphics[width=0.99\linewidth]{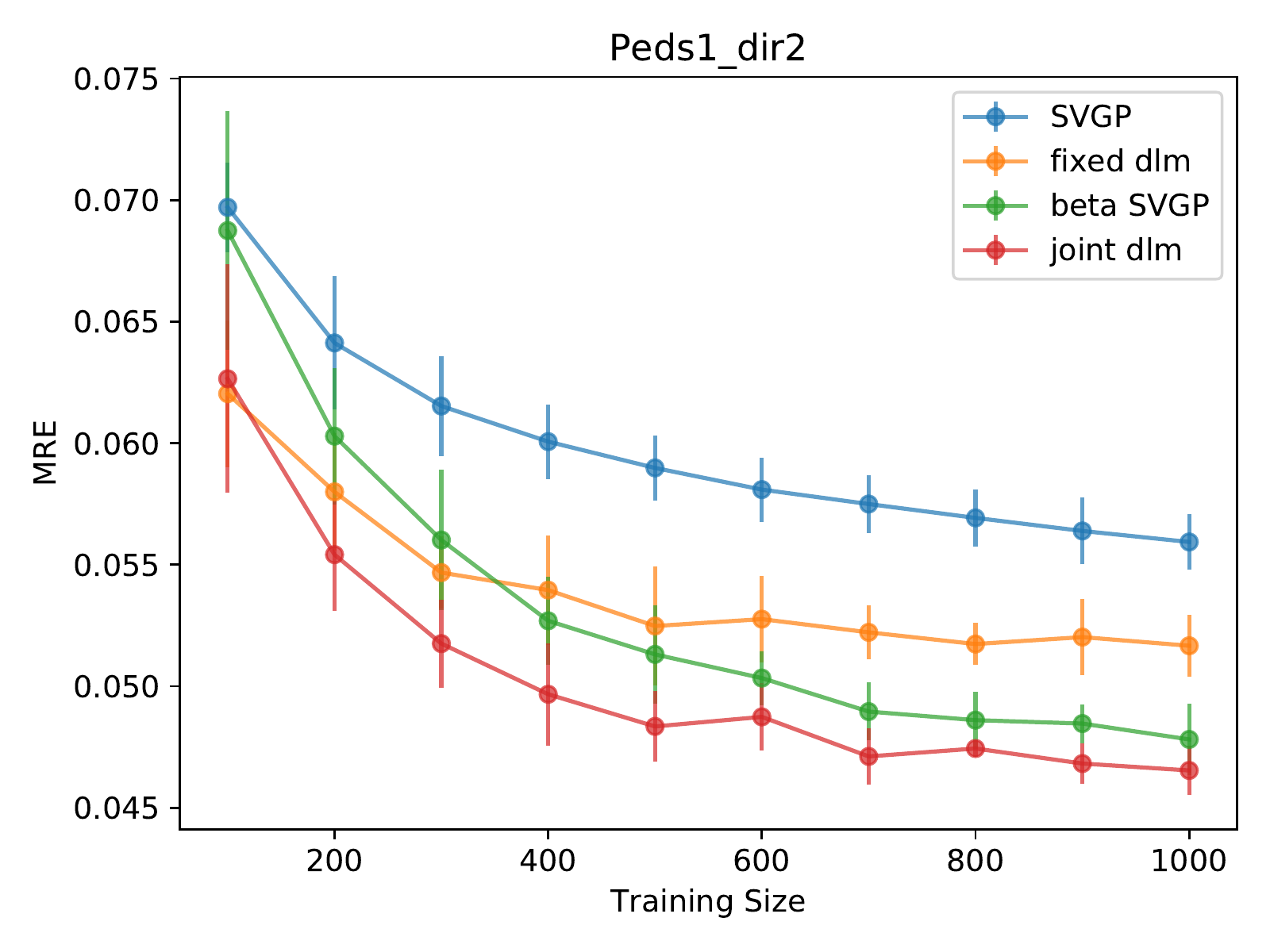}
\end{center}
\end{minipage}\quad
\begin{minipage}{0.433\linewidth}
\begin{center}
    \includegraphics[width=0.99\linewidth]{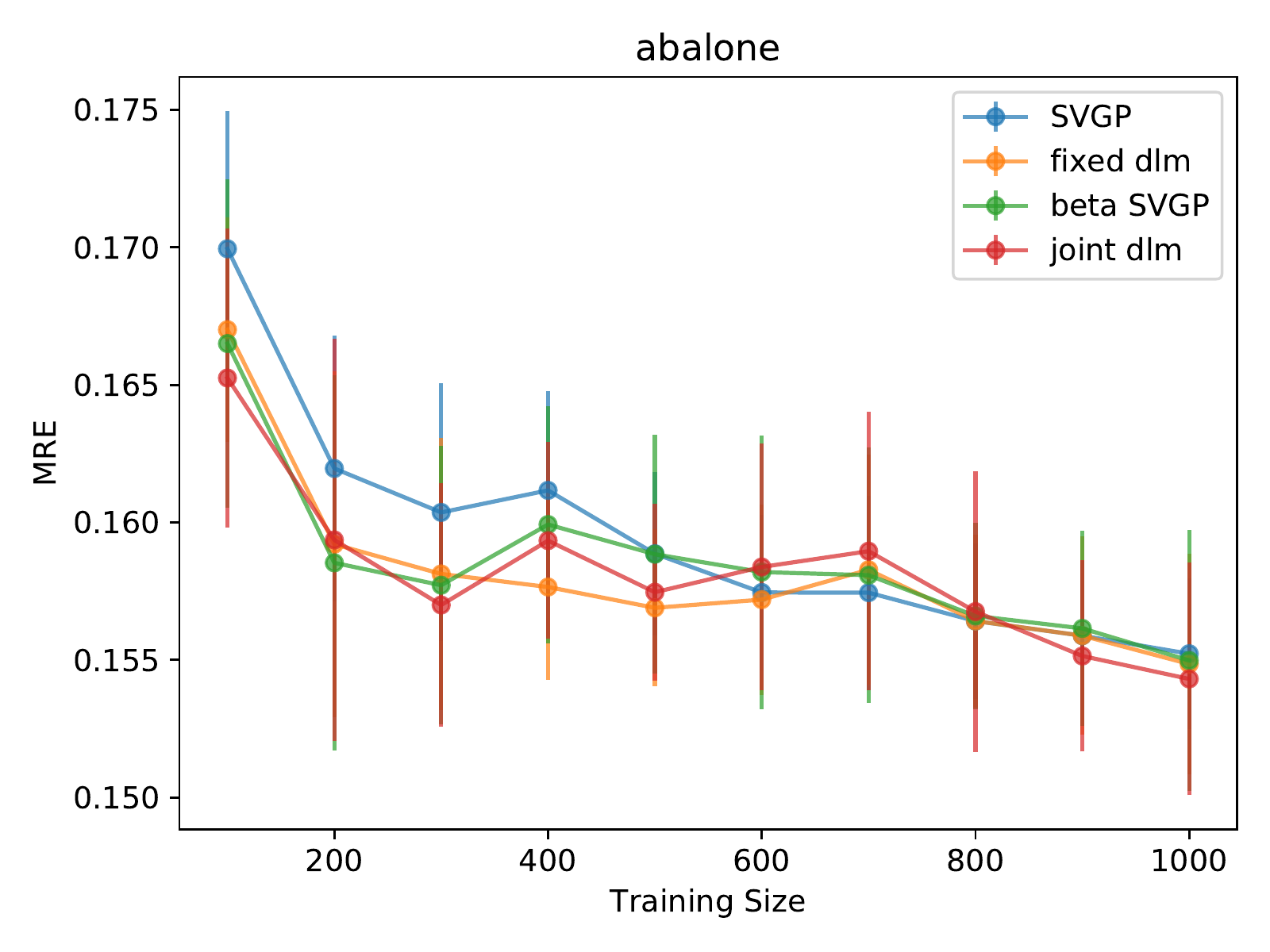}
\end{center}
\end{minipage}
\end{center}
\caption{sGP Count Prediction: Comparison of SVGP and DLM  with 10 MC samples in terms of MRE. In all plots, lower values imply better performance.}
\label{#1}
\end{figure*}
}

%% file: abstract.tex
\begin{abstract}
The paper provides a thorough investigation of Direct loss minimization (DLM), which optimizes the posterior to minimize predictive loss,  in sparse Gaussian processes. For the conjugate case, we consider DLM for log-loss and DLM for square loss showing a significant performance improvement in both cases. The application of DLM in non-conjugate cases is more complex because the logarithm of expectation in the log-loss DLM objective is often intractable and simple sampling leads to biased estimates of gradients.  The paper makes two technical contributions to address this. First, a new method using product sampling is proposed, which gives unbiased estimates of gradients (uPS) for the objective function. Second, a theoretical analysis of biased Monte Carlo estimates (bMC) shows that stochastic gradient descent converges despite the biased gradients. Experiments demonstrate empirical success of DLM. A comparison of the sampling methods shows that, while uPS is potentially more sample-efficient, bMC provides a better tradeoff in terms of convergence time and computational efficiency.
\end{abstract}

%% file: introduction.tex
\section{Introduction}

Bayesian models provide an attractive approach for learning from data. 
Assuming that model assumptions are correct, given the data and prior one can calculate a posterior distribution that compactly captures all our knowledge about the problem. Then, given a prediction task with an associated loss for wrong predictions, we can pick the best action according to our posterior. 
This is less clear, however, when exact inference is not possible.
As argued by several authors  (e.g.,\ \cite{Lacoste-JulienHG11,StoyanovRE11}),
in this case it makes sense to optimize the choice of approximate posterior so as to minimize the expected loss of the learner in the future.
This requires using the loss function directly during training of the model. 
Following \cite{Sheth2019pseudo} %
we call this approach {\em direct loss minimization} (DLM).
In this paper we explore the potential of DLM to improve performance in practice, in the context of sparse Gaussian Processes (sGP), and in the process make technical contributions to the problem of gradient estimation for log-expectation terms.

To motivate the discussion consider a model with latent variables $z$ and observations $y$, generating examples via $p(z)\prod p(y_i|z_i)$. When calculating the posterior $p(z|y)$ is hard, variational inference finds an approximation $q(z)$ by maximizing the evidence lower bound (ELBO) or minimizing its negation:
\begin{align*}
-\log p(y) 
 &\leq 
- \int  q(z) \log \left( \frac{ p(z)}{q(z)} \prod_i p(y_i|z_i) \right) dz \\
 &=  \sum_i E_{q(z_i)} [-\log p(y_i|z_i) ]  +\beta \ d_{KL}(q(z) \| p(z))  
 \end{align*}
where $d_{KL}$ is the Kullback-Leibler divergence, and $\beta=1$ (but we discuss other values of $\beta$ below).
From this perspective variational inference is seen to perform regularized loss minimization, 
with $d_{KL}$ as the regularizer. 
But viewed in this manner the loss on example $i$ is assumed to be $E_{q(z_i)} [-\log p(y_i|z_i) ]$ which is not the intended process for a Bayesian predictor. 
Instead, given a posterior, $q(z)$, the Bayesian algorithm first calculates its predictive distribution $q(y_i)=E_{q(z_i)} [p(y_i|z_i) ]$, potentially calculates a prediction $\hat{y}_i$, and then suffers 
a loss that depends on the context in which the algorithm is used. 
For the case of log-loss, where $\hat{y}_i$ is not used, the loss term is 
$-\log q(y_i)$ and the corresponding objective is  
\begin{align*}
&\mbox{LogLoss DLM objective} \\
&=\sum_i -\log E_{q(z_i)}[p(y_i|z_i)] + \beta \ d_{KL}(q(z),p(z)).
\end{align*}
Comparing LogLoss DLM to the ELBO we see that the main difference is the log term which is applied before the expectation. 
On the other hand, if we care about  
square loss in the case of regression, the training criterion becomes 
\begin{align*}
&\mbox{squareLoss DLM objective} \\
&=\sum_i (\hat{y}_i- y_i)^2 + \beta \ d_{KL}(q(z),p(z))
\end{align*}
and other losses will similarly lead to different objectives. 
This distinction is in contrast with some previous work that aims to find the best posterior without regard to its intended use. 
Our experimental evaluation shows that this distinction is important in practice.

Applying DLM for non-conjugate prediction raises 
the difficulty of optimizing objectives including $\log E_{q(z_i)}[p(y_i|z_i)]$.
The standard  Monte Carlo estimate of the objective, $\log \frac{1}{L} \sum_k p(y_i|z_i^{(k)})$, where $z_i^{(k)}\sim q(z_i)$ (or its reparameterized version) is biased leading to biased gradients --- we call this approach bMC.
We make two technical contributions in this context.
The first is a new method, uPS, for unbiased estimates of gradients for objectives with log-expectation terms through Product Sampling. 
The method is general and we develop a practical version 
for the case when $q(z_i)$ is Gaussian. 
Our second contribution is a theoretical analysis of bMC, 
showing that (under some technical conditions) stochastic gradient descent using bMC gradients converges 
despite the bias. 
bMC has been used in some prior work either explicitly or implicitly and therefore the result may be of independent interest. 

An empirical evaluation in sGP for regression, classification and count prediction
compares log-loss DLM, to ELBO, as well as  $\beta$-ELBO (which explicitly optimizes the regularization parameter for ELBO). 
The evaluation
shows that DLM is an effective approach which in some cases matches and in some cases significantly improves over the performance of variational inference and $\beta$-ELBO. 
Results comparing the sampling methods show that uPS is potentially more sample-efficient but bMC provides a better tradeoff in terms of convergence time and computational efficiency.

To summarize, the paper develops new analysis for sampling methods and optimization with log-expectation terms,
shows how this can be incorporated in DLM for sGP, and shows empirically that DLM has the potential for significant performance improvements over ELBO.

%% file: sparse-GP-details.tex
\section{ELBO and DLM for Sparse GP}

In this section we review sGP and the development of ELBO and DLM for this model. 
The GP \citep{Rasmussen2006} is a flexible Bayesian model capturing functions over arbitrary spaces but the complexity of inference in GP is cubic in the number of examples $n$.
Sparse GP solutions reduce this complexity to $O(M^2 n)$ where $M$ is the number of pseudo inputs which serve as an approximate sufficient statistic for prediction. The two approaches most widely used are FITC \citep{snelson2006sparse} and the variational solution of \cite{titsias2009variational}.
The variational solution has been extended for large datasets and general likelihoods and is known as SVGP \citep{Hensman2013,hensman2015scalable,Sheth2015,bauer2016understanding}.

In sGP, the GP prior jointly generates the pseudo values $u$ and the latent variables $f$  
which we write as $p(u)p(f|u)$ 
and the observations $y=\{y_i\}$ are generated from the likelihood model $p(y_i|f_i)$.
Most previous works use a restricted form for the posterior $q(u,f)=q(u)p(f|u)$ where where $q(u)={\cal N}(m,V)$ is Gaussian and where the conditional $p(f|u)$ remains fixed from the prior. 
Although sGP is slightly more general than the model discussed in the introduction a similar derivation yields the same forms for ELBO and DLM as above, where the loss term in the ELBO is 
$E_{q(u)p(f_i|u)}[-\log p(y_i|f_i)]=E_{q(f_i)}[-\log p(y_i|f_i)]$. SVGP optimizes the objective through reparameterization. The collapsed form  \citep{titsias2009variational}
for the regression case uses the fact that $E_{q(f_i)}[-\log p(y_i|f_i)]$ has an analytic solution and through it derives an analytic solution for $m,V$ so that only hyperparameters need to be optimized explicitly. FITC \citep{snelson2006sparse} is not specified using the same family of objective functions but has a related collapsed form which is used in our experiments. 

The log-loss term for DLM is 
$-\log  E_{q(u)p(f_i|u)}[p(y_i|f_i)]$ $= -\log   E_{q(f_i)}[p(y_i|f_i)]$ $= -\log  q(y_i)$.
Since both $q(u)$ and $p(f_i|u)$ are Gaussian distributions, the marginal $q(f_i)$ is also Gaussian with mean $\mu_i = K_{iu} K_{uu}^{-1} m$ and variance $v_i = K_{ii} + K_{iu} K_{uu}^{-1} (V - K_{uu}) K_{uu}^{-1} K_{ui}$
where $K_{uu}=K(u,u)$, $K_{iu}=K(x_i,u)$ etc.

In the following we consider log loss for regression, binary prediction through Probit regression and count prediction through Poisson regression. 
For regression we have $p(y_i|f_i)={\cal N}(f_i,\sigma_n^2)$ and the loss term is 
$-\log q(y_i)
=-\log {\cal N} (y_i | \mu_i, v_i + \sigma_n^2)$.
For probit regression $p(y_i=1|f_i)=\Phi(f_i)$ where $\Phi(f)$ is the CDF of the standard normal distribution. 
Here we have for $y_i\in\{0,1\}$, 
$-\log q(y_i)
=-\log \Phi \left(\frac{(2 y_i - 1) \mu_i}{\sqrt{v_i + 1}} \right)$.
In both cases we can calculate derivatives directly through $-\log q(y_i)$.
For Poisson regression (with log link function) we have 
$p(y_i|f_i)=e^{-e^{f_i}}e^{y_i f_i}/y_i!$ and we do not have a closed form for $q(y_i)$. In this case we must resort to sampling when optimizing the DLM objective. 

For square loss, $q(y_i)$ is the same as in the regression case, but calculating the loss requires optimal prediction $\hat{y}_i$.
In this case, the optimal prediction is the mean of the predictive distribution,
that is 
$\hat{y}_i=K_{iu}  K_{uu}^{-1}  m$.
Therefore the loss term in square loss DLM is $\frac{1}{2}(K_{iu}  K_{uu}^{-1} m -y_i)^2$. 
It is easy to show that the the optimization criterion simplifies into an objective that depends only on $m$, and the square loss DLM objective for sparse GP is 
$\frac{1}{2}\sum_i (K_{iu}  K_{uu}^{-1} m -y_i)^2 +  \frac{\beta}{2} m^T K_{uu}^{-1} m$.

To summarize, both ELBO and DLM include a loss term and KL regularization term. 
When the loss term is analytically tractable optimization can be performed as usual. When it is not, solutions use sampling where ELBO can use unbiased estimates of derivatives through reparameterization, but log-loss DLM has to compute derivatives for log-expectation terms which are more difficult.

%% file: productsampling.tex
\section{Unbiased Gradient Estimates}

In this section we develop a new approach for gradients of log-expectation terms. 
In particular, we describe an extension of a standard technique from the Reinforce algorithm \citep{Williams1992simple} that yields unbiased gradient estimates,
by sampling from a product of distributions.
The following proposition describes the technique. 
\begin{proposition}
    The estimate
    \begin{equation}
        \hat{G}(\theta) 
        = 
        \nabla_\theta \log q(f^{(l)}|\theta),
        \label{eq:unbiased-G}
    \end{equation}
    where $f^{(l)}\sim \tilde{q}(f^{(l)}|\theta)$ and $\tilde{q}(f|\theta)=\frac{q(f|\theta)p(y|f)}{\E_{q(f|\theta)} p(y|f)}$, is an unbiased estimate of $\nabla_\theta \log \E_{q(f|\theta)} p(y|f)$.
\end{proposition}
\begin{proof}
The true derivative $G(\theta)=\nabla_\theta \log \E_{q(f|\theta)} p(y|f)$ is given by
\begin{equation}
    \frac{\nabla_\theta \E_{q(f|\theta)} p(y|f)}{\E_{q(f|\theta)} p(y|f)} = \frac{G_n(\theta)}{\E_{q(f|\theta)} p(y|f)}.
    \label{eq:true-deriv}
\end{equation}
We next observe using  \citep{Williams1992simple} that $G_n(\theta)$ can be written as
\begin{equation}
    G_n(\theta) = \E_{q(f|\theta)}\Big[p(y|f) \nabla_\theta \log q(f|\theta)\Big].
    \label{eq:rl-trick}
\end{equation}
The expectation of 
(\ref{eq:unbiased-G}) with respect to the sample $f^{(l)}$ is given by %
\begin{align*}
    &\E_{\tilde{q}(f^{(l)}|\theta)} \nabla_\theta \log q(f^{(l)}|\theta) \\
    & = 
    \int_{f^{(l)}} \Big[{\nabla_\theta \log q(f^{(l)}|\theta)}\Big] \frac{q(f^{(l)}|\theta)p(y|f^{(l)})}{C} \text{d}f^{(l)}
    \\
    & = 
    \frac{1}{C}
    \E_{q(f^{(l)}|\theta)}\Big[{p(y|f^{(l)}) \nabla_\theta \log q(f^{(l)}|\theta)}\Big]
    =
    \frac{G_n(\theta)}{C}
    =
    G(\theta),
\end{align*}
where $C=\E_{q(f|\theta)} p(y|f)$, and 
the second-to-last equality follows from the identity (\ref{eq:rl-trick}).
\end{proof}

The derivation in the lemma is general and does not depend on the form of $f$. However, the estimate can have high variance and  in addition the process of sampling can be expensive. 
In this paper we develop an effective rejection sampler for the case where $f$ is 1-dimensional and $q(f)={\cal N}(\mu,\sigma^2)$. We provide a sketch here and full details are given in the supplement. 
Let $\ell(f)=p(y|f)$. 
To avoid a high rejection rate we sample from $h_2(f)={\cal N}(\mu,n \sigma^2)$ with the same mean as $q()$ but larger variance. We optimize the width multiplier $n$ to balance rejection rate in the region between intersection points of $q()$ and $h_2()$ (where $q()$ is larger) and outside this region ($q()$ is smaller). It is easy to show that this gives a valid rejection sampler with $K=\mbox{max}_f \ell(f)$, that is, $h_2(f) K \geq q(f)\ell(f)$. 
This construction requires separate sampling for each example in a batch and significant speedup can be obtained by partly vectorizing the individual samples.

%% file: convergence.tex
\section{Convergence with Biased Gradients}

\newcommand{\draw}{\ell}
\newcommand{\draws}{L}
\newcommand{\fidraw}{f_i^{(\draw)}}
\newcommand{\eidraw}{\epsilon_i^{(\draw)}}
\newcommand{\smoothness}{\alpha}
\newcommand{\mcsample}{\{\epsilon^{(\draw)}\}_{\draw=1}^\draws}
\newcommand{\gap}[1]{\sqrt{\frac{1}{2\draws}\log \frac{#1}{\delta}}}
\newcommand{\gaptwo}[1]{\sqrt{\frac{1}{\draws}\log \frac{#1}{\delta}}}

This section shows that biased Monte Carlo estimates can be used to optimize the DLM objective. 
For presentation clarity, in this section we scale the objective by the number of examples $n$ to get $ -\frac{1}{n} \sum_i \log E_{q(f_i)}[p(y_i|f_i)] + \beta \frac{1}{n} \ d_{KL}(q(),p())$.\footnote{For the sparse GP case, the KL term is over the inducing inputs, whereas for the simpler model in the introduction, the KL term is over $f$.
} 
Let $r\defined(m,V)$
and consider the univariate distribution $q(f_i|r)\defined\mathcal{N}(f_i|a_{i,1}^\top m+b_{i,1},a_{i,2}^\top V a_{i,2}+b_{i,2})$
for known vector $a_{i,1},a_{i,2}$ and scalar constants $b_{i,1},b_{i,2}$.
This form includes many models including sGP. 
In the following, references to the parameter $V$ and gradients w.r.t.\ it should be understood as appropriately vectorized.
We consider the reparameterized objective
$
\objective_i(r)
=$
$-\log \E_{\mathcal{N}(\epsilon|0,1)} p(y_i|f_i=g_i(r,\epsilon))
$
and its gradient
\begin{align}
    \nabla_r \objective_i(r)
    &=
    -
    \frac{
        \nabla_r \E_{\mathcal{N}(\epsilon|0,1)} p(y_i|f_i=g_i(r,\epsilon))
    }{
        \E_{\mathcal{N}(\epsilon|0,1)} p(y_i|f_i=g_i(r,\epsilon))
    } \nonumber\\
    &=
    -
    \frac{
        \E_{\mathcal{N}(\epsilon|0,1)}\Big[
            \frac{\partial}{\partial f_i} \big[ p(y_i|f_i=g_i(r,\epsilon)) \big]
            \nabla_r g_i(r,\epsilon)
        \Big]
    }{
        \E_{\mathcal{N}(\epsilon|0,1)} p(y_i|f_i=g_i(r,\epsilon))
    }
    \label{eq:n-true-reparam-grad-i}
    ,
\end{align}
where
$
g_i(r,\epsilon)
=
\sqrt{a_{i,2}^\top V a_{i,2} + b_{i,2}}\epsilon
+
a_{i,1}^\top m + b_{i,1}
.
$
Letting
$\phi_i(r,\epsilon):=p(y_i|f_i=g_i(r,\epsilon))$
,
$\phi_i'(r,\epsilon):=\frac{\partial}{\partial f_i}p(y_i|f_i=g_i(r,\epsilon))$,
and
$\phi_i''(r,\epsilon):=\frac{\partial^2}{\partial f_i^2}p(y_i|f_i=g_i(r,\epsilon))$,
the components of the gradient in \cref{eq:n-true-reparam-grad-i} are
\begin{align}
    \label{eq:true-reparam-grad-m-i}
    \nabla_m \objective_i(r)
    & =
    -
    \frac{
        \E_{\mathcal{N}(\epsilon|0,1)}
        \big[
            \phi_i'(r,\epsilon)
        \big]
    }{
        \E_{\mathcal{N}(\epsilon|0,1)}
        \big[
            \phi_i(r,\epsilon)
        \big]
    }
    a_{i,1}
    ,
\end{align}
\begin{align}
    \nabla_V \objective_i(r)
    & =
    -
    \frac{
        \E_{\mathcal{N}(\epsilon|0,1)}
        \big[
            \phi_i'(r,\epsilon)
            \epsilon
        \big]
    }{
        \E_{\mathcal{N}(\epsilon|0,1)}
        \big[
            \phi_i(r,\epsilon)
        \big]
    }
    \frac{a_{i,2} a_{i,2}^\top}{2\sqrt{a_{i,2}^\top V a_{i,2} + b_{i,2}}}
    \nonumber\\
    \label{eq:true-reparam-grad-V-i}
    &=
    -
    \frac{
        \E_{\mathcal{N}(\epsilon|0,1)}
        \big[
            \phi_i''(r,\epsilon)
        \big]
    }{
        \E_{\mathcal{N}(\epsilon|0,1)}
        \big[
            \phi_i(r,\epsilon)
        \big]
    }
    \frac{a_{i,2} a_{i,2}^\top}{2}
    ,
\end{align}
where the final equality holds under various conditions \citep{opper2008variational,rezende2014stochastic}.

We consider the bMC procedure that replaces the fraction in the true gradients of the loss term with 
$(\sum_{\draw=1}^\draws \nabla_r p(y_i|\fidraw)) / (\sum_{\draw=1}^\draws p(y_i|\fidraw))$
where
$\fidraw\sim q(f_i|r),1\le\draw\le\draws.$
The corresponding bMC estimates of the gradients are
\begin{align}
    \label{eq:n-step-m-i}
    d_{i,m}(r)
    & \defined
    \frac{
        \sum_{\draw=1}^\draws
        \phi_i'(r,\epsilon^{(\draw)})
    }{
        \sum_{\draw=1}^\draws
        \phi_i(r,\epsilon^{(\draw)})
    }
    a_{i,1}
    \\
    \label{eq:n-step-V-i}
    d_{i,V}(r)
    & 
    \defined
    \frac{
        \sum_{\draw=1}^\draws
        \phi_i''(r,\epsilon^{(\draw)})
    }{
        \sum_{\draw=1}^\draws
        \phi_i(r,\epsilon^{(\draw)})
    }
    \frac{a_{i,2} a_{i,2}^\top}{2}
\end{align}
where
$\mcsample$
are drawn i.i.d.\ from
$\mathcal{N}(\epsilon|0,1).$

Our proof uses
the following result from \cite{BertsekasTs1996} establishing conditions under which deterministic gradient descent with errors converges:
\begin{proposition}[Proposition 3.7 of \cite{BertsekasTs1996}]
    Let $r_t$ be a sequence generated by a gradient method
    $r_{t+1}=r_t+\gamma_t d_t$, where $d_t=(s_t+w_t)$ and $s_t$ and $w_t$  satisfy
    (i) $c_1\Vert \nabla\objective(r_t) \Vert^2 \le -\nabla \objective(r_t)^\top s_t$,
    (ii) $\Vert s_t\Vert \le c_2\Vert \nabla\objective(r_t)\Vert$,
    and
    (iii) $\Vert w_t\Vert \le \gamma_t (c_3 + c_4\Vert \nabla\objective(r_t)\Vert)$
    for some positive constants $c_1$,$c_2$, $c_3$, $c_4$.
    If $\nabla\objective()$ is Lipschitz and
    $\sum_{t=0}^\infty \gamma_t^2=0$ and $\sum_{t=0}^\infty \gamma_t=\infty$,
    then either $\objective(r_t)\rightarrow -\infty$ or else
    $\objective(r_t)$ converges to a finite value and
    $\lim_{t\rightarrow\infty}\nabla \objective(r_t)=0$.
    \label{prop:bertsekas}
\end{proposition}

The next condition is needed for the proof, and as shown by the following proposition it is easy to satisfy.

{\small 
\scriptsize 
\begin{table*}[t]
\caption{derivative bounds for different models}
\label{table:derivative-bounds}
\begin{center}
\begin{tabular}{l|ccccc}
    Likelihood & $B$ & $b'$ & $B'$ & $b''$ & $B''$ \\
    \hline
    Logistic, $\sigma(yf)$
        & $1$ & $-\tfrac{1}{4}$ & $\tfrac{1}{4}$ & $-\tfrac{1}{4}$ & $\tfrac{1}{4}$ \\
    \hline
    Gaussian, $e^{-(y-f)^2/2\sigma^2},c=\frac{1}{\sqrt{2\pi}\sigma}$
        & $c$ & $-{\tfrac{c}{\sqrt{e}\sigma} }$ & ${\tfrac{c}{\sqrt{e}\sigma} }$ & $-\tfrac{c}{\sigma^2}$ & $\tfrac{2c}{\sigma^2 e^{3/2}}$ \\
    \hline
    Probit, $\Phi(yf)$, $\Phi$ is cdf of Gaussian 
        & $1$ & $-1/\sqrt{2\pi}$ & $1/\sqrt{2\pi}$ & $-1/\sqrt{2\pi e}$ & $1/\sqrt{2\pi e}$ \\
    \hline
    Poisson, $\frac{g(f)^y e^{g(f)}}{y!}$,$g(f)=\log(e^f+1)$
        & $1$ & $-1$ & $1$ & $-2.25$ & $2.25$ \\
    \hline
    Poisson, $\frac{g(f)^y e^{g(f)}}{y!}$,$g(f)=e^f$
        & $1$ & $-y-1$ & $y$ & $-y-1/4$ & $2 y^2 + 3y + 2$ \\ 
    \hline
    Student's t, $c(1+\frac{(y-f)^2}{\sigma^2\nu})^{-\tfrac{\nu+1}{2}},c=\frac{\Gamma(\frac{\nu+1}{2})}{\Gamma(\frac{\nu}{2})\sqrt{\pi\nu}\sigma}$
    & $c$ & $-\frac{\frac{c}{\sigma}\frac{\nu+1}{\nu}\sqrt{\frac{\nu}{\nu+2}}}{(\frac{\nu+3}{\nu+2})^{(\nu+3)/2}}$ & $\frac{\frac{c}{\sigma}\frac{\nu+1}{\nu}\sqrt{\frac{\nu}{\nu+2}}}{(\frac{\nu+3}{\nu+2})^{(\nu+3)/2}}$ & $-\frac{c}{\sigma^2}\frac{\nu+1}{\nu}$ & 
    $2\frac{c}{\sigma^2}\frac{\nu+1}{\nu}(\frac{\nu+2}{\nu+5})^{(\nu+5)/2}$
\end{tabular}
\end{center}
\end{table*}
}

\begin{assumption}
\label{assumption:derivative-bound}
There exist finite constants $B,b',B',b'',B''$ such that
$B\ge \phi_i(r,\epsilon) \ge 0$,
$B'\ge \phi_i'(r,\epsilon) \ge b'$
and
$B''\ge \phi_i''(r,\epsilon) \ge b''$. Further, denote $B^*=\mbox{max}\{B,|B'|,|B''|,|b'|,|b''|\}$.
\end{assumption}

\begin{proposition}
Assumption \ref{assumption:derivative-bound} holds for the likelihood models  shown in Table \ref{table:derivative-bounds}.
\end{proposition}

The proof of the proposition is given in the supplement. 
We can now state the main result of this section:
\begin{corollary}
\label{corollary:convergence}
Suppose Assumption \ref{assumption:derivative-bound} holds. If for every $t$ and $i$, $\E_{q(f_i|r)}p(y_i|f_i)\ge\zeta>0$ and 

\begin{align}
\draws
    & >
    \frac{\log (6n/\delta_t)}{2\gamma_t^2} M \mbox{\ \ \ where } 
    \\
   M & = \max\Bigg\{
        \frac{
            B^2
        }{
            |\E_{\mathcal{N}(\epsilon|0,1)}\phi_i(r,\epsilon)|^2
        }
        ,%
        \frac{
            (B'-b')^2
        }{
            P^2
        }
        ,%
        \frac{
            (B''-b'')^2
        }{
            Q^2
        }
    \Bigg\},
\end{align}
$P=\begin{cases}|\E_{\mathcal{N}(\epsilon|0,1)}\phi_i'(r,\epsilon)|, &\text{if } \E_{\mathcal{N}(\epsilon|0,1)}\phi_i'(r,\epsilon)\neq 0 \\
1, & \text{otherwise} \end{cases}
$,
$Q=\begin{cases}|\E_{\mathcal{N}(\epsilon|0,1)}\phi_i''(r,\epsilon)|, &\text{if } \E_{\mathcal{N}(\epsilon|0,1)}\phi_i''(r,\epsilon)\neq 0 \\
1, & \text{otherwise} \end{cases} 
$,
and $\sum_t \delta_t = \delta$, then with probability at least $1-\delta$, bMC satisfies the conditions of the Proposition~\ref{prop:bertsekas} and hence converges.
\end{corollary}

To prove the claim we introduce the following lemma.
\begin{lemma}[Two-sided relative Hoeffding bound]
    \label{lemma:n-relative-bound}
    Consider i.i.d.\ draws $\{x^{(\draw)}\}$ from a random variable with mean $\mu\ne 0$ and support $[a,b]$.
    For $\delta,\alpha\in(0,1)$,
    if
    $
    \draws
    >
    \frac{1}{2}
    \frac{(b-a)^2}{(\alpha\mu)^2}
    \log\frac{2}{\delta},
    $
    then, w.p. at least $1-\delta$ over $\{x^{(\draw)}\}_{\draw=1}^\draws$,
    $(1/\draws)\sum_\draw x^{(\draw)}$ and ${\mu}$ have the same sign and 
    $
    0
    <
    1-\alpha
    \le
    \frac{(1/\draws)\sum_\draw x^{(\draw)}}{\mu}
    \le
    1+\alpha
    .
    $
\end{lemma}
\begin{proof}
First, assume $\mu>0$.
From 
Hoeffding's inequality,
we know that if the condition on $\draws$ is met,
then w.p. $\geq 1-\delta$, we have
$
\mu-\alpha\mu \le (1/\draws)\sum_\draw x^{(\draw)} \le \mu+\alpha\mu
$
from which the result follows.
If $\mu<0$, apply the same argument to the negation of the random variable.
\end{proof}

\begin{proof}[Proof of Corollary \ref{corollary:convergence}]
We first show that $\objective$ has Lipshitz gradients. 
This follows from a generalization of the mean-value theorem applied to continuous and differentiable vector-valued functions (see e.g., Theorem 5.19 of \cite{rudin1976principles}).
The Lipschitz constant will be equal to the maximum norm of the gradient over the domain and, in our case, will be finite when $\E_{q(f_i|r)}p(y_i|f_i)\ge\zeta>0.$
Note that it is always the case that the expectation is $>0$ but we must assume a uniform bound for all $t,i$.

Let $d_t=s_t+w_t$ where
$s_t=-\nabla \objective(r_t)$ so that conditions (i),(ii) hold trivially with $c_1=c_2=1$. 
We next develop the expression for $w_t$ to show that condition (iii) holds. 
Now $w_t=\sum_i w_{t,i}$ where $m$'s portion of $w_{t,i}$ is 
\begin{align}
w_{t,i,m} & = 
       \frac{a_{i,1}}{n} 
       \bigg(
        \frac{
            (1/\draws)
            \sum_\draw \phi_i'(r,\epsilon^{(\draw)})
        }{
            (1/\draws)
            \sum_\draw \phi_i(r,\epsilon^{(\draw)})
        } 
        -
        \frac{
            \E_{\mathcal{N}(\epsilon|0,1)}
                \phi_i'(r,\epsilon)
        }{
            \E_{\mathcal{N}(\epsilon|0,1)}
                \phi_i(r,\epsilon)
        }
       \bigg)
    \label{eq:n-wt-i-single}
\end{align}
and a similar expression holds for $V$'s portion.

Our claim follows from three conditions that hold with high probability.
When $E\phi'$ $\not=$ $0$ and $E\phi''\not=0$ 
the conditions require the averages
$
(1/\draws)\sum_\draw \phi_i(r,\epsilon^{(\draw)})$,
$(1/\draws)\sum_\draw \phi_i'(r,\epsilon^{(\draw)})$, 
$(1/\draws)\sum_\draw \phi_i''(r,\epsilon^{(\draw)})$
to be close to their expectations, i.e.\ within $1\pm \alpha$ relative error, 
w.p.\ $\ge 1-\delta/(3n)$.
Using Lemma~\ref{lemma:n-relative-bound}
these are accomplished by assuming that $\draws > \frac{\log (6n/\delta)}{2\alpha^2} M$.

From (\ref{eq:n-wt-i-single}) we have 
\begin{align}
&\Vert w_{t,i,m} \Vert^2 \nonumber \\
&\quad = 
       \frac{{\Vert a_{i,1}\Vert}^2}{n^2} 
       \bigg|
        \frac{
            (1/\draws)
            \sum_\draw \phi_i'(r,\epsilon^{(\draw)})
        }{
            (1/\draws)
            \sum_\draw \phi_i(r,\epsilon^{(\draw)})
        } 
        -
        \frac{
            \E_{\mathcal{N}(\epsilon|0,1)}
                \phi_i'(r,\epsilon)
        }{
            \E_{\mathcal{N}(\epsilon|0,1)}
                \phi_i(r,\epsilon)
        }
       \bigg|^2.
    \label{eq:n-wt-i-norm}
\end{align}
Considering the portion with absolute value,
if $\E_{\mathcal{N}(\epsilon|0,1)}\phi_i'(r,\epsilon)\not = 0$ then both fractions have the same sign.
Then since $\frac{1+\alpha}{1-\alpha}-1 > 1- \frac{1-\alpha}{1+\alpha}$
we have
\begin{align}
\Vert w_{t,i,m} \Vert^2 & \leq 
       \frac{{\Vert a_{i,1}\Vert}^2}{n^2} 
       \bigg(
        (\frac{
           (1+\alpha)
        }{
            (1-\alpha) 
        }  
        - 1)
        \frac{
            \E_{\mathcal{N}(\epsilon|0,1)}
                \phi_i'(r,\epsilon)
        }{
            \E_{\mathcal{N}(\epsilon|0,1)}
                \phi_i(r,\epsilon)
        }
        \bigg)^2 \nonumber\\
        &\le
        \frac{{\Vert a_{i,1}\Vert}^2}{n^2} 
        \bigg(
            \frac{
               2\alpha
            }{
                1-\alpha
            }  
        \bigg)^2
        \bigg(\frac{B^*}{\zeta}\bigg)^2
        ,
    \label{eq:n-wt-i-norm-b}
\end{align}
and using $\alpha\leq 0.5$ we get 
$\Vert w_{t,i,m} \Vert^2  \leq
        \big(
            \frac{4B^*\alpha}{\zeta}
            \frac{{\Vert a_{i,1}\Vert}}{n} 
        \big)^2
$.

When $\E_{\mathcal{N}(\epsilon|0,1)}\phi_i'(r,\epsilon)=0$,
we use the standard Hoeffding bound and  $L>\frac{(B'-b')^2\log (6 n/\delta)}{2\alpha^2}$
to guarantee that $|(1/\draws) \sum_\draw \phi_i'(r,\epsilon^{(\draw)}| \leq \alpha$ w.p.\ $\geq 1-\delta/(3 n)$. 
We also have 
$
(1/\draws)\sum_\draw\phi_i(r,\epsilon^{(\draw)})
\ge
\E_{\mathcal{N}(\epsilon|0,1)}\phi_i(r,\epsilon)(1-\alpha)
\ge 
\zeta(1-\alpha)$, 
and therefore,
for $m$'s portion we have
\begin{align}
\Vert w_{t,i,m} \Vert & \leq 
       \frac{{\Vert a_{i,1}\Vert}}{n} 
      \frac{
            \alpha
        }{
          \zeta(1-\alpha)  
        }
        \leq 
      \frac{{\Vert a_{i,1}\Vert}}{n} 
      \frac{
            2\alpha
        }{
          \zeta
      }
      ,
\end{align}
and we can bound $m$'s portion by the sum of bounds from the two cases:
\begin{align}
    \Vert w_{t,i,m} \Vert & \leq 
      \frac{{\Vert a_{i,1}\Vert}}{n} 
      \frac{
            2\alpha
        }{
          \zeta
      }
      (2B^*+1)
      .
\end{align}
Similar expressions for both cases hold simultaneously for $V$, replacing $a_{i,1}$ with $a_{i,2} a_{i,2}^\top$, and, therefore, combining bounds for $m,V$ we have
\begin{align}
\Vert w_{t,i} \Vert & \leq 
\frac{\Vert a_i\Vert}{n} \frac{2\alpha}{\zeta} (2B^*+1)
\end{align}
where $a_{i}$ is the concatenation of $a_{i,1}$ and the vectorization of $a_{i,2} a_{i,2}^\top$. %

Summing over all examples, we see that
\begin{align}
\Vert w_{t} \Vert  = \Vert \sum_i w_{t,i} \Vert  &\leq \sqrt{\sum_i \sum_j \Vert w_{t,i} \Vert\ \Vert w_{t,j} \Vert} \nonumber \\
      &\leq
A \frac{2\alpha}{\zeta} (2B^*+1)
    \label{eq:wt-i-norn}
\end{align}
where $A=\max_i \Vert a_{i} \Vert$. Using the union bound we see that this holds w.p.\ $\geq 1-\delta$.

To complete the analysis we need to make sure that the above holds for all iterations simultaneously. For this let 
$\delta_t$ be such that $\sum_t\delta_t=\delta$.
For example, $\delta_t = \frac{6}{\pi^2 }\frac{\delta}{t^2}$. Use $\delta_t$ in the definition of $L$ above to obtain the result. 

This satisfies condition (iii) if we set $\alpha$ for step $t$ to be $\alpha_t=\gamma_t$ and set $c_3=A \frac{2\sqrt{2}}{\zeta} (2B^*+1)$ and %
$c_4=0$.
\end{proof}
The implication of the choices of $\alpha_t$ and $\delta_t$ is that the number of samples $L$ increases with $t$. 
Specifically, for $\gamma_t=1/t$ this implies $L\propto t^2 \log (nt)$.
While this is a strong condition, we are not aware of any other analysis for a procedure like bMC. In practice, we use a fixed sample size $L$ in our experiments, and as shown there, the procedure is very effective. 

Notice that Corollary \ref{corollary:convergence} only guarantees convergence with high probability.
By adding a smoothing factor to the denominator of (\ref{eq:n-step-m-i}) and (\ref{eq:n-step-V-i}), we can 
strengthen the result and prove convergence w.p.\ 1. However, smoothing did not lead to a significant difference in results of our experiments. Details of the proof and experimental results are provided in the supplement.

%% file: related_work.tex
\section{Related Work}

DLM is not a new idea and it can be seen as regularized empirical risk minimization (ERM) which is a standard approach in the frequentist setting. 
An intriguing line of work in the frequentist setting follows \cite{McAllesterHK10} to develop DLM algorithms for non-differentiable losses. 
Extending the ideas in this paper to develop Bayesian DLM for non-differentiable losses is an important challenge for future work.

In the Bayesian context,
DLM can be seen as part of a larger theme which modifies the standard ELBO objective to change the loss term, change the regularization term, and allow for a regularization parameter, as captured by the GVI framework
\citep{knoblauch2019generalized,knoblauch2019robust} which is a view strongly connected to regularized loss minimization.
For example, the robustness literature, e.g., \cite{knoblauch2019generalized,cheriefabdellatif2019mmdbayes,Bissiri_2016,futami2018variational,knoblauch2019robust}, aims to optimize log loss but changes the training loss function in order to be robust to outliers or misspecification
and the safe-Bayesian approach of \cite{Grunwald2012,Grunwald2017} 
selects $\beta$ in order to handle misspecification. 
However, in all these papers the loss term is the {\em Gibbs loss}, 
$E_{q()}[\ell()]$, where $\ell()$  is the training loss.
In contrast, DLM uses the \emph{loss of the Bayesian predictor} with the motivation that this makes sense as an ERM algorithm.

A range of approaches have also been used from a theoretical perspective.
Some prior analysis of Bayesian  algorithms aims to show that the approximations recover exact inference under some conditions.
This includes, for example, consistency results for variational inference \citep{WangBlei2017, WangBlei2019} and the Laplace approximation \citep{dehaene2017computing}. For sparse GP, \cite{BurtRW19} shows that this holds when using the RBF kernel, and when the number and location of pseudo inputs are carefully selected. 
The work of 
\cite{Alquier2016properties} uses PAC Bayes theory and formulates conditions under which the variational approximation is close to the true posterior. 
In contrast to these,
\cite{Alquier2016properties} and \cite{Sheth2017excess,Sheth2019pseudo} analyze variational and DLM algorithms 
bounding their prediction loss relative to the ``best approximate pseudo posterior''.
Our paper further elaborates algorithmic details of DLM and provides an empirical evaluation.

Sparse GPs have received significant attention in the last few years. 
\cite{bauer2016understanding} investigates the performance of the variational and FITC approximations and provide many insights.
Their observations on difficulties in the optimization of hyperparameters in FITC might have parallels in DLM. Our experimental setup explicitly evaluates joint optimization of hyperparameters with DLM as well as a hybrid algorithm to address these difficulties. 
\cite{ReebDGR18} develops a new sGP algorithm by optimizing a PAC-Bayes bound. The output of their algorithm is chosen in a manner that provides better upper bound guarantees on its true error, but the actual test error is not improved over SVGP. 
The work of \cite{Cheng-NIPS-18} develops a novel variant of SVGP that uses different pseudo locations for $m$ and $V$. 
In contrast with these works our paper emphasizes the DLM objective and evaluates its potential to improve performance. 

Finally, several works have explored the idea of DLM for Bayesian algorithms.
\cite{Sheth2017excess} demonstrated the success of DLM in topic models.
The work of \cite{Sheth2016monte,jankowiak2019sparse,jankowiak2020deep}  applied log loss DLM and variants for regression showing competitive performance with ELBO. Our work significantly improves over this work by exhibiting the differences between square-loss DLM and log-loss DLM for regression, and by developing extensions, sampling methods and analysis for the non-conjugate case of log-loss DLM, which are stated as open questions  by \cite{jankowiak2020deep}.
Finally, 
\cite{masegosa2020learning} 
motivates DLM as the right procedure, but then identifies a novel alternative objective which is sandwiched between ELBO and DLM.
This offers an interesting 
alternative to DLM with the potential advantage that its loss term is 
is the Gibbs loss
(i.e., does not have log-expectation issues), but the disadvantage that it is an approximation to true DLM.
Overall, the space of loss terms, regularizers, and the balance between them offer a range of choices and identifying the best choice in any application is a complex problem. We believe that DLM is an important contribution in this space.

%% file: experiments.tex
\section{Experimental Evaluation}

Our experiments have two goals, the first is to evaluate whether DLM provides advantages over variational inference in practice, and the second is to explore the properties of the sampling methods, including efficiency, accuracy and stability. 
Due to space constraints, we summarize the main results here, and full details %
are provided in the supplement.

\subsection{Details of Algorithms and Experiments}

Preliminary experiments with joint optimization of variational parameters and hyperparameters in DLM showed that it is successful in many problems but that in some specific cases the optimization is not stable. We suspect that this is due to interaction between optimization of variational parameters and hyperparameters which complicates an experimental comparison. 
We therefore run two variants of DLM. The first performs joint optimization of variational parameters and hyperparameters. 
The second uses fixed hyperparameters, fixing them to the values learned by SVGP. This also allows us to compare the variational posterior of SVGP and DLM on the same hyperparameters.

Prior theoretical results do not have a clear recommendation for setting the regularization parameter $\beta$ where some analysis uses $\beta=0$ (no regularization), 
$\beta=1$ (the standard setting), and $\beta=\Theta(\sqrt{n})$.
Here we use grid search with a validation set on an exponentially-spaced grid, i.e., $\beta=[n,n/2,n/4,n/8,...,0.01]$.
In some experiments below we diverge from this and present results for specific values of $\beta$. 
To facilitate a fair comparison, we include ELBO with $\beta=1$ and a variant of ELBO that selects $\beta$ in exactly the same manner as DLM.

We selected 4 moderate size datasets for each of the likelihoods, giving 16 test cases including regression, square error, classification, and count prediction. In addition, we selected one large classification dataset that has been used before for evaluating sparse GP. 

All algorithms are trained with the Adam optimizer. 
 Isotropic RBF kernels are used except for the {\em airline} dataset where an ARD RBF kernel was used. 
 Evaluations are performed on held-out test data and 5 repetitions are used to generate error bars.
Full details of the experiments are given in the supplement.

\PutPlotSelection{fig:selection}

\subsection{Results}

Our first set of experiments aims to evaluate the merit of the DLM objective as compared to ELBO. 
To achieve this, we fix the number of pseudo points and then each point in Figure~\ref{fig:selection} (a-e), 
shows the final test set loss score {\em when the algorithm has converged on the corresponding sample size}. 
That is, we compare the quality that results from optimizing the objective, and not the optimization algorithm or convergence speed. This  allows a cleaner separation of the objectives.

{\bf Log-loss and sq-loss in sGP Regression:} Figure~\ref{fig:selection}(a) shows the result for log-loss regression on the {\em sarcos} dataset where log-loss DLM has a significant advantage. 
Figure~\ref{fig:selection}(b) shows the result for square-loss on the same dataset. Here we see that square-loss DLM has a significant advantage over other algorithms (including log-loss DLM). 
This illustrates the point made in the introduction, that optimizing DLM for a specific loss can have an advantage over methods that aim for a generic posterior. 
We can also observe that $\beta$-ELBO shows a clear improvement over ELBO, 
which suggests that selection of $\beta$ should be adopted more generally in variational inference. 
The supplement includes results for  3 additional datasets with similar trends. 

{\bf $\beta$-values:} It is interesting to consider the $\beta$ values selected by the algorithms. 
For most datasets and most training set sizes a small value of  $\beta<1$ is often a good choice. 
However, this is not always the case.
Figure~\ref{fig:selection}(c) shows
a plot of log-loss as a function of $\beta$ for a small (691) training set size on the {\em cadata} dataset.
We observe that the optimal $\beta$ is larger than 1 for all methods. 
For larger size data (see supplement) joint DLM selects $\beta<1$ but other methods do not. 

{\bf Log-loss DLM in non-conjugate sGP:} 
Figure~\ref{fig:selection} (d-e) show log-loss results for classification on the {\em ringnorm} dataset and for count regression on the {\em peds1} dataset, where 
DLM for count regression uses bMC sampling with 10 samples. 
We observe that log-loss DLM is comparable to or better than ELBO and $\beta$-ELBO. 
The supplement includes results for  3 additional datasets for each likelihood with similar trends. 
In some cases hyperparameter optimization in joint-DLM is sensitive, but taken together
the two DLM variants are either comparable to or significantly better than ELBO and $\beta$-ELBO. 
In addition, results from the same experiments which are included in the supplement show that DLM achieves better calibration in the non-conjugate cases without sacrificing classification error or count mean relative error.

{\bf Non-conjugate DLM on a large dataset:} 
We next consider whether DLM is applicable on large datasets and whether it still shows an advantage over ELBO. 
For this we use the {\em airline} dataset \citep{hensman2015scalable} which has been used before to evaluate sGP for classification.
Due to the size of the dataset we do not perform $\beta$ selection and instead present results for values 0.1, 1, and 10. 
In contrast with previous plots, 
Figure~\ref{fig:selection}(f) is a learning curve, showing log-loss {\em as a function of training epochs}.
We observe that 
{\em  for all values of $\beta$ in the experiment} both variants of 
$\beta$-DLM significantly improve over $\beta$-ELBO and  they significantly improve over ELBO ($\beta=1$). 

{\bf Evaluation of the sampling algorithms:} 
We first explore the quality of samples regardless of their effect on learning. 
Figure~\ref{fig:selection} portions (g,h) show estimates of bias for bMC and uPS on the {\em abalone} count prediction dataset (where the true gradient is estimated from 10000 bMC samples). 
The statistics for the gradients are collected immediately after the initialization of the algorithm. 
Additional plots in the supplement show estimates for the direction of the update step $d_t$ and its norm relative to the true gradient 
(similar to conditions (i) and (ii) of Proposition~\ref{prop:bertsekas} but for  $d_t$ and similar to conditions in Proposition 4.1 in \cite{BertsekasTs1996}).
The plots show that uPS indeed has lower bias as expected (note the scale in the plots).

We next compare the quality of predictions when learning using the sampling methods, to each other and to the results of exact computations.
Learning curves for {\em airline} for $\beta=0.1$ are shown in Figure~\ref{fig:selection}(i)
and plots for $\beta=1,10$ are given in the supplement. 
We observe that with enough samples both algorithms can recover the performance of the exact algorithm. 
We also observe in plot 1(i) that to achieve this uPS can use 10 samples and bMC needs 100 samples.
Similarly, uPS with 1 sample is is better than bMC with 10 samples. This suggests that uPS makes better use of samples and has a potential advantage. 
The supplement shows learning curves comparing uPS and bMC for count prediction on two datasets. 
In this case even one sample of bMC yields good results and there are no significant differences between bMC and uPS in terms of log-loss.
Finally, 
learning curves for log-loss in regression given in the supplement show that bMC can recover the results of exact gradients with $\geq 10$ samples. 
Overall, uPS is unbiased and might make more efficient use of samples.
However, despite the speedup developed for uPS, it is significantly slower in practice due to the cost of generating the samples, and bMC provides a better tradeoff in practice.

%% file: conclusion.tex
\section{Conclusion}

The paper explores the applicability and utility of DLM in sparse GP.
We make two technical contributions for sample based estimates of gradients of log-expectation terms: 
uPS provides unbiased samples and bMC is biased but is proved to lead to convergence nonetheless.
An extensive experimental evaluation shows that DLM for sparse GP is competitive and in some cases significantly better than the variational approach and that bMC provides a better time-accuracy tradeoff than uPS in practice.
While we have focused on sGP, DLM is at least in principle generally applicable. 
As mentioned above, this has already been demonstrated for the correlated topic model, where the hidden variable is not 1-dimensional, but where  equations simplify and gradients can be efficiently estimated through sampling. 
We believe that variants of the methods in this paper will enable applicability in 
probabilistic matrix factorization, GPLVM (through its reparameterized objective), and the variational auto-encoder and we leave these for future work. 
Extending the analysis of bMC 
to provide finite time bounds is another important direction for future work.

%% file: appendix.tex
\input{ps-implementation}
\input{convergence_expectation}

\input{derivatives}

\section{Complete experimental details}

\myparagraph{Training}
For regression, the algorithms are implemented in PyTorch. DLM is implemented as described in the main paper. 
Where simplified objectives are available, specifically regression ELBO for SVGP and regression objective for FITC, we implement the  
collapsed forms.
For classification and count prediction, 
we extend the implementation from GPyTorch \citep{GPyTorch}.
Isotropic RBF kernels are used unless otherwise specified. 
We use a zero mean function for experiments in regression and count prediction and a constant mean function for binary prediction (because some of the datasets require this to obtain reasonable performance with GP). 

All algorithms are trained with the Adam optimizer where we use a learning rate of $10^{-1}$ for batch data training and $10^{-3}$ for stochastic training.
The same stopping criteria consisting of either convergence or max iterations is used in all cases.
Almost all runs across algorithms and datasets resulted in convergence.
Convergence is defined when
the difference between the minimum and maximum of the loss in the last $I$ iterations does not exceed $10^{-4}$, for $I=50$ iterations in regression, and $I=20$ iterations in classification and count prediction.
For square loss DLM  the optimization for $m$ has a closed form, i.e., it is optimized in one step. 
If the log loss does not converge, we stop when the number of iterations exceeds 5000 for regression, and 3000 for classification and count regression.
Evaluations are performed on held-out test data and 5 repetitions are used to generate error bars.

\myparagraph{Datasets}
Table~\ref{tab:dataset} shows the datasets used and their characteristics.
In the table, ``dim'' refers to the number of features and $M$ is the number of inducing points used in our experiments. Notice that in some datasets, categorical features are converted to dummy coding, i.e., we use $L-1$ binary features to represent a feature with $L$ categories. One category is assigned the all zero code while the other $L-1$ categories are assigned to the unit vector with the corresponding entry set to 1.

\begin{table}[h]
    \centering
    \begin{tabular}{c|c|c|c|c}
        dataset & type & size &dim& $M$ \\
        \hline
        pol\footnotemark[1] & regression & 15000 & 26 & 100 \\
        cadata\footnotemark[2] & regression & 20640 & 8 & 206 \\
        sarcos\footnotemark[3] & regression & 48933 & 21 & 100 \\
        song\footnotemark[4] & regression & 515345 & 90 & 100 \\
        \hline
        banana\footnotemark[5] & classification & 5300 & 2 & 53 \\
        thyroid\footnotemark[4] & classification & 3772 & 6 & 37 \\
        twonorm\footnotemark[6] & classifcation & 7400 & 20 & 74 \\
        ringnorm\footnotemark[7] & classification & 7400 & 20 &74 \\
        airline\footnotemark[8] & classification & 2055733 & 8 & 200 \\
        \hline
        abalone\footnotemark[4] &count& 4177 & 9 & 41 \\
        Peds1\_dir0\footnotemark[9] &count& 4000 & 30 & 40 \\
        Peds1\_dir1\footnotemark[9] &count& 4000 & 30 & 40 \\
        Peds1\_dir2\footnotemark[9] &count& 4000 & 30 & 40 \\
        \hline
    \end{tabular}
    \caption{Details of datasets}
    \label{tab:dataset}
\end{table}
\footnotetext[1]{\url{https://github.com/trungngv/fgp/tree/master/data/pol}}
\footnotetext[2]{\url{https://www.csie.ntu.edu.tw/~cjlin/libsvmtools/datasets/regression/cadata}}
\footnotetext[3]{\citep{Rasmussen2006}}
\footnotetext[4]{\url{http://archive.ics.uci.edu/ml/index.php}}
\footnotetext[5]{\url{https://www.kaggle.com/saranchandar/standard-classification-banana-dataset}}
\footnotetext[6]{\url{https://www.cs.toronto.edu/~delve/data/twonorm/desc.html}}
\footnotetext[7]{\url{https://www.cs.toronto.edu/~delve/data/ringnorm/desc.html}}
\footnotetext[8]{\citep{hensman2015scalable}}
\footnotetext[9]{\url{http://visal.cs.cityu.edu.hk/downloads/}}

\myparagraph{Evaluation}
Each regression dataset is split into portions with relative sizes 67/8/25 for training, validation and testing.
For classification and count regression, we select a number of training sizes (up to 2000) and pick 10\% of all data to be the validation set.
From the remaining examples we randomly choose up to 1000 samples for testing (to reduce test time for the experiments). %
For the larger song dataset ($\approx 0.5$M samples in total), we randomly choose a subset of 10000 examples for test data in order to reduce the test time in experiments. 
To reduce run time for DLM on large datasets we use mini-batch training with batches of 6000 samples. 

For the $\approx 2$M-size airline dataset of \cite{hensman2015scalable}, we split a 100000 test set from the full dataset, and trained on the remaining data for 20 epochs with Adam and learning rate $10^{-3}$.
The number of inducing points was set to 200 and the mini-batch size was 1000.
Here, we used the RBF-ARD kernel.
For fixed-DLM the train/evaluation protocol is as follows: SVGP was trained with all hyperparameters and variational parameters being learned; then, DLM was initialized with the learned SVGP hyperparameters which were then fixed; the DLM variational parameters were learned from scratch.

In all cases, mean negative log likelihood (NLL) $-\log E_{q(f)}p(y|f)$ is calculated on the test set.
NLL is computed exactly for regression and classification. For count regression it is calculated using quadrature.
Additionally, we compute test set mean squared error (MSE) in regression, mean error in classification, and mean relative error (MRE) in count regression; the latter is defined as $\frac{|\hat{y} - y|}{\max(1, y)}$, $\hat{y}=E_{q(y)}[y]=E_{q(f)q(y|f)}[y]$.
$\hat{y}$ can be calculated analytically as $E_{q(y|f)}[y]=\lambda=e^f$ and $E_{q(f)}[e^f]$ is the MGF of the normal distribution.

All datasets are normalized with respect to training data and the same normalization is performed on validation and test data.

\myparagraph{Results}
Here, we include the complete experimental results stated in the main paper. 

{\bf Log-loss and sq-loss in sGP Regression and $\beta$ values:}
Figure~\ref{fig:reg-log} shows results for log loss in regression.
In 3 of the datasets joint-DLM is significantly better than other algorithms 
and in {\em cadata}, where hyperparameter selection is sensitive, fixed-DLM 
is significantly better than other algorithms.
Figure~\ref{fig:reg-log} also shows values selected for $\beta$ on small and large train sizes for the {\em cadata} dataset. 
As discussed in the main paper this illustrates that values of $\beta$ larger than 1 are needed in some cases. 
Figure~\ref{fig:reg-sq} shows results for square loss on the same datasets.
The pattern from log-loss is repeated here where joint-sq-DLM dominates in 3 of the datasets and fixed-sq-DLM dominates in {\em cadata}.

\PutRegressionLog{fig:reg-log}
\PutRegressionSq{fig:reg-sq}

{\bf Log-loss DLM in non-conjugate sGP:} 
Figure~\ref{fig:class-log} shows log loss in classification, and Figure~\ref{fig:class-err} shows the corresponding classification error 
in the same experiments. 
In this case except for ringnorm the differences are small and DLM variants are comparable to SVGP variants.

\PutClassificationLog{fig:class-log}
\PutClassificationErr{fig:class-err}

Figure~\ref{fig:poisson-log} shows log loss in count regression, and Figure~\ref{fig:poisson-mre} shows relative error
in the same experiments.
For log loss joint-DLM dominates in 3 of the datasets and fixed-DLM is equal or better in the 4th dataset. 
Figure~\ref{fig:poisson-mre} shows that 
the better calibrated prediction in terms of log loss is achieved while maintaining competitive MRE. 

\PutPoissonLoss{fig:poisson-log}
\PutPoissonMRE{fig:poisson-mre}

{\bf Non-conjugate DLM on a large dataset:} 
Figure~\ref{fig:airlineA} shows a comparison between SVGP and the two DLM variants on the airline dataset for three values of $\beta$.
As observed in the main paper, for this dataset, both DLM variants perform better than SVGP for all values of $\beta$ tested.

\PutAirlineA{fig:airlineA}

{\bf Evaluation of the sampling algorithms (bias statistics):}
Figure~\ref{fig:PoissonGradMean} shows statistics of the gradients for the mean variables using uPS, bMC and smooth-bMC.
The statistics for the gradients are collected immediately after the initialization of the algorithm. 
We show statistics for conditions similar to (i,ii) in Proposition 2 of the main paper as as as an estimate of the bias as compared to exact gradients. 
uPS is well behaved for all 3 measures. 
We observed that bMC with 1 sample is significantly more noisy. For the other cases the constant for condition (i) is roughly 1 (as would be with the true gradient) and the norm in condition (ii) is closer to the true gradient. The bias for bMC is significantly larger than uPS. 
smooth-bMC reduces the bias of bMC without negative effect on conditions (i,ii). 
However, as discussed below this does not lead to improvements in log loss.

\PutPoissonGradientsMean{fig:PoissonGradMean}

{\bf Evaluation of the sampling algorithms (learning comparison):}

Figures~\ref{fig:airlineB1}, \ref{fig:airlineB2}, and \ref{fig:airlineB3} compare learning with exact gradients to learning with bMC and uPC for $\beta=0.1, 1, 10$ respectively on the {\em airline} dataset. 
We observe that with enough samples both uPS and bMC recover the result of exact gradients, but uPS can do so with less samples. 
Figure~\ref{fig:airline-smooth-bMC} compares learning with exact gradients to learning with bMC and smooth-bMC when $\beta=0.1$. Smooth-bMC is very close to bMC when the number of samples are the same and thus, they overlap with each other. 
Hence in this case the potential improvement in bias resulting from smoothing does not lead to performance improvement in terms of log loss. 

\PutAirlineBA{fig:airlineB1}
\PutAirlineBB{fig:airlineB2}
\PutAirlineBC{fig:airlineB3}
\PutAirlineSmoothBMCLearningCurve{fig:airline-smooth-bMC}

Figure~\ref{fig:poisson-learning-curve} shows learning curves for count prediction on two datasets, comparing bMC and uPS sampling. 
Figure~\ref{fig:poisson-smooth-learning-curve} compares bMC and smooth-bMC.
In these experiments we have found uPS to be more sensitive and have reduced the learning rate for Adam from 0.1 to 0.01.
Here there are no significant differences between uPS, bMC and smooth bMC in terms of log loss. 

\PutPoissonLearningCurve{fig:poisson-learning-curve}
\PutPoissonSmoothBMCLearningCurve{fig:poisson-smooth-learning-curve}

Finally, 
Figure~\ref{fig:regression-learning} compares learning with exact gradients to learning with bMC on the 4 regression datasets. For all datasets, bMC-1 is the worst and there is almost no difference between exact and bMC-100.

\PutRegressionLearningCurve{fig:regression-learning}

In summary all the experiments suggest that with enough samples bMC results in competitive performance. uPS makes better use of samples in some cases. However, this comes with a significant cost in terms or run time. Hence bMC appears to be a better choice in practice.

%% file: ps-implementation.tex
\section{Efficient Implementation of Product Sampling}

\paragraph{Efficient Rejection Sampling:}
Recall that we want to sample from 
$\tilde{q}(f|\theta)=\frac{q(f|\theta)p(y|f)}{\E_{q(f|\theta)} p(y|f)}$ where the
normalizing constant $\E_{q(f|\theta)}p(y|f)$ is not known. Naive rejection sampling will have a high rejection rate and more advanced sampling techniques, such as adaptive rejection sampling, will be too slow because we need to sample the gradient for each example in each minibatch of optimization. 
We next show how to take advantage of the structure of $\tilde{q}(f)$ to construct an efficient sampler. 
Recall the standard setting for rejection sampling. To sample from an unnormalized distribution $h_1(f)$ we introduce $h_2(f)$ which is easy to sample from and such that $K h_2(f)\geq h_1(f)$. Then we sample $f^*\sim h_2(f)$, and accept $f^*$ with probability $h_1(f^*)/K h_2(f^*)$.

In our case $h_1$ is a product of a normal distribution $q(f)={\cal N}(\mu,\sigma^2)$ and a likelihood function $\ell(f)=p(y|f)$.
In the following we assume that $\ell(f)\leq \ell_{max}$ is bounded, which true for discrete $y$ and can be enforced by lower bounding the variance when $y$ is continuous. 
The main issue for sampling is the overlap between the ``high value regions'' of $q()$ and $\ell()$. If they are well aligned, 
for example, $\mbox{argmax}_{f\in \mu \pm \sigma} \ell(f)\geq 0.5$, then we can use $h_2(f)=q(f)$ with $K=1$ and the rejection rate will not be high. However, if they are not aligned then sampling from $q()$ will have a high rejection rate. To address this, we fix a small integer $n$ and sample from a broader distribution  with the same mean $h_2(f)={\cal N}(\mu,n \sigma^2)$. 

Let $a,b$ be the intersection points of the PDFs of $q()$ and $h_2()$ ($\mu\pm r$ for $r=\sigma\sqrt{{\log n}/(1-1/n)}$) and let 
$m_1=\mbox{max}_{f\in [a,b]} \ell(f)$ and $m_2=\mbox{min}_{f\in [a,b]} \frac{h_2(f)}{q(f)}=\frac{1}{\sqrt{n}}$.
Note that $\frac{m_1}{m_2}$ increases with $n$. 
To balance the sampling ratios within and outside $[a,b]$, 
we pick the largest $n\leq 10$ s.t.\ $m_1\leq m_2 \ell_{max}$ and use $K=\ell_{max}$.
Then in the interval $[a,b]$ we have $h_2(f) \ell_{max}\geq h_2(f) \frac{m_1}{m_2}\geq q(f)\ell(f)$ and outside the interval we have 
$h_2(f)\geq q(f)$ and therefore $h_2(f)\ell_{max} \geq q(f)\ell(f)$ as required.

The only likelihood specific step in the computation is the value of $m_1$. For the binary case with sigmoid or probit likelihood the maximum is obtained at one of the endpoints $p(a),p(b)$. For count regression with Poisson likelihood with link function $\lambda=e^f$, if the observation $\log y\in[a,b]$ then we also need to evaluate $p(y|\lambda=y)$.
The crucial point is that because of the structure of $q()$ and $h_2()$ the values of 
$m_1$,$m_2$ can be calculated analytically in constant time and the cost of determining $n$ is not prohibitive.

\paragraph{Vectorized sampling:} The process above yields efficient sampling, where after an initial set of learning iterations the average number of rejected samples is low (approximately 2 in our evaluation).
However, in practice the process is still slow. One of the reasons is the fact that we calculate $n$ which defines the sampling distribution separately for each example $i$ and then perform rejection sampling separately for each $i$. Modern implementations gain significant speedup by vectorizing operations, but this is at odds with individual rejection sampling. We partly alleviate this cost by a hybrid procedure as follows. 
Note that for each $i$ we have $h_2(f_i)={\cal N}(\mu_i,n_i\sigma_i^2)$ and that the samples for different $i$'s are independent. We can therefore collect these and sample from a multivariate normal with diagonal covariance. However, each such vector of samples will have some rejected entries. 
Our hybrid procedure repeats the vectorized sampling twice, uses the first successful sample for each $i$, and for entries which had no successful sample, resorts to individual sampling. We have found that this reduces overall run time by at least 50\%.

%% file: convergence_expectation.tex
\section{Convergence of smooth-bMC with Probability 1}

\newcommand{\R}{\mathcal{R}}
\newcommand{\F}{\mathcal{F}}

\begin{assumption}
\label{assumption}
We assume that there exists a function $f:\R^n \rightarrow \R$ with the following properties:
\begin{enumerate}[label=(\alph*)]
    \item There holds $f(r) \geq 0$ for all $r \in \R^n$.
    \item The function $f$ is continuously differentiable and there exists some constant $L$ such that 
    \begin{equation*}
        \Vert \nabla f(r) - \nabla f(\Bar{r}) \rVert \leq L \Vert r-\Bar{r} \rVert, \forall r, \Bar{r} \in \R^n.
    \end{equation*}
    \item There exists positive constant $c_1, c_2$ such that $\forall t$,
    \begin{align*}
        &c_1 \Vert \nabla f(r_t) \rVert^2 \leq -\nabla f(r_t)^T \E[s_t|\F_t], \\
        &\E[\Vert s_t \rVert^2] \leq c_2 \Vert \nabla f(r_t) \rVert^2.
    \end{align*}
    \item There exists positive constant $p, q$ such that
    \begin{equation*}
        \E[\Vert w_t \rVert^2] \leq (\gamma_t (q+p\Vert \nabla f(r_t) \rVert))^2.
    \end{equation*}
\end{enumerate}
\label{assumption:BT41modified}
\end{assumption}
Notice that condition (d) in Assumption \ref{assumption} implies $\E[\Vert w_t \rVert] \leq \gamma_t (q + p\Vert \nabla f(r_t) \rVert)$. This can be derived from Jensen's inequality where quadratic function is convex.

\begin{proposition}
Consider the algorithm
\begin{equation*}
    r_{t+1} = r_t + \gamma_t (s_t + w_t),
\end{equation*}
where the stepsizes $\gamma_t$ are nonnegative and satisfy
\begin{equation*}
    \sum_{t=0}^{\infty} \gamma_t = \infty, \sum_{t=0}^{\infty} \gamma_t^2 \leq \infty. 
\end{equation*}
Under Assumption \ref{assumption}, the following hold with probability 1:
\begin{enumerate}[label=(\alph*)]
    \item The sequence $f(r_t)$ converges.
    \item We have $\lim_{t\rightarrow \infty} \nabla f(r_t) = 0.$
    \item Every limit point of $r_t$ is a stationary point of $f$.
\end{enumerate}
\label{proposition:BT41modified}
\end{proposition}

The proposition and its proof are a slight modification of Proposition 4.1 by \cite{BertsekasTs1996}. Compared to that result, 
Assumption~\ref{assumption:BT41modified} splits the conditions on the step direction $g_t=s_t+w_t$ from \cite{BertsekasTs1996} into portion c on $s_t$ and portion d on $w_t$. This slight weakening of the condition enables our application in Theorem~\ref{thm:wp1}.

\begin{proof}
This proof slightly modifies the proof of Proposition 4.1 in \cite{BertsekasTs1996}. 
As shown there (in Eq 3.39), if $\nabla f()$ is $L$-Lipschitz then
$f(r+z)-f(r)\leq z^T \nabla f(r) + \frac{L}{2} \Vert z \rVert^2$ for two vectors $r, z$,
Then replacing $z$ with $\gamma_t (s_t + w_t)$ and taking expectation, we have 
\begin{align*}
    \E [f(r_{t+1})] &\leq f(r_t) + \gamma_t \nabla f(r_t)^T \E[s_t+w_t] + \frac{\gamma_t^2 L}{2} \E[\Vert s_t + w_t \rVert^2] \\
    &\leq f(r_t) + \gamma_t \nabla f(r_t)^T \E[s_t] + \gamma_t \nabla f(r_t)^T \E [w_t] + \gamma_t^2 L \E[\Vert s_t\rVert^2] + \gamma_t^2 L \E[\Vert w_t \rVert^2] \\
    &\leq f(r_t) + \gamma_t (-c_1 \Vert \nabla f(r_t)\rVert^2 + \Vert \nabla f(r_t)\rVert \E[\Vert w_t \rVert]) + \gamma_t^2 L (c_2^2 \Vert \nabla f(r_t) \rVert^2 \\
    &+ \gamma_t^2 q^2 + 2 \gamma_t^2 pq \Vert \nabla f(r_t)\rVert + \gamma_t^2 p^2 \Vert \nabla f(r_t) \rVert^2) \\
    &\leq f(r_t) - \gamma_t (c_1 - \gamma_t p - \gamma_t c_2^2 L  - \gamma_t^3 p^2 L) \Vert \nabla f(r_t) \rVert^2 + \gamma_t^2 (q+2\gamma_t^2 pq L) \Vert \nabla f(r_t) \rVert + \gamma_t^4 q^2 L.
\end{align*}
The second inequality uses $\Vert a+b\Vert^2 \leq 2\Vert a \Vert^2 + 2\Vert b \Vert^2$ for any two vector $a, b$. The third inequality uses the conditions in Assumption \ref{assumption}.
Let $c_t=c_1 - \gamma_t p - \gamma_t c_2^2 L - \gamma_t^3 p^2 L$, $d_t = q + 2 \gamma_t^2 pq L$. Then,
\begin{align*}
    \E[f(r_{t+1})] &\leq f(r_t) - \gamma_t c_t \Vert \nabla f(r_t) \rVert^2 + \gamma_t^2 d_t \Vert \nabla f(r_t) \rVert + \gamma_t^4 q^2 L \\
    &\leq f(r_t) -\gamma_t c_t \Vert \nabla f(r_t) \rVert^2 + \gamma_t^2 d_t (1+\Vert \nabla f(r_t) \rVert^2) + \gamma_t^4 q^2 L \\
    &=f(r_t) - \gamma_t (c_t - \gamma_t d_t) \Vert \nabla f(r_t) \rVert^2 + \gamma_t^2 d_t + \gamma_t^4 q^2 L \\
    &=f(r_t) - X_t + Z_t,
\end{align*}
where $X_t = \begin{cases} \gamma_t (c_t - \gamma_t d_t) \Vert \nabla f(r_t) \rVert^2, & \text{if } c_t\geq \gamma_t d_t, \\ 0, & \text{otherwise,} \end{cases}$, and \\$Z_t = \begin{cases} \gamma_t^2 d_t + \gamma_t^4 q^2 L, &\text{if } c_t \geq \gamma_t d_t, \\ \gamma_t^2 d_t + \gamma_t^4 q^2 L - \gamma_t (c_t - \gamma_t d_t) \Vert \nabla f(r_t) \rVert^2, &\text{otherwise.} \end{cases}$ 

Notice that $c_t - \gamma_t d_t$ is monotonically decreasing in $\gamma_t$ and $\lim_{t\rightarrow \infty} \gamma_t = 0$, so there exists some finite time after which $\gamma_t d_t \leq c_t$. It follows that after some finite time, we have $Z_t = \gamma_t^2 d_t + \gamma_t^4 q^2 L$ and therefore $\sum_{t=0}^\infty Z_t < \infty$. Applying Proposition 4.2 (Supermartingale Convergence Theorem) in \cite{BertsekasTs1996}, we can conclude that $f(r_t)$ converges 
which establishes part (a) of the proposition, and in addition that
$\sum_t X_t < \infty$. 

Similarly after some time, we have $c_t - \gamma_t d_t \geq \frac{c_1}{2}$ and
\begin{equation*}
    X_t = \gamma_t (c_t - \gamma_t d_t) \Vert \nabla f(r_t) \rVert^2 \geq \frac{c_1}{2} \gamma_t \Vert \nabla f(r_t) \rVert^2.
\end{equation*}
Hence,
\begin{equation*}
    \sum_{t=0}^{\infty} \gamma_t \Vert \nabla f(r_t) \rVert^2 < \infty.
\end{equation*}

Below we prove that $\Vert \nabla f(r_t) \rVert$ converges to $0$. Let $g_t = s_t + w_t$,
\begin{align*}
    \E[\Vert g_t \rVert^2] &= \E[\Vert s_t + w_t \rVert^2] \\
    &\leq \E[2 \Vert s_t \rVert^2 + 2 \Vert w_t \rVert^2] \\
    &\leq 2 c_2 \Vert \nabla f(r_t) \rVert^2 + 2 \gamma_t^2 (q + p \Vert \nabla f(r_t) \rVert)^2 \\
    &=2 c_2 \Vert \nabla f(r_t) \rVert^2 + 2 \gamma_t^2 (q^2 + 2 pq \Vert \nabla f(r_t) \rVert + p^2 \Vert \nabla f(r_t) \rVert^2 ) \\
    &\leq 2 (c_2 + p^2 \gamma_t^2 + 2pq \gamma_t^2) \Vert \nabla f(r_t) \rVert^2 + 2 \gamma_t^2 q^2 + 4pq \gamma_t^2.
\end{align*}
Suppose $\max \gamma_t \leq \gamma$, let $K_1 = 2(c_2 + p^2 \gamma^2 + 2pq \gamma^2)$ and $K_2 = 2 \gamma^2 q^2 + 4pq \gamma^2$, we have $\E[\Vert g_t \rVert^2] \leq K_1 \Vert \nabla f(r_t) \rVert^2 + K_2 $. 
Then all the remaining steps in the proof in \citep{BertsekasTs1996} for claims (b),(c) can be followed by replacing $s_t$ in \cite{BertsekasTs1996} with our $g_t$. 
\end{proof}

In order to establish convergence w.p.\ 1 we need to bound the norm of the step direction. To achieve this we add a smoothing parameter $\nu$ to the denominator of the estimate. This yields the  \textit{smooth-bMC} algorithm, whose step directions are
\begin{align*}
d_{i,m}(r)
    & \defined
    \frac{
        \sum_{\draw=1}^\draws
        \phi_i'(r,\epsilon^{(\draw)})
    }{
        \sum_{\draw=1}^\draws
        \phi_i(r,\epsilon^{(\draw)}) + \nu
    }
    a_{i,1}
    \\
    d_{i,V}(r)
    & 
    \defined
    \frac{
        \sum_{\draw=1}^\draws
        \phi_i''(r,\epsilon^{(\draw)})
    }{
        \sum_{\draw=1}^\draws
        \phi_i(r,\epsilon^{(\draw)}) + \nu
    }
    \frac{a_{i,2} a_{i,2}^\top}{2}
\end{align*}
and thus
\begin{align*}
w_{t,i,m} & = 
       \frac{a_{i,1}}{n} 
       \bigg(
        \frac{
            (1/\draws)
            \sum_\draw \phi_i'(r_t,\epsilon^{(\draw)})
        }{
            (1/\draws)
            \sum_\draw \phi_i(r_t,\epsilon^{(\draw)}) + \nu_t
        } 
        -
        \frac{
            \E_{\mathcal{N}(\epsilon|0,1)}
                \phi_i'(r_t,\epsilon)
        }{
            \E_{\mathcal{N}(\epsilon|0,1)}
                \phi_i(r_t,\epsilon)
        }
       \bigg)
\end{align*}
and a similar expression holds for V's portion. 
We now have:

\begin{corollary}
\label{thm:wp1}
If for every $t$ and $i$, $\E_{q(f_i|r)}p(y_i|f_i)\ge\zeta>0$,  smooth-bMC uses $\nu_t=\gamma_t \zeta$, and 
\begin{align}
\draws
    >
    \frac{\log (6n/\delta_t)}{2\gamma_t^2}
    \max\Bigg\{
        &\frac{
            B^2
        }{
            |\E_{\mathcal{N}(\epsilon|0,1)}\phi_i(r,\epsilon)|^2
        }
        ,\nonumber\\
        &\frac{
            (B'-b')^2
        }{
            P^2
        }
        ,\nonumber\\
        &\frac{
            (B''-b'')^2
        }{
            Q^2
        }
    \Bigg\},
\end{align}
where 
$P=\begin{cases}|\E_{\mathcal{N}(\epsilon|0,1)}\phi_i'(r,\epsilon)|, &\text{if } \E_{\mathcal{N}(\epsilon|0,1)}\phi_i'(r,\epsilon)\neq 0 \\
1, & \text{otherwise} \end{cases}
$,
$Q=\begin{cases}|\E_{\mathcal{N}(\epsilon|0,1)}\phi_i''(r,\epsilon)|, &\text{if } \E_{\mathcal{N}(\epsilon|0,1)}\phi_i''(r,\epsilon)\neq 0 \\
1, & \text{otherwise} \end{cases} 
$,
and $\delta_t = \gamma_t^4$, then smooth-bMC satisfies the conditions of 
Proposition~\ref{proposition:BT41modified}
and hence converges w.p.\ 1. 
\end{corollary}

\begin{proof}
Conditions (a,b,c) of Assumption \ref{assumption} are handled exactly as in the main paper.
Thus we only need to show that (d) holds.

First consider the case when $\E[\phi']\neq 0$ and $\E[\phi'']\neq 0$. With
\begin{align*}
    L \geq \frac{\log (6 n/ \delta_t)}{2 \alpha^2} \max\left\{\frac{B^2}{\E[\phi_i(r, \epsilon)]^2}, \frac{(B'-b')^2}{\left|\E[\phi_i' (r, \epsilon)] \right|^2}, \frac{(B''-b'')^2}{\left|\E[\phi_i'' (r, \epsilon)] \right|^2} \right\}
\end{align*}
we have that $\forall i, t$, 
\begin{align*}
    (1-\alpha) |\E[\phi_i]| \leq \left|(1/L) \sum_{l=1}^L \phi_i (r, \epsilon^{(l)})\right| \leq (1+\alpha) |\E[\phi_i]|,\\
    (1-\alpha) |\E[\phi_i']| \leq \left|(1/L) \sum_{l=1}^L \phi_i' (r, \epsilon^{(l)})\right| \leq (1+\alpha) |\E[\phi_i']|, \\
    (1-\alpha) |\E[\phi_i'']| \leq \left|(1/L) \sum_{l=1}^L \phi_i'' (r, \epsilon^{(l)})\right| \leq (1+\alpha) |\E[\phi_i'']|.
\end{align*}
hold simultaneously w.p. $\geq 1-\frac{\delta_t}{n}$.

Since $\phi_i >0$, we have $(1/L) \sum_{l=1}^L \phi_i (r, \epsilon^{(l)}) + \nu \geq (1-\alpha) \E[\phi_i] +\nu \geq (1-\alpha) \E [\phi_i]$. 
In addition 
\begin{align*}
    (1/L) \sum_{l=1}^L \phi_i (r, \epsilon^{(l)}) + \nu \leq (1+\alpha) \E[\phi_i] + \nu \leq (1+\alpha + \frac{\nu}{\zeta})\E[\phi_i].
\end{align*}
Then 
\begin{align*}
    \Vert w_{t, i, m} \rVert^2 &\leq \frac{\Vert a_{i, 1} \rVert^2}{n^2} \max \left(\frac{1+\alpha}{1-\alpha} -1, 1-\frac{1-\alpha}{1+\alpha + \frac{\nu}{\zeta}}\right)^2 \left(\frac{\E[\phi_i'(r_t, \epsilon)]}{\E[\phi_i (r_t, \epsilon)]}\right)^2 \\
    &= \frac{\Vert a_{i, 1} \rVert^2}{n^2} \max \left(\frac{2\alpha}{1-\alpha}, \frac{2\alpha + \frac{\nu}{\zeta}}{1+\alpha + \frac{\nu}{\zeta}}\right)^2 \left(\frac{\E[\phi_i'(r_t, \epsilon)]}{\E[\phi_i (r_t, \epsilon)]}\right)^2.
\end{align*}

Let $\nu = \alpha \zeta$, then $\frac{\nu}{\zeta} = \alpha$. Thus
\begin{align*}
    \max \left(\frac{2\alpha}{1-\alpha}, \frac{2\alpha + \frac{\nu}{\zeta}}{1+\alpha + \frac{\nu}{\zeta}}\right) &= \max \left(\frac{2\alpha}{1-\alpha}, \frac{3\alpha}{1+2\alpha}\right) \leq \frac{3\alpha}{1-\alpha}.
\end{align*}
Using $\alpha \leq 0.5$ we get $\Vert w_{t, i, m} \rVert^2 \leq \left(\frac{6B^* \alpha}{\zeta} \frac{\Vert a_{i, 1} \rVert}{n} \right)^2$.

Next consider the case when $\E[\phi']=0$. In this case we can select $L\geq \frac{(B'-b')^2 \log (6n/\delta_t)}{2\alpha^2}$ such that $|(1/L) \sum_l \phi'(r, \epsilon^{(l)})| \leq \alpha$ w.p.$\geq 1-\delta_t/(3n)$. At the same time, $(1/L) \sum_l \phi_i(r, \epsilon^{(l)}) + \nu \geq (1-\alpha)\E[\phi_i(r,\epsilon)]$. Thus,
\begin{align*}
    \Vert w_{t, i, m} \rVert^2 \leq \left(\frac{\Vert a_{i, 1} \rVert}{n} \frac{\alpha}{\zeta(1-\alpha)} \right)^2 \leq \left(\frac{2\alpha}{\zeta} \frac{\Vert a_{i, 1}\rVert}{n} \right)^2.
\end{align*}
Combined two cases, we have
\begin{align*}
\Vert w_{t,i,m} \Vert^2 \leq \left( \frac{2\alpha} {\zeta} \frac{\Vert a_{i,1} \Vert}{n} (3B^* +1) \right)^2.
\end{align*}

$w_{t,i, V}$ can be bounded similarly. Thus, $\Vert w_{t,i} \rVert^2$ can be upper bounded with high probability. Overall, w.p. at least $1-\delta_t$, 
\[
\Vert w_t \rVert^2 \leq \sum_i \sum_j \Vert w_{t,i} \Vert \Vert w_{t,j} \Vert \leq \left(A \frac{2\alpha}{\zeta} (3B^* + 1)\right)^2,
\]
where $A=\max_i \Vert a_i \Vert$.

However, w.p. at most $\delta_t$, the above inequality does not hold.
In order to bound the expectation we use the following upper bound which always holds:
\begin{align*}
    \Vert w_{t,i,m} \rVert^2 &= \frac{\Vert a_{i, 1} \rVert^2}{n^2} \left| \frac{(1/L) \sum_l \phi_i'(r, \epsilon^{(l)})}{\nu + (1/L) \sum_l \phi_i (r, \epsilon)} - \frac{\E[\phi_i'(r, \epsilon)]}{\E[\phi_i(r,\epsilon)]} \right|^2 \\
    &\leq \frac{\Vert a_{i, 1}\rVert^2}{n^2} \left(2\left( \frac{(1/L) \sum_l \phi_i'(r, \epsilon^{(l)})}{\nu + (1/L) \sum_l \phi_i (r, \epsilon^{(l)})} \right)^2 + 2 \left( \frac{\E[\phi_i'(r, \epsilon)]}{\E[\phi_i (r, \epsilon)]} \right)^2 \right) \\
    &\leq 2\frac{\Vert a_{i, 1}\rVert^2}{n^2} \max\{|b'|, |B'|\}^2 \left( \frac{1}{\nu^2} + \frac{1}{\zeta^2} \right) \\
    &\leq 2 \frac{\Vert a_{i, 1}\rVert^2}{n^2} (B^*)^2 \left( \frac{1}{\nu^2} + \frac{1}{\zeta^2} \right) \\
    &= 2 \frac{\Vert a_{i, 1}\rVert^2}{n^2} (B^*)^2 \left(\frac{1}{\alpha^2} + 1 \right)\frac{1}{\zeta^2} \\
    &\leq 4 \frac{\Vert a_{i, 1}\rVert^2}{n^2} (B^*)^2 \frac{1}{\alpha^2 \zeta^2}.
\end{align*}
In the third step, we used the fact that $\phi_i'$ is bounded between $b'$ and $B'$. In the last step, we use the fact that $1\leq \frac{1}{\alpha^2}$. Similar arguments can be derived for $w_{t, i, V}$. Thus $\Vert w_{t, i} \rVert^2 \leq 4\frac{\Vert a_i \rVert^2}{n^2} (B^*)^2 \frac{1}{\alpha^2 \zeta^2}$. Further, 
\begin{align*}
    \Vert w_t \rVert^2 = \Vert \sum_i w_{t, i} \rVert^2 \leq \sum_i \sum_j \Vert w_{t, i} \rVert \Vert w_{t,j} \rVert \leq 4 A^2 (B^*)^2 \frac{1}{\alpha^2 \zeta^2} = \frac{D}{\alpha^2}.
\end{align*}
where $D=\left(A\frac{2}{\zeta} B^* \right)^2$ is a constant.

Thus, $\E[\Vert w \rVert^2]$ can be bounded,
\begin{align*}
    \E [\Vert w_{t} \rVert^2] &\leq (1-\delta_t) (A \frac{2 \alpha}{\zeta} (3B^* +1))^2 + \delta_t D \leq (A \frac{2 \alpha}{\zeta} (3 B^* +1))^2 + \delta_t \frac{D}{\alpha^2}.
\end{align*}
Here let $\alpha=\gamma_t$ (hence $\nu_t=\gamma_t \zeta$) and $\delta_t = \gamma_t^4$, then condition (d) of Assumption \ref{assumption} holds with $p=0$ and $q^2=(A \frac{2}{\zeta} (3B^* +1))^2 + \left(A \frac{2}{\zeta} B^* \right)^2.$ 
\end{proof}

Suppose $\gamma_t = \frac{1}{t}$, then the sample size $L\propto t^2 \log (nt)$, which matches the sample size in the main paper. 
As in the discussion there, in practice we use a fixed sample size $L$ and we use a fixed smoothing factor $\nu$.

%% file: derivatives.tex
\section{Proof of Proposition 4 from Main Paper: Bounds on $\phi, \phi', \phi''$}

In this section we show that bounds of the form $\phi<B$, $b'<\Phi'<B'$, and $b''<\Phi''<B''$ holds in each cases as listed in the following table:

\smallskip

{\small 
\scriptsize 
\begin{tabular}{l|ccccc}
    Likelihood & $B$ & $b'$ & $B'$ & $b''$ & $B''$ \\
    \hline
    Logistic, $\sigma(yf)$
        & $1$ & $-\tfrac{1}{4}$ & $\tfrac{1}{4}$ & $-\tfrac{1}{4}$ & $\tfrac{1}{4}$ \\
    \hline
    Gaussian, $e^{-(y-f)^2/2\sigma^2},c=\frac{1}{\sqrt{2\pi}\sigma}$
        & $c$ & $-{\tfrac{c}{\sqrt{e}\sigma} }$ & ${\tfrac{c}{\sqrt{e}\sigma} }$ & $-\tfrac{c}{\sigma^2}$ & $\tfrac{2c}{\sigma^2 e^{3/2}}$ \\
    \hline
    Probit, $\Phi(yf)$, $\Phi$ is cdf of Gaussian 
        & $1$ & $-1/\sqrt{2\pi}$ & $1/\sqrt{2\pi}$ & $-1/\sqrt{2\pi e}$ & $1/\sqrt{2\pi e}$ \\
    \hline
    Poisson, $\frac{g(f)^y e^{g(f)}}{y!}$,$g(f)=\log(e^f+1)$
        & $1$ & $-1$ & $1$ & $-2.25$ & $2.25$ \\
    \hline
    Poisson, $\frac{g(f)^y e^{g(f)}}{y!}$,$g(f)=e^f$
        & $1$ & $-y-1$ & $y$ & $-y-1/4$ & $2 y^2 + 3y + 2$ \\ 
    \hline
    Student's t, $c(1+\frac{(y-f)^2}{\sigma^2\nu})^{-\tfrac{\nu+1}{2}},c=\frac{\Gamma(\frac{\nu+1}{2})}{\Gamma(\frac{\nu}{2})\sqrt{\pi\nu}\sigma}$
    & $c$ & $-\frac{\frac{c}{\sigma}\frac{\nu+1}{\nu}\sqrt{\frac{\nu}{\nu+2}}}{(\frac{\nu+3}{\nu+2})^{(\nu+3)/2}}$ & $\frac{\frac{c}{\sigma}\frac{\nu+1}{\nu}\sqrt{\frac{\nu}{\nu+2}}}{(\frac{\nu+3}{\nu+2})^{(\nu+3)/2}}$ & $-\frac{c}{\sigma^2}\frac{\nu+1}{\nu}$ & 
    $2\frac{c}{\sigma^2}\frac{\nu+1}{\nu}(\frac{\nu+2}{\nu+5})^{(\nu+5)/2}$
\end{tabular}
}

{\bf logistic:} For convenience let the label $y_i$ be in $\{-1,1\}$. Then 
$\phi=\sigma(y_i f_i)\leq 1$, $\phi' = y_i \sigma(y_i f_i) (1-\sigma(y_i f_i)) \in [-0.25,0.25]$, 
and 
$\phi'' = y_i^2 \sigma(y_i f_i) (1-\sigma(y_i f_i)) [1- 2 \sigma(y_i f_i)] \in [-0.25,0.25]$.

{\bf Gaussian:}
The Gaussian likelihood is
$\phi=p(y|f,\sigma^2)=c\exp(-(y-f)^2/2\sigma^2),c=\frac{1}{\sqrt{2\pi}\sigma}$.
In the following, let $x=(y-f)/\sigma$ to reduce clutter.
The first derivative of $\phi$ w.r.t.\ $f$ is given by
$\phi'=c\exp(-x^2/2)x\frac{1}{\sigma}$ whose root at $f=y~(x=0)$ corresponds to the maximum of the likelihood $c$.
The second derivative is given by
$\phi''=\frac{c}{\sigma^2}\exp(-x^2/2)(x^2-1)$.
The first derivative evaluations at the second derivative roots defined by $f=y\pm\sigma~(x=\pm 1)$ are $\pm\frac{c}{\sigma}\sqrt{e}$.
The third derivative is given by
$\phi'''=\frac{c}{\sigma^3}\exp(-x^2/2)x(x^2-3)$.
The second derivative evaluations at the third derivative roots defined by $f=y~(x=0)$ and $f=y\pm\sigma\sqrt{3}~(x^2=3)$ are $-\frac{c}{\sigma^2}$ and
$\frac{2c}{\sigma^2}\exp(-3/2)$,
respectively.
Finally, the first and second derivatives clearly approach 0 as $f$ approaches $\pm\infty$ since $\exp(f^2/2)$ dominates the growth of any polynomial in $f$.

{\bf probit:}
For convenience let the label $y_i$ be in $\{-1,1\}$. Then 
$\phi=\Phi(y_i f_i)$, where $\Phi$ is the CDF of the standard normal and clearly $\phi\in [0,1]$.
Let $h()$ be the PDF of the standard normal.
Then $\phi'=y_i h(y_i f_i)$ and $\phi'=y_i^2 h'(y_i f_i)$. Bounds for these are given by bounds above for the normal distribution with $\mu=0$ and $\sigma^2=1$, where we have to account for the sign flip in $\phi'=y_i h()$.

{\bf Poisson:} Here we also need to consider the link function. Two standard options are $\lambda=g(f_i)=e^{f_i}$ or $\lambda=g(f_i)=\ln(e^{f_i}+1)$. 

For the first option we have $\lambda = g(f_i) = e^{f_i}$. As above it is obvious that $\phi \leq 1$. 
$\phi' = \frac{\lambda^y e^{-\lambda}}{y!} (y-\lambda)=y\phi -(y+1) \frac{\lambda^{y+1} e^{-\lambda}}{(y+1)!}.$ Thus, $\phi' \geq b' = -y-1 $ and $\phi' \leq B' = y $. 
$\phi'' = \frac{\lambda^y e^{-\lambda}}{y!}(y^2 - (2y+1)\lambda + \lambda^2) = \frac{\lambda^y e^{-\lambda}}{y!}((\lambda - (y+\frac{1}{2}))^2 - y -\frac{1}{4}) \geq -y - \frac{1}{4}$. 
On the other hand, $\phi'' = (\frac{\lambda^y e^{-y}}{y!}) y^2 - \frac{\lambda^{y+1} e^{-\lambda}}{(y+1)!} (2y+1)(y+1) + \frac{\lambda^{y+2} e^{-\lambda}}{(y+2)!} (y+1)(y+2) \leq 2 y^2 + 3y + 2$.

For the second option we have $\lambda = g(f_i)=\ln(e^{f_i}+1)$.
Below we use $\phi_{t} (\lambda_i)$ to denote $p(t|f_i)=\frac{\lambda_i^t e^{\lambda_i}}{t!}$.
First note that $g'(f_i) =\frac{e^{f_i}}{e^{f_i}+1}=\sigma(f_i)\in [0,1]$ and $g''(f_i)\in [0,0.25]$ .
We have 
$\phi=\frac{\lambda_i^{y_i} e^{\lambda_i}}{y_i!}\leq 1$,
$\phi' =g'(f_i)\phi'(\lambda_i) = g'(f_i) (\phi(\lambda_i) - \phi_{y_i-1}(\lambda_i))$ implying $-1<\phi'<1$,
and 
\begin{align*}
\phi'' &= g''(f_i) (\phi(\lambda_i) - \phi_{y_i-1}(\lambda_i)) + (g'(f_i))^2 (\phi_{y_i-1}'(\lambda_i)-\phi(\lambda_i)') \\
&= g''(f_i) (\phi(\lambda_i)-\phi_{y_i-1}(\lambda_i)) + (g'(f_i))^2 (\phi_{y_i-2}(\lambda_i)-\phi_{y_i-1}(\lambda_i)-\phi_{y_i-1}(\lambda_i)+\phi_{y_i}(\lambda_i))
\end{align*}
which is bounded because all its components are bounded.
Plugging in $0\leq g'\leq 1$, $0\leq g''\leq 0.25$, we get that $\phi''>-2.25$ and $\phi''<2.25$.

{\bf Student T:}
The student's t likelihood is
$p(y|f,\nu,\sigma^2)=c\Big(1+\frac{(y-f)^2}{\sigma^2\nu}\Big)^{-(\nu+1)/2},c=\frac{\Gamma(\frac{\nu+1}{2})}{\Gamma(\frac{\nu}{2})\sqrt{\pi\nu}\sigma},\nu\in\mathbb{R}^+$.
In the following, let $x=(y-f)/\sigma$ to reduce clutter.
The first derivative is given by
$-c\frac{\nu+1}{\nu}\frac{x}{(1+\frac{x^2}{\nu})^{(\nu+3)/2}}(-\frac{1}{\sigma})$.
The only first derivative root at $f=y~(x=0)$ corresponds to the maximum of the likelihood $c$.
The second derivative is given by
\begin{multline*}
    -c\frac{\nu+1}{\nu}
    [
        \frac{1}{(1+\frac{x^2}{\nu})^{(\nu+3)/2}}
        -
        \frac{x^2\frac{\nu+3}{\nu}}{(1+\frac{x^2}{\nu})^{(\nu+5)/2}}
    ]
    \frac{1}{\sigma^2}
    \\
    =
    -c\frac{\nu+1}{\nu}
    (1+\frac{x^2}{\nu})^{-(\nu+5)/2}
    [ 1 + \frac{x^2}{\nu} - x^2\frac{\nu+3}{\nu} ]
    \frac{1}{\sigma^2}
    \\
    =
    -c\frac{\nu+1}{\nu}
    \frac{1 - x^2\frac{\nu+2}{\nu}}{(1+\frac{x^2}{\nu})^{(\nu+5)/2}}
    \frac{1}{\sigma^2}.
\end{multline*}
The first derivative evaluations at the second derivative roots defined by
$f=y\pm\sigma\sqrt{\frac{\nu}{\nu+2}}~(x=\pm\sqrt{\frac{\nu}{\nu+2}})$ are
$\pm \frac{c}{\sigma}\frac{\nu+1}{\nu}\sqrt{\frac{\nu}{\nu+2}}/(\frac{\nu+3}{\nu+2})^{(\nu+3)/2}$.
Also, since $\nu>0$, the denominator of the first derivative is a polynomial of degree at least 3 implying that as $f$ approaches $\pm\infty$, the first derivative approaches 0.
Hence, the first derivative is bounded over its domain.
The third derivative is given by 
\begin{multline*}
    -c\frac{\nu+1}{\nu}
    [
        -\frac{\nu+5}{\nu}
        \frac{x}{(1+\frac{x^2}{\nu})^{(\nu+7)/2}}
        [1-x^2\frac{2+\nu}{\nu}]
        +
        \frac{1}{(1+\frac{x^2}{\nu})^{(\nu+5)/2}}
        (-2x)\frac{2+\nu}{\nu}
    ]
    (\frac{-1}{\sigma^3})
    \\ =
    c\frac{\nu+1}{\nu^2}
    (1+\frac{x^2}{\nu})^{-(\nu+7)/2}
    x
    [
        ({\nu+5})
        [1-x^2\frac{2+\nu}{\nu}]
        +
        (1+\frac{x^2}{\nu})
        2({2+\nu})
    ]
    (\frac{-1}{\sigma^3})
    \\ =
    c\frac{\nu+1}{\nu^2}
    (1+\frac{x^2}{\nu})^{-(\nu+7)/2}x
    [\nu+5-x^2\frac{(2+\nu)(5+\nu)}{\nu}
    +2(2+\nu)+\frac{x^2}{\nu}{2(2+\nu)}]
    (\frac{-1}{\sigma^3})
    \\ =
    c\frac{\nu+1}{\nu^2}
    (1+\frac{x^2}{\nu})^{-(\nu+7)/2}x
    [-\frac{(2+\nu)(3+\nu)}{\nu}x^2+3(3+\nu)]
    (\frac{-1}{\sigma^3})
    \\ =
    c\frac{(\nu+1)(\nu+3)}{\nu^2}
    (1+\frac{x^2}{\nu})^{-(\nu+7)/2}x
    [-\frac{2+\nu}{\nu}x^2+3]
    (\frac{-1}{\sigma^3}).
\end{multline*}
The second derivative evaluations at the third derivative roots $f=y~(x=0)$ and $f=y\pm\sigma\sqrt{\frac{3\nu}{2+\nu}} ~(x^2=\frac{3\nu}{2+\nu})$ are
$-\frac{c}{\sigma^2}\frac{\nu+1}{\nu}$ and
$2\frac{c}{\sigma^2}\frac{\nu+1}{\nu}(\frac{\nu+2}{\nu+5})^{(\nu+5)/2}$, respectively.
Also, the denominator of the second derivative is a polynomial of degree at least 5 whereas the numerator is a polynomial of degree 2 implying that as $f$ approaches $\pm\infty$, the second derivative approaches 0.
Hence, the second derivative is bounded over its domain.